\documentclass[11pt]{article}
\usepackage[margin=1in]{geometry}
\PassOptionsToPackage{authoryear}{natbib}
\usepackage[utf8]{inputenc}
\usepackage[T1]{fontenc}
\usepackage{hyperref}
\usepackage{url}
\usepackage{booktabs}
\usepackage{amsfonts}
\usepackage{nicefrac}
\usepackage{microtype}
\usepackage{xcolor}
\usepackage{mathtools}
\usepackage{amsmath}
\usepackage{amssymb}
\usepackage{amsthm}
\usepackage{enumitem}
\setlist[itemize]{leftmargin=*,itemsep=2pt,topsep=2pt,parsep=0pt}
\setlist[enumerate]{leftmargin=*,nosep}
\theoremstyle{plain}
\newtheorem{theorem}{Theorem}[section]
\newtheorem{lemma}[theorem]{Lemma}
\newtheorem{proposition}[theorem]{Proposition}
\newtheorem{corollary}[theorem]{Corollary}
\theoremstyle{definition}

\newtheorem{assumption}[theorem]{Assumption}
\theoremstyle{remark}
\newtheorem{remark}[theorem]{Remark}

\usepackage{makecell}
\usepackage{adjustbox}
\usepackage{multirow}
\usepackage{graphicx}
\usepackage[compact]{titlesec}
\usepackage{natbib}
\setcitestyle{authoryear,round}
\renewcommand{\eqref}[1]{Eq.~(\ref{#1})}

\title{Reward Transfer from Inverse Reinforcement Learning: A Coupled Minimax Approach}
\author{
Guang-Yuan Hao$^{1}$\thanks{Corresponding author: \href{mailto:gh463@cornell.edu}{gh463@cornell.edu}},
Lars van der Laan$^{2,3}$,
Aurélien Bibaut$^{2}$,
Nathan Kallus$^{1,2}$
\\[1.5em]
$^{1}$Cornell Tech, Cornell University
\\
$^{2}$Netflix Research
\\
$^{3}$Department of Statistics, University of Washington
}

\begin{document}

\maketitle

\begin{abstract}
We study the transfer of rewards learned using inverse reinforcement learning from expert demonstrations in one environment to reinforcement learning in a new, different environment.
This arises naturally when demonstrations are collected in a controlled environment.
We formulate the problem as a joint system of Bellman equations across the source and target environments and develop minimax estimators for the target soft-$q$-function. Whereas a sequential solution approach first estimates the source reward
and then plugs it into the target control problem, a coupled
approach solves the source and target system of equations jointly. We show that, in contrast to the sequential approach, the coupled approach
removes the first-order influence of source Bellman residual error. We characterize the local behavior of each approach, develop finite-sample soft-$q$-function error bounds, and prove regret guarantees for the resulting soft-control
policy. An empirical investigation using a sepsis simulator validates the theoretical comparison.
\end{abstract}

\section{Introduction}

Expert demonstrations, such as those from car drivers, help navigate
environments with unknown rewards, but are often collected in controlled
settings, such as closed-course test tracks, while learned control policies must
be deployed in new environments, such as city streets. We can imitate experts to
perform well in the same source environment where demonstrations are observed,
and we may even use inverse reinforcement learning (IRL) to improve on simple
behavior cloning \citep{ng2000algorithms,abbeel2004apprenticeship,ziebart2008maximum,fu2018learning,geng2020deep}.
But the target environment may have a different transition law, discount factor,
or soft-control regularization. For this, IRL is crucial: we can learn a reward
from demonstrations in the source environment and transfer it to the target
environment, learning a policy that optimizes the same reward function in a new
setting \citep{fu2018learning,schlaginhaufen2024towards}.
In this paper, we characterize how well this transfer can be done and which
approaches are preferable. In particular, we show the value in a coupled approach that
takes the target environment into account even when learning from the source.

In ordinary offline control, the Bellman equation uses a known reward, so the
main statistical error comes from target transitions. In reward transfer, the
target reward is itself estimated from source demonstrations and is defined only
after fixing a normalization, because rewards that induce the same source
behavior are not uniquely identified
\citep{ng1999policy,cao2021identifiability,skalse2023invariance,
rust1987optimal,hotz1993conditional,geng2020deep}. Target error therefore has
two sources: Bellman error from target transitions and reward error from source
demonstrations.

The most straightforward approach is modular: first estimate an
anchor-normalized source reward, then treat that estimate as fixed when solving
the target Bellman equation. This is simple and natural, but it lets source
Bellman residual error enter the target equation with first-order influence. We
show that the same population transfer target admits a coupled estimation
strategy with different local geometry. By solving the source and target
residual equations as one saddle-point system and eliminating the source KKT
block, the target equation acquires a Schur-complement correction. The resulting
profiled target residual is first-order insensitive to perturbations in the
source value nuisance.

\textbf{Contributions.}
\begin{itemize}\vspace{-0.5em}
    \item We formulate offline reward transfer from source IRL to target
    soft-control under shifts in dynamics, discounting, and regularization.
    \item We compare modular and coupled minimax transfer estimators that share
    the same population target but have different local KKT geometry.
    \item We derive the profiled target equation induced by the coupled KKT
    system and show that it cancels first-order sensitivity to the source
    Bellman residual.
    \item We give first-order error decompositions showing that both estimators
    retain source operator and target empirical error, while only modular
    transfer retains the direct source-residual term.
    \item We prove finite-sample high-probability value bounds over finite
    function classes and convert them into regret bounds for the target
    soft-control policy.
    \item We demonstrate the methods empirically in a sepsis simulator, illustrating our theory's predictions.
\end{itemize}

\section{Problem Setup: Reward Transfer from Source IRL to Target Control}
\label{sec:setup}

\paragraph{Environments.} We consider two Markov decision processes (MDPs) over
a shared state space \(\mathcal S\) and finite action space \(\mathcal A\),
with a shared reward function
\(r^\star:\mathcal S\times\mathcal A\to\mathbb R\), but potentially distinct
transition kernels \(P_1\) and \(P_2\). We refer to the former, indexed by
\(1\), as the \textit{source} environment and the latter, indexed by \(2\), as
the \textit{target} environment.

\paragraph{Soft control.} We consider soft control problems in both
environments: in the source, this specifies the agent we assume to observe, and
in the target, it specifies the task we aim to solve. For
\(k\in\{1,2\}\), fix a known full-support reference policy
\(\pi_{k,\rm ref}\), discount \(\gamma_k\in(0,1)\), and temperature
\(\tau_k>0\). In environment \(k\), policy \(\pi_k^\star\) maximizes, over
all measurable policies \(\pi\),
\begin{equation}
\label{eq:soft_V_pi}
\textstyle
V^\pi_k(s):=\mathbb E_{P_k,\pi}\!\bigl[
\sum_{t=0}^\infty \gamma_k^t
\{r^\star(s_t,a_t)-\tau_k
\log \frac{\pi(a_t\mid s_t)}{\pi_{k,\rm ref}(a_t\mid s_t)}\}
\mid s_0=s\bigr],
\end{equation}
where the expectation is taken over
\(a_t\mid s_t\sim\pi(\cdot\mid s_t)\) and
\(s_{t+1}\mid s_t,a_t\sim P_k(\cdot\mid s_t,a_t)\). We normalize
\(\tau_1=1\) without loss of generality by measuring rewards in units of the
source temperature; the target temperature \(\tau_2\) is therefore interpreted
relative to the source temperature. The reference policies are part of the
soft-control objectives, not necessarily the policies that generated the data.
If \(\pi_{1,\rm ref}\) is uniform, the source problem reduces to the usual
maximum-entropy IRL setting \citep{ziebart2008maximum,rust1987optimal}.

\paragraph{Data.} From each
environment, \(k\in\{1,2\}\), we collect $n_k$ i.i.d.~reward-free transitions
\[\mathcal D_k=\{s_i^{(k)},a_i^{(k)},{s_i^{(k)}}':i=1,\ldots,n_k\}\quad\text{with
${s_i^{(k)}}'\mid s_i^{(k)},a_i^{(k)}\sim P_k$.}\]
We let \(\rho_k\),
\(\rho_k^{\mathcal S}\), and \(\pi_{bk}\) denote the joint state-action,
marginal state, and conditional action distributions in dataset \(k\), so that
$(s_i^{(k)},a_i^{(k)},{s_i^{(k)}})\sim\rho_k\times P_k$ and \(\rho_k(s,a)=\rho_k^{\mathcal S}(s)\pi_{bk}(a\mid s)\).

\paragraph{Expert.}
In following with maximum entropy IRL \cite{ziebart2008maximum,rust1987optimal}, we further assume that the source behavior policy is the source soft-optimal
policy: \(\pi_{b1}=\pi_1^\star\). This key assumption allows us to infer the reward function shared between environments. (The target behavior policy \(\pi_{b2}\) is unrestricted.)

\paragraph{Learning objective.}
Our statistical target is the target soft \(q\)-function:
\begin{align}
    Q_2^\star(s,a)
    &:=
    r^\star(s,a)
    +\gamma_2\,
    \mathbb E_{s'\sim P_2(\cdot\mid s,a)}
    [V_2^{\pi_2^\star}(s')],
\end{align}
where \(\pi_2^\star\) and \(V_2^{\pi_2^\star}\) are defined by the target
soft-control problem in \eqref{eq:soft_V_pi}. This is the value of taking the
fixed action \(a\) at state \(s\) and then following the target soft-optimal
policy.

\subsection{Source IRL and reward normalization}
\label{subsec:source_irl_module}

Let \(\pi_{b1}\) be the source expert behavior policy. IRL identifies rewards
only up to state-dependent transformations, so we select a representative by
fixing an anchor policy \(\mu\) and anchor function
\(g:\mathcal S\to\mathbb R\) through \((\mu r)(s)=g(s)\). Thus \(g(s)\) fixes
the mean one-step reward in state \(s\) under \(\mu(\cdot\mid s)\). The anchor
policy is distinct from the soft-control reference \(\pi_{1,\rm ref}\):
\(\pi_{1,\rm ref}\) appears in the expert objective, while \(\mu\) only selects
one reward from an equivalence class. Once fixed, transfer becomes a standard
statistical problem of estimating and using this representative. Such
normalizations are standard in dynamic discrete-choice and IRL models:
anchor-action, outside-option, sum-to-zero, and data-driven value
normalizations all arise as choices of \((\mu,g)\)
\citep{rust1987optimal,hotz1993conditional,bajari2010identification,
geng2020deep,van2025efficient}; see Appendix~\ref{app:normalizations}.

Following \citet{van2025efficient}, define
\(u_b^\star(s,a):=\log\{\pi_{b1}(a\mid s)/\pi_{1,\rm ref}(a\mid s)\}\) and
\(u_g^\star:=u_b^\star-g\).
The source value representation solves
\begin{equation}
\label{eq:source_q1_equation}
q_1^\star
=u_g^\star+\gamma_1P_1^\mu q_1^\star
=(I-\gamma_1P_1^\mu)^{-1}u_g^\star .
\end{equation}
The anchor-normalized reward is then
\begin{equation}
\label{eq:source_reward_formula}
r^\star=(I-\Pi_\mu)q_1^\star+g,
\qquad
r^\star(s,a)=q_1^\star(s,a)-(\mu q_1^\star)(s)+g(s).
\end{equation}
The variable \(q_1^\star\) is a source-side representation of the
normalized reward, not the target value function. This distinction is important:
all target consequences of source estimation error pass through the linear map
\(q_1\mapsto(I-\Pi_\mu)q_1+g\). The main analysis treats \(\pi_{b1}\) as known
and focuses on transfer after source identification. Appendix~\ref{sec:learned_behavior_policy}
gives the extension in which \(\pi_{b1}\) is learned from demonstrations, and
Appendix~\ref{sec:learned_behavior_policy_nuisance} describes how to include
that policy estimator as an additional nuisance.

\subsection{Target soft control}
\label{subsec:target_control_module}

Given the normalized source reward, transfer means solving the target control
problem with the target dynamics and target regularization. The target stage
does not imitate the source policy; it optimizes the same recovered task
objective in the new environment. For the minimax inequality formulation, use
the shifted reward \(r_C^\star:=r^\star+C=(I-\Pi_\mu)q_1^\star+g+C\) and shifted
target action value \(q_2^\star:=Q_2^\star+C/(1-\gamma_2)\). The constant shift
leaves the induced target policy unchanged. Define
\[
\textstyle
\Omega_{\tau_2,\pi_{2,\rm ref}}(q)(s)
:=\tau_2\log\sum_{a\in\mathcal A}
\pi_{2,\rm ref}(a\mid s)\exp\{q(s,a)/\tau_2\}.
\]
We write \(\Omega:=\Omega_{\tau_2,\pi_{2,\rm ref}}\). The transferred target value and
policy satisfy
\begin{align*}
    q_2^\star
    &=(I-\Pi_\mu)q_1^\star+g+C+\gamma_2P_2\Omega(q_2^\star),
    \quad
    \pi_2^\star(a\mid s)
    =
    \frac{\pi_{2,\rm ref}(a\mid s)\exp(q_2^\star(s,a)/\tau_2)}
    {\sum_{a'}\pi_{2,\rm ref}(a'\mid s)\exp(q_2^\star(s,a')/\tau_2)} .
\end{align*}
Thus source error enters target control only through the reward map
\(q_1\mapsto(I-\Pi_\mu)q_1+g+C\).
\subsection{Coupled residual system}
\label{subsec:coupled_transfer_system}

It is useful to write source recovery and target control as one residual
system. This does not change the population solution, but it exposes the path by
which source estimation error can enter the target equation. The population
transfer target solves \(b_1(q_1^\star)=0\) and
\(b_2(q_1^\star,q_2^\star)=0\), where
\begin{equation}
\label{eq:coupled_residuals}
    b_1(q_1):=u_g^\star+\gamma_1P_1^\mu q_1-q_1,
    \qquad
    b_2(q_1,q_2):=(I-\Pi_\mu)q_1+g+C+\gamma_2P_2\Omega(q_2)-q_2 .
\end{equation}
This coupled residual representation is the central object of the paper: it
makes explicit that transfer is a source--target estimation problem, not a
standard target-control problem with a known reward.

\section{Minimax Estimators for Offline Reward Transfer}
\label{sec:minimax_estimators}

Both estimators below are built from the residuals in
\eqref{eq:coupled_residuals}. Let \(\widehat b_1\) and \(\widehat b_2\) denote
the empirical residuals obtained by replacing \(P_1^\mu\) and \(P_2\) with
sample analogues. Let \(\mathcal Q_1,\mathcal Q_2\) be primal classes and
\(\mathcal L_1,\mathcal L_2\) be dual classes.
The primal functions are candidate source and target value representations. The
dual functions test Bellman residuals: if a residual is large in a direction
represented by the dual class, the saddle objective penalizes it. The quadratic
terms select stable solutions among functions that approximately satisfy the
residual equations. This form lets us compare two estimators built from the same
equations but coupled differently.

\subsection{Modular and coupled empirical estimators}
\label{subsec:modular_coupled_estimators}

\paragraph{Modular minimax transfer.}
The modular estimator first solves source recovery,
\begin{equation}
\label{eq:mod_stage1_minimax}\textstyle
(\widehat q_1^{\rm mod},\widehat l_1^{\rm mod})
\in
\arg\min_{q_1\in\mathcal Q_1}\max_{l_1\in\mathcal L_1}
\left\{\frac12\|q_1\|_{\rho_1}^2+
\langle l_1,\widehat b_1(q_1)\rangle_{\rho_1}\right\},
\end{equation}
and then plugs \(\widehat q_1^{\rm mod}\) into the target problem,
\begin{equation}\textstyle
\label{eq:mod_stage2_minimax}
(\widehat q_2^{\rm mod},\widehat l_2^{\rm mod})
\in
\arg\min_{q_2\in\mathcal Q_2}\max_{l_2\in\mathcal L_2}
\left\{\frac12\|q_2\|_{\rho_2}^2+
\langle l_2,\widehat b_2(\widehat q_1^{\rm mod},q_2)\rangle_{\rho_2}\right\}.
\end{equation}

This separation is attractive computationally and conceptually: the source IRL
module can be developed without access to the target data. Its statistical cost
is that any source residual left after the first stage is inserted directly into
the target reward.

\paragraph{Coupled minimax transfer.}
The coupled estimator solves a single saddle-point problem. It uses the same
residuals as the modular estimator, but the source block is allowed to adjust in
the presence of the target block. For \(\beta\ge0\), define
\begin{equation}
\label{eq:empirical_coupled_lagrangian}
\widehat{\mathcal L}_{\rm coup}^{\beta}(q_1,q_2,l_1,l_2)
:=
\frac{\beta}{2}\|q_1\|_{\rho_1}^2
+
\frac12\|q_2\|_{\rho_2}^2
+
\langle l_1,\widehat b_1(q_1)\rangle_{\rho_1}
+
\langle l_2,\widehat b_2(q_1,q_2)\rangle_{\rho_2}.
\end{equation}
Then
\begin{equation}
\label{eq:coup_minimax_estimator}
(\widehat q_1^{\rm coup},\widehat q_2^{\rm coup},
\widehat l_1^{\rm coup},\widehat l_2^{\rm coup})
\in
\arg\min_{q_1\in\mathcal Q_1,\,q_2\in\mathcal Q_2}
\max_{l_1\in\mathcal L_1,\,l_2\in\mathcal L_2}
\widehat{\mathcal L}_{\rm coup}^{\beta}(q_1,q_2,l_1,l_2).
\end{equation}
The coupled estimator uses the same residual equations as the modular estimator,
but it does not freeze the source block before solving the target block.

The parameter \(\beta\) controls how strongly the coupled objective regularizes
the source representation. The main comparison is not about changing the
population target; it is about changing the empirical geometry around that
target.

\subsection{Population target and KKT geometry}
\label{subsec:population_equivalence_kkt_geometry}

Before comparing errors, we establish two facts. First, the modular and coupled
procedures define the same population transfer target. Second, even when
the population target is the same, their KKT systems can transmit empirical
source errors differently.

The corresponding population coupled Lagrangian is
\[
\mathcal L_{\rm coup}^{\beta}(q_1,q_2,l_1,l_2)
:=
\frac{\beta}{2}\|q_1\|_{\rho_1}^2
+
\frac12\|q_2\|_{\rho_2}^2
+
\langle l_1,b_1(q_1)\rangle_{\rho_1}
+
\langle l_2,b_2(q_1,q_2)\rangle_{\rho_2}.
\]

\begin{theorem}[Well-posedness and tightness]
\label{thm:population_well_posed_tightness}
Assume \(\gamma_1,\gamma_2\in[0,1)\), \(P_1^\mu\) and \(P_2\) are Markov
operators, and \(\Omega_{\tau_2,\pi_{2,\rm ref}}\) is the log-sum-exp soft value
operator. Then the source equation has a unique solution \(q_1^\star\), and,
given \(q_1^\star\), the target soft Bellman equation has a unique solution
\(q_2^\star\). If \(r_C^\star=(I-\Pi_\mu)q_1^\star+g+C\ge0\), then the relaxed
constraint \(b_2(q_1^\star,q_2)\le0\) is tight at every minimum
\(\rho_2\)-norm feasible solution. Hence the relaxed and equality formulations
recover the same \((q_1^\star,q_2^\star)\).
\end{theorem}

The tightness condition rules out spurious solutions created by relaxing the
target Bellman equality into an inequality. It is a device for the minimax
formulation: once the shifted reward is
nonnegative, the minimum-norm feasible target value must satisfy the Bellman
equation with equality.

\begin{proposition}[Population equivalence]
\label{prop:population_equivalence}
Under Theorem~\ref{thm:population_well_posed_tightness}, the modular population
procedure and the coupled population minimax problem have the same primal
solution \((q_1^\star,q_2^\star)\).
\end{proposition}

Thus any difference between the two estimators is not a population bias. It
comes from how sampling error perturbs the two estimating systems.

\begin{theorem}[Population saddle point and quadratic growth]
\label{thm:population_saddle_kkt}
Assume Theorem~\ref{thm:population_well_posed_tightness}. Suppose the dual
classes contain \((l_1^\star,l_2^\star)\), with \(l_2^\star\ge0\), satisfying
\begin{equation}
\label{eq:dual_adjoint_l2}
\left\langle l_2^\star,(I-\gamma_2P_2^{\pi_2^\star})h\right\rangle_{\rho_2}
=
\langle q_2^\star,h\rangle_{\rho_2},
\qquad \forall h,
\end{equation}
\begin{equation}
\label{eq:dual_adjoint_l1}
\left\langle l_1^\star,(I-\gamma_1P_1^\mu)h\right\rangle_{\rho_1}
=
\beta\langle q_1^\star,h\rangle_{\rho_1}
+
\left\langle l_2^\star,(I-\Pi_\mu)h\right\rangle_{\rho_2},
\qquad \forall h.
\end{equation}
Here \(P_2^{\pi_2^\star}:=P_2D\Omega(q_2^\star)\). Then
\((q_1^\star,q_2^\star,l_1^\star,l_2^\star)\) is a saddle point of
\(\mathcal L_{\rm coup}^{\beta}\) satisfying
\begin{equation}
\label{eq:kkt_stationarity_explicit}
\beta q_1^\star-(I-\gamma_1P_1^\mu)^\top l_1^\star+(I-\Pi_\mu)^\top l_2^\star=0,
\qquad
q_2^\star-(I-\gamma_2P_2^{\pi_2^\star})^\top l_2^\star=0.
\end{equation}
Moreover, for every \((q_1,q_2)\),
\begin{equation}
\label{eq:population_saddle_quadratic_growth}
\mathcal L_{\rm coup}^{\beta}(q_1,q_2,l_1^\star,l_2^\star)
-
\mathcal L_{\rm coup}^{\beta}(q_1^\star,q_2^\star,l_1^\star,l_2^\star)
\ge
\frac{\beta}{2}\|q_1-q_1^\star\|_{\rho_1}^2+
\frac12\|q_2-q_2^\star\|_{\rho_2}^2 .
\end{equation}
\end{theorem}

The adjoint identity \eqref{eq:dual_adjoint_l1} is the key difference between
coupled and modular transfer: the coupled source stationarity equation includes
the target-to-source feedback term
\(\langle l_2^\star,(I-\Pi_\mu)h\rangle_{\rho_2}\).
Modular source estimation has
only the source self term.
The quadratic-growth inequality says that, after the correct dual certificate
is chosen, the population saddle objective controls squared error in the primal
variables. This later lets residual concentration imply value-function
concentration.

The next assumption is a coverage condition for these dual certificates. It
requires the sampling distributions to cover the discounted trajectories whose
residuals the dual variables represent. The cross-coverage terms appear only
because coupled transfer sends target information back through the source
adjoint equation.

\begin{assumption}[Coverage for dual boundedness]
\label{ass:population_dual_coverage}
Let \(d^{P_2}_{\pi_2^\star,\rho_2}\) and \(d^{P_1}_{\mu,\rho_1}\) denote the
usual discounted occupancies induced by \(\rho_2\) and \(\rho_1\). Assume
\(\|q_1^\star\|_\infty\le B_{Q1}\), \(\|q_2^\star\|_\infty\le B_{Q2}\),
\[
\left\|d^{P_2}_{\pi_2^\star,\rho_2}/\rho_2\right\|_\infty\le\kappa_2,
\qquad
\left\|d^{P_1}_{\mu,\rho_1}/\rho_1\right\|_\infty\le\kappa_1,
\]
and, for the coupled cross term,
\[
\left\|d^{P_2}_{\pi_2^\star,\rho_2}/\rho_1\right\|_\infty\le\kappa_{12},
\qquad
\left\|\mu/\pi_{b1}\right\|_\infty\le\kappa_{\mu\mid\pi_{b1}},
\qquad
\left\|d^{P_2,\mathcal S}_{\pi_2^\star,\rho_2}/\rho_1^{\mathcal S}\right\|_\infty
\le\kappa_{12}^{\mathcal S}.
\]
\end{assumption}

Here, \(\kappa_2\) controls the target
dual, \(\kappa_1\) controls ordinary source evaluation under \(\mu\), and
\(\kappa_{12},\kappa_{12}^{\mathcal S}\) control the target-to-source feedback
needed by the coupled KKT system.

\begin{proposition}[Dual representations and boundedness]
\label{prop:dual_representation_coverage}
\label{prop:dual_boundedness_coverage}
Under Assumption~\ref{ass:population_dual_coverage}, we have
\begin{align*}
\textstyle l_2^\star
&\textstyle =
\frac{(I-\gamma_2(P_2^{\pi_2^\star})^\top)^{-1}(\rho_2\odot q_2^\star)}{\rho_2},
\qquad
l_{1,{\rm mod}}^\star =l_{1,{\rm self}}^\star
=
\frac{(I-\gamma_1(P_1^\mu)^\top)^{-1}(\rho_1\odot q_1^\star)}{\rho_1},
\\\textstyle
l_{1,{\rm coup}}^\star
&\textstyle =
\beta l_{1,{\rm self}}^\star+l_{1,{\rm cross}}^\star,
\qquad
l_{1,{\rm cross}}^\star
=
\frac{(I-\gamma_1(P_1^\mu)^\top)^{-1}(I-\Pi_\mu)^\top(\rho_2\odot l_2^\star)}{\rho_1}.
\end{align*}
Consequently,
\begin{align*}
\|l_2^\star\|_\infty
&\textstyle\le \frac{\kappa_2}{1-\gamma_2}B_{Q2},
~
\|l_{1,{\rm mod}}^\star\|_\infty
\le \frac{\kappa_1}{1-\gamma_1}B_{Q1},
~
\|l_{1,{\rm coup}}^\star\|_\infty
\le
\beta\frac{\kappa_1}{1-\gamma_1}B_{Q1}
+
\frac{\kappa_1(\kappa_{12}+\kappa_{\mu\mid\pi_{b1}}\kappa_{12}^{\mathcal S})}
{(1-\gamma_1)(1-\gamma_2)}B_{Q2}.
\end{align*}
\end{proposition}

The resolvent formulas are backward occupancy representations. They show why
coverage appears in the constants: a dual variable measures how residual error
accumulates along discounted trajectories before being compared under the
sampling distribution. The coupled source dual has an additional cross term
because target residuals are pushed back through the source reward map.

\begin{corollary}[First-order contrast between modular and coupled transfer]
\label{cor:first_order_modular_coupled_contrast}
Let \(\Delta_i=q_i-q_i^\star\), and let \(R_\Omega(\Delta_2)\) denote the
second-order remainder from Taylor expanding \(\Omega\) around \(q_2^\star\).
With coupled dual variables,
\[
\mathcal L_{\rm coup}^{\beta}(q_1^\star+\Delta_1,q_2^\star+\Delta_2,l_1^\star,l_2^\star)
-
\mathcal L_{\rm coup}^{\beta}(q_1^\star,q_2^\star,l_1^\star,l_2^\star)
=
\frac{\beta}{2}\|\Delta_1\|_{\rho_1}^2+
\frac12\|\Delta_2\|_{\rho_2}^2+R_\Omega(\Delta_2).
\]
With the modular source dual \(l_{1,{\rm mod}}^\star\), the same expansion
contains the surviving linear term
\[
\left\langle l_2^\star,(I-\Pi_\mu)\Delta_1\right\rangle_{\rho_2}.
\]
Thus modular transfer is first-order sensitive to source perturbations, while
coupled transfer cancels that sensitivity through \eqref{eq:dual_adjoint_l1}.
\end{corollary}

\section{Structural Orthogonality of Coupled Minimax Transfer}
\label{sec:structural_orthogonality}

The preceding corollary shows the cancellation in the fixed-dual objective. We
now express the same phenomenon at the residual level. Let
\(\eta=(q_1,l_1)\) be the source block and \(\theta=(q_2,l_2)\) the target block.
The goal is to derive the target equation that remains after the source KKT
equations have been profiled out. This is the equation the coupled estimator
implicitly uses near the population solution. Write the population KKT system as
\[
S_N(\theta,\eta)=
\begin{pmatrix}
D_{q_1}\mathcal L_{\rm coup}^{\beta}(q_1,q_2,l_1,l_2)\\
b_1(q_1)
\end{pmatrix}=0,
\qquad
S_T(\theta,\eta)=
\begin{pmatrix}
D_{q_2}\mathcal L_{\rm coup}^{\beta}(q_1,q_2,l_1,l_2)\\
b_2(q_1,q_2)
\end{pmatrix}=0.
\]

\begin{theorem}[Profiled target equation and structural orthogonality]
\label{thm:profiled_target_orthogonality}
Assume \(\gamma_1<1\) and \(\|P_1^\mu\|\le1\). Then
\(I-\gamma_1P_1^\mu\) and
\[
D_\eta S_N(\theta^\star,\eta^\star)=
\begin{pmatrix}
\beta I & -(I-\gamma_1P_1^\mu)^\top\\
-(I-\gamma_1P_1^\mu) & 0
\end{pmatrix}
\]
are invertible. Define
\begin{equation}
\label{eq:psi_coup_def}
\Psi_{\rm coup}(\theta,\eta)
:=
S_T(\theta,\eta)
-
D_\eta S_T(\theta^\star,\eta^\star)
\bigl[D_\eta S_N(\theta^\star,\eta^\star)\bigr]^{-1}
S_N(\theta,\eta).
\end{equation}
Then \(\Psi_{\rm coup}\) is the Schur-complement target equation obtained by
linearized elimination of the source KKT block, and
\[
D_\eta\Psi_{\rm coup}(\theta^\star,\eta^\star)=0.
\]
\end{theorem}

The correction in \eqref{eq:psi_coup_def} agrees with the exact reduced target
equation on the local profile manifold \(S_N(\theta,\eta(\theta))=0\), and its
\(\theta\)-derivative is the Schur complement of the full KKT Jacobian. The
important point is interpretive: no external orthogonal score is introduced;
the correction is the target block after eliminating the source block.

Thus the orthogonality is not a design choice layered on top of the estimator.
It is the local target equation induced by solving the source and target blocks
as one system.

\begin{corollary}[Explicit orthogonalized target residual]
\label{cor:orthogonalized_target_residual}
For the reward-transfer system, the residual component of the profiled target
KKT equation is
\begin{equation}
\label{eq:orthogonalized_target_residual}
\widetilde b_{2,{\rm coup}}(q_1,q_2)
=
b_2(q_1,q_2)
+
(I-\Pi_\mu)(I-\gamma_1P_1^\mu)^{-1}b_1(q_1).
\end{equation}
Consequently,
$D_{q_1}\widetilde b_{2,{\rm coup}}(q_1^\star,q_2^\star)=0.$
\end{corollary}

Since \(b_1(q_1)=u_g^\star-(I-\gamma_1P_1^\mu)q_1\), the \(q_1\)-terms in
\eqref{eq:orthogonalized_target_residual} cancel at the population level. The
profiled residual therefore removes the direct first-order effect of a source
Bellman residual. It does not remove error in the source transition operator, which still appears inside the source resolvent.

This distinction drives the rest of the analysis. Coupling cannot make source
data irrelevant, but it can prevent an avoidable source fitting residual from
being treated as a target reward perturbation. Appendix~\ref{app:profiling}
relates this Schur-complement orthogonality to classical profiling.

\section{First-Order Error Propagation in Offline Reward Transfer}
\label{sec:error_propagation}

We now turn the structural orthogonality result into error decompositions. Let
\(\widehat P_1^\mu\) and \(\widehat P_2\) be empirical transition operators and
\[
\widehat b_1(q_1):=u_g^\star+\gamma_1\widehat P_1^\mu q_1-q_1.
\]
Let \(P_2^{\pi_2^\star}:=P_2D\Omega(q_2^\star)\). The linear target operator is
\(I-\gamma_2P_2^{\pi_2^\star}\).

The decompositions below are local: they describe the leading terms obtained by
linearizing around \((q_1^\star,q_2^\star)\). Their purpose is not to give the
final high-probability rate, but to identify the error channels that any
finite-sample bound must control.

\begin{lemma}[Source error decomposition]
\label{lem:source_error_decomposition}
If \(I-\gamma_1\widehat P_1^\mu\) is invertible, then every
\(\widehat q_1\) satisfies
\begin{equation}
\label{eq:source_error_decomposition}
\widehat q_1-q_1^\star
=
\gamma_1(I-\gamma_1\widehat P_1^\mu)^{-1}(\widehat P_1^\mu-P_1^\mu)q_1^\star
-
(I-\gamma_1\widehat P_1^\mu)^{-1}\widehat b_1(\widehat q_1).
\end{equation}
\end{lemma}

The first term is source operator estimation error. The second is the empirical
source Bellman residual left after fitting. Modular and coupled transfer treat
these two terms differently.

The operator term is intrinsic: if the source transition operator is estimated
incorrectly, the recovered reward representation changes. The residual term is
algorithmic: it measures how far the fitted source object is from solving its
empirical Bellman equation.

\begin{theorem}[First-order error propagation for modular and coupled transfer]
\label{thm:first_order_transfer_error_propagation}
Under the local smoothness, invertibility, and stochastic equicontinuity
conditions stated in Appendix~\ref{app:first_order_error_propagation},
\begin{align}
\label{eq:mod_first_order_decomposition}
(I-\gamma_2P_2^{\pi_2^\star})(\widehat q_2^{\rm mod}-q_2^\star)
&=
\gamma_1(I-\Pi_\mu)(I-\gamma_1\widehat P_1^\mu)^{-1}
(\widehat P_1^\mu-P_1^\mu)q_1^\star
\nonumber\\
&\quad
-(I-\Pi_\mu)(I-\gamma_1\widehat P_1^\mu)^{-1}
\widehat b_1(\widehat q_1^{\rm mod})
\nonumber\\
&\quad
+\gamma_2(\widehat P_2-P_2)\Omega(q_2^\star)+\mathrm{Rem}_{\rm mod},
\\[0.4em]
\label{eq:coup_first_order_decomposition}
(I-\gamma_2P_2^{\pi_2^\star})(\widehat q_2^{\rm coup}-q_2^\star)
&=
\gamma_1(I-\Pi_\mu)(I-\gamma_1\widehat P_1^\mu)^{-1}
(\widehat P_1^\mu-P_1^\mu)q_1^\star
\nonumber\\
&\quad
+\gamma_2(\widehat P_2-P_2)\Omega(q_2^\star)+\mathrm{Rem}_{\rm coup}.
\end{align}
The remainders are second-order Taylor terms and local empirical-process
increments; sufficient conditions for their first-order negligibility are given
in Appendix~\ref{app:first_order_error_propagation}.
\end{theorem}

Thus both procedures pay for source operator error and target empirical error.
Only modular transfer has the additional first-order source-residual term. This
is the precise local sense in which the coupled estimator is more robust to
imperfect source Bellman fitting.

The statement should not be read as saying that coupling dominates modular
transfer in every finite sample. The coupled estimator solves a larger empirical
problem and pays its own concentration price. The point is sharper: in the
linearized target equation, the direct source-residual channel is absent for the
coupled estimator and present for the modular estimator.

\section{Global Finite-Sample and Policy Performance Guarantees}
\label{sec:global_policy_guarantees}

We now state global high-probability guarantees obtained directly from empirical
minimax optimality and uniform concentration over finite function classes. The
finite-class setting keeps the proof focused on how source and target errors
combine; the same decomposition identifies what must be controlled under richer
function approximation. Let \(N_1=|\mathcal Q_1||\mathcal L_1|\),
\(N_2^{\rm mod}=|\mathcal Q_2||\mathcal L_2|\), and
\(N_2^{\rm coup}=|\mathcal Q_1||\mathcal Q_2||\mathcal L_2|\). Assume
\(\|q_i\|_\infty\le B_{Q_i}\) for \(i=1,2\),
\(\|u_b^\star\|_\infty\le B_U\), and \(0\le g(s)\le B_G\). Define the residual
ranges \(B_{{\rm Res},1}:=B_U+B_G+(\gamma_1+1)B_{Q1}\) and
\(B_{{\rm Res},2}:=4B_{Q1}+B_G+(\gamma_2+1)B_{Q2}\).

\begin{assumption}[Standing assumptions for global guarantees]
\label{ass:global_policy_guarantees}
The finite classes are realizable and uniformly bounded:
\(q_1^\star\in\mathcal Q_1\), \(q_2^\star\in\mathcal Q_2\),
\(l_1^\star\in\mathcal L_1\), \(l_2^\star\in\mathcal L_2\), with primal bounds
\(B_{Q1},B_{Q2}\). The target dual radius satisfies
\(B_{L2}:=\kappa_2B_{Q2}/(1-\gamma_2)\) and
\(\|d^{P_2}_{\pi_2^\star,\rho_2}/\rho_2\|_\infty\le\kappa_2\).
For the coupled source dual, let \(B_{L12}\) be the cross-environment source
radius from Proposition~\ref{prop:dual_boundedness_coverage}. For modular
propagation, assume \(\|\rho_2/\rho_1\|_\infty<\infty\),
\(\|\mu/\pi_{b1}\|_\infty\le1/\epsilon_{b1}\), and pathwise target
concentrability with constant \(\kappa_{2,{\rm path}}\) along the line segment
between the true target value and the modular plug-in target value.
\end{assumption}

Realizability enables comparison with the truth, boundedness gives concentration,
and coverage converts residual control into value control; the pathwise coverage
condition is needed only for the modular plug-in path.

\begin{theorem}[Global finite-sample bounds]
\label{thm:global_q2_bounds}
Under Assumption~\ref{ass:global_policy_guarantees}, for any \(\delta\in(0,1)\):
\begin{align}
\label{eq:modular_global_q2_bound}
 \|\widehat q_2^{\rm mod}-q_2^\star\|_{\rho_2}^2
& \le
8
M_1^{\rm mod}\sqrt{\frac{2\ln(4N_1/\delta)}{n_1}}
+
8M_2\sqrt{\frac{2\ln(4N_2^{\rm mod}/\delta)}{n_2}}
,
\\
\label{eq:coupled_global_q2_bound}
\|\widehat q_2^{\rm coup}-q_2^\star\|_{\rho_2}^2
& \le
4
M_1^{\rm coup}\sqrt{\frac{2\ln(2N_1/\delta)}{n_1}}
+
4M_2\sqrt{\frac{2\ln(2N_2^{\rm coup}/\delta)}{n_2}}
\end{align}
with probability at least \(1-\delta\), where the coupled bound uses \(\beta=0\),
and
\[
\begin{aligned}
M_1^{\rm mod}
&=
\frac{4\kappa_{2,{\rm path}}\|\rho_2/\rho_1\|_\infty}
{\epsilon_{b1}(1-\gamma_2)^2}
\left(
\tfrac12 B_{Q1}^2
+
B_{L1}^{\rm mod}B_{{\rm Res},1}
\right),\\
M_1^{\rm coup}
&=
B_{L12}B_{{\rm Res},1},
\qquad
M_2
=
\tfrac12 B_{Q2}^2+B_{L2}B_{{\rm Res},2}.
\end{aligned}
\]
\end{theorem}

This is the global counterpart of the local decomposition: both estimators have
source and target terms, but modular first controls source error under
\(\rho_1\) and then propagates it, whereas coupled controls the source
correction inside one joint saddle comparison. Define
\(\mathcal B_{\rm mod}(\delta)\) and \(\mathcal B_{\rm coup}(\delta)\) as the
right-hand sides of
\eqref{eq:modular_global_q2_bound} and \eqref{eq:coupled_global_q2_bound}.

\subsection{Policy regret from transferred-value error}
\label{subsec:policy_regret}

The deployed object is a policy, so we convert \(\rho_2\)-error in the
transferred soft value into regret using smoothness of the softmax map and
coverage of the optimal target occupancy. For a policy \(\pi\), let
\(J_2(\pi):=(1-\gamma_2)^{-1}\mathbb E_{d^{P_2}_{\pi,\rho_2}}
[r_C^\star(s,a)-\tau_2\log\{\pi(a\mid s)/\pi_{2,\rm ref}(a\mid s)\}]\), and let
\(\widehat\pi_2\) be induced by \(\widehat q_2\).

\begin{assumption}
\label{ass:policy_regret_conversion}
\(\|q_2^\star\|_\infty\le B_{Q2}\),
\(\|d^{P_2}_{\pi_2^\star,\rho_2}/\rho_2\|_\infty\le\kappa_2\), and
\(\pi_2^\star(a\mid s)\ge\epsilon_{\pi2}^\star\) for all \((s,a)\).
\end{assumption}

The lower bound on \(\pi_2^\star\) is the support condition needed to translate
value error into policy error.

\begin{theorem}[Policy regret from transferred-value error]
\label{thm:policy_regret}
Under Assumption~\ref{ass:policy_regret_conversion}, for any estimate
\(\widehat q_2\) and induced policy \(\widehat\pi_2\),
\begin{equation}
\label{eq:policy_regret_from_q}
J_2(\pi_2^\star)-J_2(\widehat\pi_2)
\le
C_\pi\|\widehat q_2-q_2^\star\|_{\rho_2},
\qquad
C_\pi:=
\frac{B_{Q2}\sqrt{|\mathcal A|\kappa_2}}
{(1-\gamma_2)\tau_2\sqrt{\epsilon_{\pi2}^\star}}.
\end{equation}
Consequently, with probability at least \(1-\delta\),
\(J_2(\pi_2^\star)-J_2(\widehat\pi_2^{\rm mod})
\le C_\pi\sqrt{\mathcal B_{\rm mod}(\delta)}\) and
\(J_2(\pi_2^\star)-J_2(\widehat\pi_2^{\rm coup})
\le C_\pi\sqrt{\mathcal B_{\rm coup}(\delta)}\).
\end{theorem}

The square root appears because Theorem~\ref{thm:global_q2_bounds} controls
squared value error; longer horizons, lower temperature, weaker coverage, and
smaller optimal action probabilities all increase \(C_\pi\).

\section{Experimental Investigation}
\label{sec:experiments}
We evaluate the predicted first-order robustness of coupled minimax transfer in
the Gumbel-Max sepsis SCM simulator of \citet{oberst2019counterfactual}, a
synthetic sequential-treatment benchmark with four clinical variables and three
binary treatment decisions. We tabularize the simulator into \(128\) states and
\(8\) actions, take \(P_1\) from the original dynamics, and form \(P_2\) by
perturbing clinical transition mechanisms while preserving support; source and
target differ in dynamics, discounting, and soft-control temperature.

We compare \emph{Modular}, which estimates \(\widehat q_1^{\rm mod}\), forms
\((I-\Pi_\mu)\widehat q_1^{\rm mod}+g+C\), and then solves target control;
\emph{Coupled}, which jointly optimizes the source and target residuals; and
\emph{Coupled-Offset}, which starts from the Modular solution and learns a
joint offset correction. All methods use the same datasets, architecture,
initialization, learning rates, and evaluation protocol. Each configuration
averages \(10\) dataset draws and \(10\) optimization seeds; each episode has
\(20\) steps. The target dataset \(\mathcal D_2\) is fixed at
\(25{,}000\) episodes, while the source dataset \(\mathcal D_1\) is varied as a
fraction of a \(12{,}500\)-episode reference size; thus \(\mathcal D_1\) fraction \(0.2\)
corresponds to \(2{,}500\) source episodes. This mimics settings where expert
demonstrations are scarcer than ordinary target data. We use the
no-treatment action as the source anchor and the uniform policy as
\(\pi_{2,\rm ref}\), and report target regret, target \(q_2\) and \(V_2\)
errors, source reward error, and source \(q_1\) error.

\begin{figure}[t]
    \centering
    \includegraphics[width=0.84\linewidth]{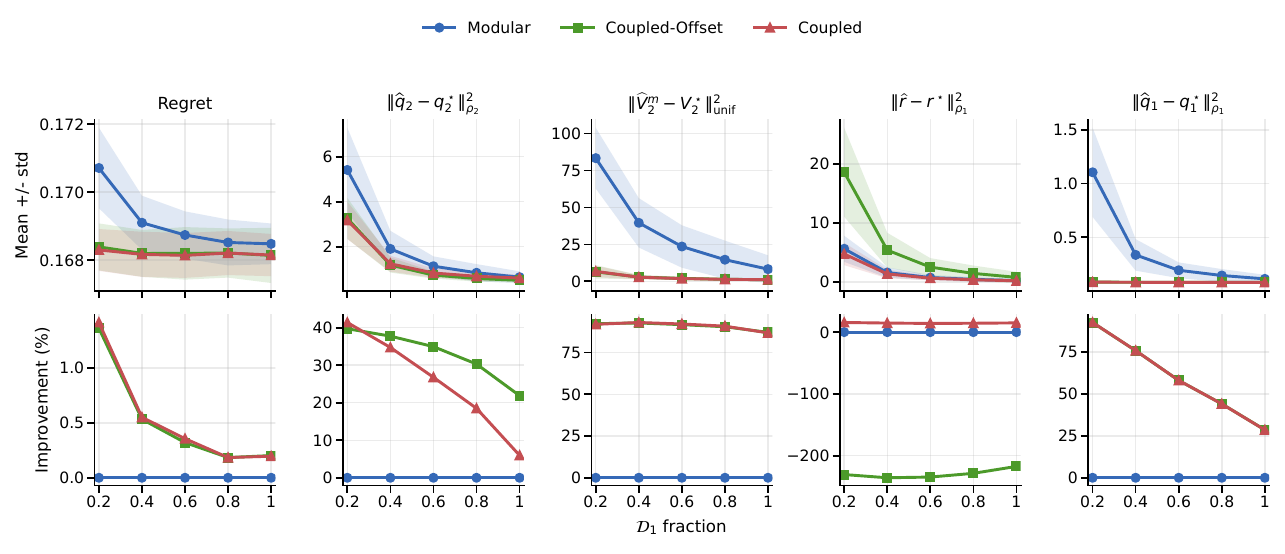}
    \caption{
    Source-size experiment for \(\tau_2=0.05\): absolute metrics over \(\mathcal D_1\)
    fractions \(0.2\)--\(1.0\) (top) and percent improvement over Modular
    (bottom); larger is better.
    }
    \label{fig:d1_subset}
\end{figure}

Figure~\ref{fig:d1_subset} shows the main source-size experiment. At
\(\tau_2=0.05\), the coupled methods improve most when source data are scarce,
exactly where the empirical source residual is most variable. This matches
Theorem~\ref{thm:first_order_transfer_error_propagation}: modular transfer
passes this residual directly into the target reward, while coupled transfer
removes that first-order channel. With \(\mathcal D_1\) fraction \(0.2\), Coupled reduces
weighted target \(q_2\) MSE from \(5.413\) to \(3.173\) and unweighted \(V_2\)
MSE from \(83.38\) to \(6.739\); Coupled-Offset gives similar target-side gains.
The effect persists across target temperatures (Table~\ref{tab:target_transfer_tau}):
at \(\tau_2=0.2\), Coupled reduces unweighted \(V_2\) MSE from \(72.77\) to
\(6.327\). Regret improvements are smaller, as expected for a downstream policy
metric, but remain consistent and are most pronounced at lower temperature,
matching the \(1/\tau_2\) dependence in Theorem~\ref{thm:policy_regret}.

\begin{table}[t]
\newcommand{\tabms}[2]{\begin{tabular}[t]{@{}r@{}}\ensuremath{#1}\\[-0.35ex]\scriptsize\textcolor{black!55}{\ensuremath{\pm #2}}\end{tabular}}
\setlength{\tabcolsep}{2.3pt}
\renewcommand{\arraystretch}{1.18}
\caption{Performance across temperatures $\tau_2$ at \(\mathcal D_1\) fraction \(0.2\)
(standard deviations below in gray).}
\label{tab:target_transfer_tau}
\centering
\small
\begin{adjustbox}{max width=\textwidth}
\begin{tabular}{c|ccc|ccc|ccc|ccc|ccc}
\toprule
\multirow{2}{*}{Method}
& \multicolumn{3}{c}{Regret}
& \multicolumn{3}{c}{\(\|\widehat q_2-q_2^\star\|_{\rho_2}^2\)}
& \multicolumn{3}{c}{\(\|\widehat V_2-V_2^\star\|_{\rm unif}^2\)}
& \multicolumn{3}{c}{\(\|\widehat r-r^\star\|_{\rho_1}^2\)}
& \multicolumn{3}{c}{\(\|\widehat q_1-q_1^\star\|_{\rho_1}^2\)} \\
\cmidrule(lr){2-4}\cmidrule(lr){5-7}\cmidrule(lr){8-10}\cmidrule(lr){11-13}\cmidrule(lr){14-16}
\(\tau_2\) & .05 & .2 & .4 & .05 & .2 & .4 & .05 & .2 & .4 & .05 & .2 & .4 & .05 & .2 & .4 \\
\midrule
Modular
& \tabms{.171}{.00118} & \tabms{.0595}{.00208} & \tabms{.0554}{.000723}
& \tabms{5.413}{1.879} & \tabms{4.060}{1.503} & \tabms{1.601}{.644}
& \tabms{83.38}{20.61} & \tabms{72.77}{21.76} & \tabms{22.60}{9.057}
& \tabms{5.632}{2.206} & \tabms{5.632}{2.206} & \tabms{5.632}{2.206}
& \tabms{1.105}{.415} & \tabms{1.105}{.415} & \tabms{1.105}{.415} \\
Coupled-Offset
& \tabms{.168}{.000691} & \tabms{.0571}{.000566} & \tabms{.0547}{.000494}
& \tabms{3.262}{.909} & \tabms{.934}{.210} & \tabms{.604}{.125}
& \tabms{6.415}{4.526} & \tabms{5.008}{2.909} & \tabms{2.677}{1.215}
& \tabms{18.62}{7.543} & \tabms{17.31}{7.368} & \tabms{17.42}{7.251}
& \tabms{.0832}{.00149} & \tabms{.0866}{.00232} & \tabms{.0847}{.00178} \\
Coupled
& \tabms{.168}{.000615} & \tabms{.0570}{.000484} & \tabms{.0546}{.000487}
& \tabms{3.173}{.834} & \tabms{1.015}{.226} & \tabms{.594}{.109}
& \tabms{6.739}{3.835} & \tabms{6.327}{3.002} & \tabms{2.872}{1.218}
& \tabms{4.756}{1.894} & \tabms{4.771}{1.910} & \tabms{4.800}{1.916}
& \tabms{.0827}{.000886} & \tabms{.0857}{.00169} & \tabms{.0842}{.00126} \\
\bottomrule
\end{tabular}
\end{adjustbox}
\end{table}

The large \(V_2\) gains come from reducing error on actions selected by the
learned policy. Decomposing \(V_2(s)=(\pi_2q_2)(s)\), the selected-action
\(q_2\) MSE at \(\tau_2=0.05\) and \(\mathcal D_1\) fraction \(0.2\) is \(82.74\) for
Modular, versus \(6.72\) for Coupled and \(6.38\) for Coupled-Offset; the
policy-mismatch term is only about \(0.016\) for all methods.
Thus Modular makes larger errors exactly where the learned target policy places
mass (see Appendix~\ref{app:v2_decomposition} for the detailed error
decomposition).

Reward diagnostics add one caution: anchor-action recovery uses \(q_1\)-error
contrasts with the no-treatment action, so baseline error affects every
recovered reward. At \(\tau_2=0.05\) and \(\mathcal D_1\) fraction \(0.2\),
Coupled-Offset and Coupled have similar weighted \(q_1\) MSEs but very different
no-treatment \(q_1\) MSEs (\(489.2\) versus \(121.8\)), explaining the reward
MSE gap (\(18.62\) versus \(4.76\)). Full plots, tables, simulator details, and
training settings are in Appendix~\ref{app:experimental_details}.

\section{Related Work}

\paragraph{IRL, DDC, and reward ambiguity.}
Classical IRL recovers rewards under which observed behavior is optimal
\citep{ng2000algorithms,abbeel2004apprenticeship}, with maximum-entropy variants
using soft Bellman equations and softmax policies
\citep{ziebart2008maximum,ziebart2010modeling} and sample-based deep variants
\citep{finn2016guided}. Rewards are generally unidentified without normalization,
due to shaping and related invariances
\citep{ng1999policy,cao2021identifiability,rolland2022identifiability,skalse2023invariance, schlaginhaufen2024towards}.
Dynamic discrete choice models use the same logit-choice Bellman structure to
recover payoffs \citep{rust1987optimal,hotz1993conditional,aguirregabiria2010dynamic,bajari2010identification};
anchor restrictions in DDC and IRL select a representative
\citep{rust1987optimal,hotz1993conditional,geng2020deep,van2025efficient}. We
take this normalization step seriously: the paper is not about resolving all
reward ambiguities, but about what happens statistically after a representative
has been chosen and transferred.

\paragraph{Reward transfer and imitation.}
Adversarial imitation methods such as GAIL and AIRL learn reward-like objects
from demonstrations \citep{ho2016generative,fu2018learning}; AIRL motivates
reward transfer under dynamics changes rather than direct policy imitation. This
paper addresses a different part of the pipeline: conditional on a chosen source
IRL signal, we analyze how source estimation error propagates through an offline
target-control problem. Recent work studies when recovered rewards generalize or
transfer across experts and environments
\citep{rolland2022identifiability,schlaginhaufen2024towards}; here we analyze
how source IRL estimation error propagates into offline target control and how
coupling removes one first-order channel.

\paragraph{Soft control and offline value estimation.}
The target problem is KL-regularized soft control
\citep{kappen2005linear,todorov2009efficient,levine2018reinforcement,haarnoja2017reinforcement,haarnoja2018soft,geist2019theory}.
These literatures give the soft Bellman equations and softmax policies used for
target control. Our contribution is not a new soft-control algorithm, but an
analysis of the extra statistical channel created when the target reward is
estimated from a separate source IRL stage.
Our estimators use minimax Bellman residual ideas from offline value estimation
\citep{ernst2005tree,munos2008finite,uehara2020minimax,xie2021batch,uehara2023offline},
but the target residual is coupled to a source residual through the learned
reward map.

\paragraph{Two-stage estimation and orthogonality.}
The modular-coupled contrast is connected to two-stage estimation, profiling,
and orthogonalization
\citep{murphy1985estimation,newey1994asymptotic,chernozhukov2018double,foster2023orthogonal};
profile-likelihood theory studies the target-score geometry after optimizing
over nuisances
\citep{severini1992profile,murphy1997semiparametric,murphy1999observed,murphy2000profile}.
In double machine learning, orthogonal scores are typically designed to reduce
sensitivity to nuisance error. Here the orthogonalized target equation is
induced by Schur-complement elimination of the source block in the coupled
Bellman KKT system.

\section{Conclusion}
\label{sec:conclusion}

We studied offline reward transfer from a source environment with expert
demonstrations to a target environment with only off-policy data. The difficulty
is that target control uses a reward recovered from the source, so source-stage
error enters the target Bellman equation. A modular plug-in estimator first
learns the reward and then solves the target problem; any source Bellman
residual left by the first stage can therefore affect target value at first
order.

The coupled estimator instead solves source recovery and target value estimation as one
single problem. It has the same population target, but its KKT equations profile
out the source block and yield a Schur-complement target equation that removes
the direct first-order effect of source residual error. The remaining terms come
from estimating the source and target Bellman operators and from coverage and
finite-sample complexity. The sepsis experiments show the predicted gains when
source data are scarce. These results give a concrete statistical reason to
couple IRL and target control, and point to richer function classes and scalable
solvers.

\bibliographystyle{plainnat}
\bibliography{reference}

\appendix

\section{Reward normalization examples}
\label{app:normalizations}

The reward normalizations introduced in
Section~\ref{subsec:source_irl_module} are standard in economics and are often
substantively meaningful
\citep{hotz1993conditional,bajari2010identification,geng2020deep,
van2025efficient}. In our notation, a normalization fixes an anchor policy
\(\mu\) and anchor function \(g:\mathcal S\to\mathbb R\) through
\[
    (\mu r)(s)=\sum_{a\in\mathcal A}\mu(a\mid s)r(s,a)=g(s),
    \qquad s\in\mathcal S .
\]
The fixed anchor-action normalization is the special case
\(\mu(a\mid s)=1\{a=a^\dagger\}\), for which
\((\mu r)(s)=r(s,a^\dagger)=g(s)\) \citep{geng2020deep}. When \(g\equiv0\),
this reduces to the classical zero-reward normalization
\(r(s,a^\dagger)=0\) used in Rust's engine-replacement model
\citep{rust1987optimal}. More generally, taking
\(\mu(a\mid s)=1\{a=a^\dagger(s)\}\) allows state-specific anchors, where the
reference action may vary with the state. DDC models often use such anchors to
normalize the payoff of a no-action, status-quo, or outside option to be zero
or otherwise known
\citep{rust1987optimal,hotz1993conditional,aguirregabiria2010dynamic,
geng2020deep}. If \(\mu(a\mid s)=1/|\mathcal A|\), then
\((\mu r)(s)\) is the statewise average reward and, when \(g\equiv0\), the
normalization becomes a sum-to-zero normalization \citep{kallus2016revealed}.

The same notation also permits data-driven choices of \((\mu,g)\). For
example, a statewise value normalization is a special case. If \(V_r^\mu\) is
the value function of policy \(\mu\) under reward \(r\), then imposing
\(V_r^\mu=h\) is equivalent, by the Bellman equation
\[
    V_r^\mu = \mu r + \gamma \mu P V_r^\mu,
\]
to the anchor normalization \((\mu r)(s)=g(s)\) with
\(g:=h-\gamma\mu P h\). In medical decision problems, one may instead normalize
rewards relative to standard care using historical outcome data
\citep{kallus2018removing}. If an auxiliary dataset records states and realized
outcomes \((s_i,y_i)\) from a population with the same reward function and
behavior policy \(\pi\), then
\[
    g(s)=E[y\mid s]=\sum_{a\in\mathcal A}\pi(a\mid s)r(s,a)=(\pi r)(s).
\]
Taking \(\mu=\pi\) therefore yields a natural data-driven normalization.

\section{Relation to Classical Profiling}
\label{app:profiling}

Orthogonality induced by profiling is classical in semiparametric likelihood
theory
\citep{severini1992profile,murphy1997semiparametric,murphy1999observed,
murphy2000profile}, and related profiling ideas also appear in high-dimensional
debiased inference for sparse \(\ell^1\)-regularized regression
\citep{wang2025debiased}. At a high level, this orthogonality follows from the
first-order optimality of the profiled nuisance fit, or equivalently from an
envelope argument \citep{milgrom2002envelope}. Our setting is different: the
profiled object is not a likelihood nuisance, but the source block in a
regularized minimax Bellman system. Eliminating this block yields a
Schur-complement correction for the target Bellman KKT equation. This correction
removes the direct first-order source-residual channel, while leaving the
source-operator and target-empirical errors specific to the Bellman minimax
problem.

\section{Learning the Source Behavior Policy: End-to-End Bounds}
\label{sec:learned_behavior_policy}

The main analysis treats the source behavior policy \(\pi_{b1}\) as known, so
that the source soft-control signal
\[
u_g^\star(s,a)
:=
\log\frac{\pi_{b1}(a\mid s)}{\pi_{1,\rm ref}(a\mid s)}-g(s)
\]
is available. In practice, \(\pi_{b1}\) may itself be estimated from source
demonstrations. We now state the resulting end-to-end guarantee.

Let
\[
\mathcal D_b=\{(s_i^b,a_i^b)\}_{i=1}^{n_b}
\]
be i.i.d. source demonstrations with
\[
s_i^b\sim\rho_1^{\mathcal S},
\qquad
a_i^b\sim\pi_{b1}(\cdot\mid s_i^b).
\]
We estimate \(\pi_{b1}\) over a finite policy class \(\Pi_{b1}\) by empirical
cross-entropy minimization:
\[
\widehat\pi_{b1}
\in
\arg\min_{\pi\in\Pi_{b1}}
\frac1{n_b}\sum_{i=1}^{n_b}
-\log \pi(a_i^b\mid s_i^b).
\]
Assume realizability, \(\pi_{b1}\in\Pi_{b1}\), and assume that all candidate
policies are clipped below by \(\epsilon\in(0,1]\):
\[
\pi(a\mid s)\ge \epsilon,
\qquad
\forall \pi\in\Pi_{b1},\ (s,a).
\]
Then, with high probability, cross-entropy generalization controls the expected
conditional KL error
\[
\mathbb E_{s\sim\rho_1^{\mathcal S}}
\mathrm{KL}\!\left(
\pi_{b1}(\cdot\mid s)\,\middle\|\,\widehat\pi_{b1}(\cdot\mid s)
\right)
\le
\varepsilon_{\rm CE}(n_b,\delta),
\]
where
\[
\varepsilon_{\rm CE}(n_b,\delta)
:=
2\log(1/\epsilon)
\sqrt{\frac{2\log(|\Pi_{b1}|/\delta)}{n_b}}.
\]
Moreover, the clipping condition converts this expected KL control into an
\(L_2(\rho_1)\) source-signal error bound:
\[
\|\log\widehat\pi_{b1}-\log\pi_{b1}\|_{\rho_1}
\le
\left(
\frac{2}{\epsilon}\varepsilon_{\rm CE}(n_b,\delta)
\right)^{1/2}.
\]

To separate the transfer estimation error from the behavior-policy estimation
error, define
\[
\widehat u_g(s,a)
:=
\log\frac{\widehat\pi_{b1}(a\mid s)}{\pi_{1,\rm ref}(a\mid s)}-g(s).
\]
Let \((\widetilde q_1^\star,\widetilde q_2^\star)\) denote the population
transfer target obtained by replacing \(u_g^\star\) with \(\widehat u_g\), while
keeping the population transition kernels fixed:
\[
\widetilde q_1^\star
=
\widehat u_g+\gamma_1P_1^\mu\widetilde q_1^\star,
\]
and
\[
\widetilde q_2^\star
=
(I-\Pi_\mu)\widetilde q_1^\star+g+C
+
\gamma_2P_2\Omega_{\tau_2,\pi_{2,\rm ref}}(\widetilde q_2^\star).
\]
Then
\[
\|\widehat q_2-q_2^\star\|_{\rho_2}
\le
\|\widehat q_2-\widetilde q_2^\star\|_{\rho_2}
+
\|\widetilde q_2^\star-q_2^\star\|_{\rho_2}.
\]
The first term is the finite-sample transfer error conditional on the learned
behavior model. The second term is the additional bias caused by estimating
\(\pi_{b1}\).

\begin{theorem}[End-to-end transfer with learned behavior policy]
\label{thm:end_to_end_learned_pi_b1}
Fix \(\delta\in(0,1)\). Assume the conditions of
Theorem~\ref{thm:global_q2_bounds} hold for the plug-in population system
defined by
\[
\widehat u_g
=\log(\widehat\pi_{b1}/\pi_{1,\rm ref})-g.
\]
Assume also
\[
\left\|
\frac{d^{P_1}_{\mu,\rho_1}}{\rho_1}
\right\|_\infty
\le
\kappa_1,
\]
and the pathwise target concentrability condition
\[
\sup_{\pi\in\Pi_{\rm path}}
\left\|
\frac{d^{P_2}_{\pi,\rho_2}}{\rho_2}
\right\|_\infty
\le
\kappa_2^{\rm path}.
\]
Finally, assume the cross-measure and action-coverage conditions
\[
\chi_{21}:=
\left\|
\frac{\rho_2}{\rho_1}
\right\|_\infty
<\infty,
\qquad
\kappa_{\mu\mid\pi_{b1}}
:=
\left\|
\frac{\mu}{\pi_{b1}}
\right\|_\infty
<\infty .
\]
Then, with probability at least \(1-\delta\), for either transfer estimator
\[
{\rm alg}\in\{{\rm mod},{\rm coup}\},
\]
we have
\begin{equation}
\label{eq:end_to_end_q2_learned_pi_b1}
\|\widehat q_2^{\rm alg}-q_2^\star\|_{\rho_2}
\le
\sqrt{\mathcal B_{\rm alg}(\delta)}
+
\frac{
(1+\sqrt{\kappa_{\mu\mid\pi_{b1}}})
\sqrt{\kappa_1\kappa_2^{\rm path}\chi_{21}}
}{
(1-\gamma_1)(1-\gamma_2)
}
\left(
\frac{2}{\epsilon}\varepsilon_{\rm CE}(n_b,\delta)
\right)^{1/2}.
\end{equation}
Here \(\mathcal B_{\rm mod}(\delta)\) and
\(\mathcal B_{\rm coup}(\delta)\) are the global \(q_2\)-bounds from
Theorem~\ref{thm:global_q2_bounds}, applied to the plug-in population system.

In particular, if
\[
\kappa_{\mu\mid\pi_{b1}}\le \frac1{\epsilon_{b1}},
\]
then the second term in \eqref{eq:end_to_end_q2_learned_pi_b1} is bounded by
\[
\frac{
2\sqrt{\kappa_1\kappa_2^{\rm path}\chi_{21}}
}{
\sqrt{\epsilon_{b1}}(1-\gamma_1)(1-\gamma_2)
}
\left(
\frac{2}{\epsilon}\varepsilon_{\rm CE}(n_b,\delta)
\right)^{1/2}.
\]
\end{theorem}

The first term in \eqref{eq:end_to_end_q2_learned_pi_b1} is the statistical
error of the reward-transfer estimator conditional on the learned behavior
model. The second term is the cost of learning \(\pi_{b1}\). It propagates
through the source Bellman equation, the transferred reward map
\(q_1\mapsto (I-\Pi_\mu)q_1+g+C\), and the target soft Bellman equation.

Combining Theorem~\ref{thm:end_to_end_learned_pi_b1} with
Theorem~\ref{thm:policy_regret} yields the corresponding end-to-end policy
regret bound. Namely, under the assumptions of
Theorem~\ref{thm:policy_regret},
\[
J_2(\pi_2^\star)-J_2(\widehat\pi_2^{\rm alg})
\le
C_\pi
\left[
\sqrt{\mathcal B_{\rm alg}(\delta)}
+
\frac{
(1+\sqrt{\kappa_{\mu\mid\pi_{b1}}})
\sqrt{\kappa_1\kappa_2^{\rm path}\chi_{21}}
}{
(1-\gamma_1)(1-\gamma_2)
}
\left(
\frac{2}{\epsilon}\varepsilon_{\rm CE}(n_b,\delta)
\right)^{1/2}
\right].
\]

\section{Learning the Source Behavior Policy as an Additional Nuisance}
\label{sec:learned_behavior_policy_nuisance}

The preceding analysis treats the source behavior policy \(\pi_{b1}\) as known.
If \(\pi_{b1}\) is estimated from demonstrations and then plugged into
the source signal \(u_g^\star\), an additional first-order error enters through
\[
\log\widehat\pi_{b1}(a\mid s)-\log\pi_{b1}(a\mid s).
\]
The profiled coupled residual in
Section~\ref{sec:structural_orthogonality} removes the first-order propagation
of the source Bellman residual \(b_1(q_1)\), but it does not by itself remove
this policy-estimation error. We now explain how to extend the same profiling
idea to include \(\pi_{b1}\) as an additional nuisance component.

Let \(\pi\) be a candidate source behavior policy and define
\[
u_\pi(s,a):=\log\frac{\pi(a\mid s)}{\pi_{1,\rm ref}(a\mid s)}-g(s).
\]
The source residual becomes
\[
b_1(q_1,\pi)
:=
u_\pi+\gamma_1P_1^\mu q_1-q_1,
\]
while the target residual remains
\[
b_2(q_1,q_2)
:=
(I-\Pi_\mu)q_1+g+C+\gamma_2P_2\Omega(q_2)-q_2 .
\]
Let \(S_\pi(\pi)\) denote the population estimating equation for the behavior
policy, normalized so that
\[
S_\pi(\pi_{b1})=0.
\]
For example, for behavior cloning under cross-entropy loss, \(S_\pi(\pi)\) is
the population score equation associated with maximum likelihood estimation of
\(\pi_{b1}\).

Assume \(S_\pi\) is differentiable at \(\pi_{b1}\), and let
\[
\mathcal J_\pi:=D_\pi S_\pi(\pi_{b1}).
\]
Also define the derivative of the source soft-control signal
\[
\mathcal U_\pi[\Delta_\pi](s,a)
:=
D_\pi u_\pi\big|_{\pi=\pi_{b1}}[\Delta_\pi](s,a).
\]
In the nonparametric normalization \(S_\pi(\pi)=\pi-\pi_{b1}\), we have
\[
\mathcal J_\pi=I,
\qquad
\mathcal U_\pi[\Delta_\pi](s,a)
=
\frac{\Delta_\pi(a\mid s)}{\pi_{b1}(a\mid s)}.
\]

Including \(\pi\) as part of the nuisance block gives the augmented nuisance
system
\[
S_\pi(\pi)=0,
\qquad
b_1(q_1,\pi)=0.
\]
The corresponding Schur-complement correction yields the augmented profiled
target residual
\begin{equation}
\label{eq:orthogonalized_target_residual_with_policy}
\widetilde b_{2,\rm coup}^{\,\pi}(q_1,q_2,\pi)
=
b_2(q_1,q_2)
+
(I-\Pi_\mu)(I-\gamma_1P_1^\mu)^{-1}b_1(q_1,\pi)
-
(I-\Pi_\mu)(I-\gamma_1P_1^\mu)^{-1}
\mathcal U_\pi \mathcal J_\pi^{-1}S_\pi(\pi).
\end{equation}
The first correction term is the same source-residual profiling correction as in
\eqref{eq:orthogonalized_target_residual}. The second correction term is new:
it removes the first-order effect of estimating \(\pi_{b1}\).

Indeed, at the population truth \((q_1^\star,q_2^\star,\pi_{b1})\),
\[
D_{q_1}
\widetilde b_{2,\rm coup}^{\,\pi}
(q_1^\star,q_2^\star,\pi_{b1})
=0,
\qquad
D_{\pi}
\widetilde b_{2,\rm coup}^{\,\pi}
(q_1^\star,q_2^\star,\pi_{b1})
=0.
\]
Thus the augmented profiled target equation is locally insensitive, to first
order, to both the source value nuisance \(q_1\) and the learned behavior-policy
nuisance \(\pi\).

In the nonparametric normalization \(S_\pi(\pi)=\pi-\pi_{b1}\), the correction
can be written more explicitly as
\[
\widetilde b_{2,\rm coup}^{\,\pi}(q_1,q_2,\pi)
=
b_2(q_1,q_2)
+
(I-\Pi_\mu)(I-\gamma_1P_1^\mu)^{-1}b_1(q_1,\pi)
-
(I-\Pi_\mu)(I-\gamma_1P_1^\mu)^{-1}
\left[
\frac{\pi-\pi_{b1}}{\pi_{b1}}
\right],
\]
where the ratio is pointwise in \((s,a)\). Equivalently, using
\(b_1(q_1,\pi)=\log\pi-g+\gamma_1P_1^\mu q_1-q_1\), the \(q_1\)-terms cancel and
the residual depends on \(\pi\) through the locally centered signal
\[
\log\pi-g-\frac{\pi-\pi_{b1}}{\pi_{b1}}.
\]
This is the analogue of the orthogonalized residual in
Section~\ref{sec:structural_orthogonality}, now enlarged to account for
behavior-policy learning.

\section{Proofs for Section~\ref{sec:minimax_estimators}}
\label{app:minimax_estimators_proofs}

\subsection{Proof of Theorem~\ref{thm:population_well_posed_tightness}}

\begin{proof}
The first claim follows from the contraction mapping theorem. Indeed,
\[
\mathcal T_1q:=u_g^\star+\gamma_1P_1^\mu q
\]
is a \(\gamma_1\)-contraction in the sup norm, since \(P_1^\mu\) is a Markov
operator and \(\gamma_1<1\). Therefore \(\mathcal T_1\) has a unique fixed point
\(q_1^\star\).

For the second claim, fix \(q_1^\star\) and define
\[
\mathcal T_2q:=(I-\Pi_\mu)q_1^\star+g+C+\gamma_2P_2\Omega(q).
\]
The log-sum-exp operator \(\Omega_{\tau_2,\pi_{2,\rm ref}}\) is nonexpansive in the
sup norm:
\[
\|\Omega(q)-\Omega(q')\|_\infty\le \|q-q'\|_\infty.
\]
Since \(P_2\) is Markov, \(\mathcal T_2\) is a \(\gamma_2\)-contraction in the
sup norm. Hence it has a unique fixed point \(q_2^\star\).

We now prove tightness of the relaxed inequality formulation. Write
\[
r_C^\star:=(I-\Pi_\mu)q_1^\star+g+C,
\qquad
\mathcal T_2q:=r_C^\star+\gamma_2P_2\Omega(q).
\]
Since \(r_C^\star\ge0\), \(P_2\) is positive, and
\[
\Omega(0)(s)
=
\tau_2\log\sum_{a\in\mathcal A}\pi_{2,\rm ref}(a\mid s)
=
0,
\]
we have
\[
\mathcal T_2 0=r_C^\star\ge0.
\]
Moreover, \(\mathcal T_2\) is monotone. Starting value iteration from
\(q^{(0)}=0\),
\[
q^{(k+1)}=\mathcal T_2q^{(k)},
\]
therefore gives a nonnegative sequence converging to \(q_2^\star\). Hence
\[
q_2^\star\ge0.
\]

Let \(q_2\) be feasible for the relaxed constraint. Then
\[
b_2(q_1^\star,q_2)=\mathcal T_2q_2-q_2\le0,
\]
or equivalently \(q_2\ge \mathcal T_2q_2\). Thus \(q_2\) is a supersolution of
the monotone contraction \(\mathcal T_2\), and
\[
q_2\ge \mathcal T_2q_2\ge \mathcal T_2^2q_2\ge\cdots .
\]
Because \(\mathcal T_2^kq_2\to q_2^\star\), we obtain
\[
q_2\ge q_2^\star\ge0.
\]
Therefore every relaxed feasible \(q_2\) dominates \(q_2^\star\) pointwise, and
\[
\|q_2\|_{\rho_2}\ge \|q_2^\star\|_{\rho_2}.
\]
Since \(q_2^\star\) is itself feasible, it is a minimum-\(\rho_2\)-norm feasible
solution. If a feasible minimum-norm solution \(q_2\) had strict slack on a set
of positive \(\rho_2\)-measure, then it would strictly dominate the smallest
feasible supersolution \(q_2^\star\) on a set of positive \(\rho_2\)-measure, and
hence could not have the same \(\rho_2\)-norm. Thus every minimum-norm feasible
solution satisfies
\[
b_2(q_1^\star,q_2)=0
\qquad
\rho_2\text{-a.s.}
\]
The final claim follows by combining tightness with uniqueness of the source and
target fixed points.
\end{proof}

\subsection{Proof of Proposition~\ref{prop:population_equivalence}}

\begin{proof}
In the modular source minimax problem, dual maximization is finite only if
\(b_1(q_1)=0\). By Theorem~\ref{thm:population_well_posed_tightness}, this
equation has the unique solution \(q_1^\star\). Hence the finite primal optimizer
of the source module must be \(q_1^\star\).

Given \(q_1^\star\), the modular target minimax problem is finite only if
\(b_2(q_1^\star,q_2)=0\). Again by
Theorem~\ref{thm:population_well_posed_tightness}, this equation has the unique
solution \(q_2^\star\). Hence the modular population procedure has primal
solution \((q_1^\star,q_2^\star)\).

Similarly, in the coupled minimax problem, dual maximization is finite only if
both
\[
b_1(q_1)=0,
\qquad
b_2(q_1,q_2)=0
\]
hold. The coupled transfer system has the unique solution
\((q_1^\star,q_2^\star)\). Therefore the coupled population minimax problem has
the same primal solution. The value of \(\beta\) changes the relative weighting
of the source primal regularization term, but it does not change the population
primal solution under exact residual satisfaction.
\end{proof}

\subsection{Proof of Theorem~\ref{thm:population_saddle_kkt}}

\begin{proof}
Primal feasibility is exactly the population coupled transfer system, and hence
holds by definition of \((q_1^\star,q_2^\star)\). The adjoint identities
\eqref{eq:dual_adjoint_l1}--\eqref{eq:dual_adjoint_l2} are the weak forms of
the stationarity equations in \eqref{eq:kkt_stationarity_explicit}; equivalently,
they imply
\[
D_{q_1}\mathcal L_{\rm coup}^{\beta}
(q_1^\star,q_2^\star,l_1^\star,l_2^\star)=0,
\qquad
D_{q_2}\mathcal L_{\rm coup}^{\beta}
(q_1^\star,q_2^\star,l_1^\star,l_2^\star)=0.
\]

We prove the quadratic growth inequality. Let
\[
\Delta_1:=q_1-q_1^\star,
\qquad
\Delta_2:=q_2-q_2^\star.
\]
Using
\[
b_1(q_1)-b_1(q_1^\star)=-(I-\gamma_1P_1^\mu)\Delta_1
\]
and
\[
b_2(q_1,q_2)-b_2(q_1^\star,q_2^\star)
=
(I-\Pi_\mu)\Delta_1
+
\gamma_2P_2\{\Omega(q_2)-\Omega(q_2^\star)\}
-
\Delta_2,
\]
we expand
\[
\begin{aligned}
&
\mathcal L_{\rm coup}^{\beta}(q_1,q_2,l_1^\star,l_2^\star)
-
\mathcal L_{\rm coup}^{\beta}(q_1^\star,q_2^\star,l_1^\star,l_2^\star)
\\
&=
\frac{\beta}{2}\|\Delta_1\|_{\rho_1}^2
+
\Big[
\beta\langle q_1^\star,\Delta_1\rangle_{\rho_1}
-
\langle l_1^\star,(I-\gamma_1P_1^\mu)\Delta_1\rangle_{\rho_1}
+
\langle l_2^\star,(I-\Pi_\mu)\Delta_1\rangle_{\rho_2}
\Big]
\\
&\quad+
\frac12\|\Delta_2\|_{\rho_2}^2
+
\Big[
\langle q_2^\star,\Delta_2\rangle_{\rho_2}
-
\langle l_2^\star,\Delta_2\rangle_{\rho_2}
+
\gamma_2
\langle l_2^\star,
P_2\{\Omega(q_2)-\Omega(q_2^\star)\}
\rangle_{\rho_2}
\Big].
\end{aligned}
\]
The bracketed source term is zero by \eqref{eq:dual_adjoint_l1} with
\(h=\Delta_1\). For the target term, convexity of \(\Omega\) gives
\[
\Omega(q_2)-\Omega(q_2^\star)\ge D\Omega(q_2^\star)[\Delta_2].
\]
Since \(P_2\) is positive and \(l_2^\star\ge0\), the bracketed target term is at
least
\[
\langle q_2^\star,\Delta_2\rangle_{\rho_2}
-
\left\langle l_2^\star,
(I-\gamma_2P_2^{\pi_2^\star})\Delta_2
\right\rangle_{\rho_2},
\]
which is zero by \eqref{eq:dual_adjoint_l2} with \(h=\Delta_2\). Therefore
\[
\mathcal L_{\rm coup}^{\beta}(q_1,q_2,l_1^\star,l_2^\star)
-
\mathcal L_{\rm coup}^{\beta}(q_1^\star,q_2^\star,l_1^\star,l_2^\star)
\ge
\frac{\beta}{2}\|\Delta_1\|_{\rho_1}^2
+
\frac12\|\Delta_2\|_{\rho_2}^2.
\]

The saddle property follows because primal feasibility makes
\(\mathcal L_{\rm coup}^{\beta}(q_1^\star,q_2^\star,l_1,l_2)\) independent of
\((l_1,l_2)\), while the inequality above shows that
\((q_1^\star,q_2^\star)\) minimizes
\(\mathcal L_{\rm coup}^{\beta}(\cdot,\cdot,l_1^\star,l_2^\star)\). Hence
\[
\mathcal L_{\rm coup}^{\beta}
(q_1^\star,q_2^\star,l_1,l_2)
\le
\mathcal L_{\rm coup}^{\beta}
(q_1^\star,q_2^\star,l_1^\star,l_2^\star)
\le
\mathcal L_{\rm coup}^{\beta}
(q_1,q_2,l_1^\star,l_2^\star),
\]
which proves the saddle property.
\end{proof}

\subsection{Proof of dual representations in Proposition~\ref{prop:dual_representation_coverage}}

\begin{proof}
The target stationarity identity \eqref{eq:dual_adjoint_l2} gives
\[
\rho_2\odot l_2^\star
=
\bigl(I-\gamma_2(P_2^{\pi_2^\star})^\top\bigr)^{-1}
(\rho_2\odot q_2^\star).
\]
The modular source identity
\[
\left\langle l_{1,{\rm mod}}^\star,
(I-\gamma_1P_1^\mu)h
\right\rangle_{\rho_1}
=
\langle q_1^\star,h\rangle_{\rho_1},
\qquad
\forall h,
\]
gives the representation for \(l_{1,{\rm mod}}^\star\). Finally, the coupled
source identity \eqref{eq:dual_adjoint_l1} gives
\[
\rho_1\odot l_{1,{\rm coup}}^\star
=
\bigl(I-\gamma_1(P_1^\mu)^\top\bigr)^{-1}
\left(
\beta\rho_1\odot q_1^\star
+
(I-\Pi_\mu)^\top(\rho_2\odot l_2^\star)
\right).
\]
Dividing by the corresponding sampling measures gives the displayed formulas,
and the self/cross decomposition follows by linearity.
\end{proof}

\subsection{Proof of dual boundedness in Proposition~\ref{prop:dual_boundedness_coverage}}

\begin{proof}
By definition,
\[
\bigl(I-\gamma_2(P_2^{\pi_2^\star})^\top\bigr)^{-1}\rho_2
=
\frac{1}{1-\gamma_2}
d^{P_2}_{\pi_2^\star,\rho_2}.
\]
Since the resolvent is positive and
\(\rho_2\odot |q_2^\star|\le B_{Q2}\rho_2\),
\[
\left|
\bigl(I-\gamma_2(P_2^{\pi_2^\star})^\top\bigr)^{-1}
(\rho_2\odot q_2^\star)
\right|
\le
\frac{B_{Q2}}{1-\gamma_2}
d^{P_2}_{\pi_2^\star,\rho_2}.
\]
Dividing by \(\rho_2\) and applying the target coverage assumption gives
\[
\|l_2^\star\|_\infty\le B_{L2}.
\]

Similarly,
\[
\bigl(I-\gamma_1(P_1^\mu)^\top\bigr)^{-1}\rho_1
=
\frac{1}{1-\gamma_1}
d^{P_1}_{\mu,\rho_1}.
\]
Since \(\rho_1\odot |q_1^\star|\le B_{Q1}\rho_1\), positivity and source
coverage yield
\[
\|l_{1,{\rm mod}}^\star\|_\infty\le B_{L1}^{\rm self}.
\]

It remains to control \(l_{1,{\rm cross}}^\star\). Let
\[
m_2
:=
\bigl(I-\gamma_2(P_2^{\pi_2^\star})^\top\bigr)^{-1}
(\rho_2\odot q_2^\star).
\]
The preceding target bound gives
\[
|m_2|
\le
\frac{B_{Q2}}{1-\gamma_2}
d^{P_2}_{\pi_2^\star,\rho_2}.
\]
Since \(I-(\Pi_\mu)^\top\) is not positive,
\[
\left|(I-\Pi_\mu)^\top m_2\right|
\le
|m_2|+(\Pi_\mu)^\top |m_2|.
\]
Thus
\[
\left|(I-\Pi_\mu)^\top m_2\right|
\le
\frac{B_{Q2}}{1-\gamma_2}
\left(
d^{P_2}_{\pi_2^\star,\rho_2}
+
(\Pi_\mu)^\top d^{P_2}_{\pi_2^\star,\rho_2}
\right).
\]
The cross-measure coverage assumption gives
\[
d^{P_2}_{\pi_2^\star,\rho_2}\le\kappa_{12}\rho_1.
\]
For any nonnegative measure \(m\),
\[
((\Pi_\mu)^\top m)(s,a)=\mu(a\mid s)\sum_{a'}m(s,a').
\]
Therefore,
\[
\begin{aligned}
\big((\Pi_\mu)^\top d^{P_2}_{\pi_2^\star,\rho_2}\big)(s,a)
&=
\mu(a\mid s)d^{P_2,\mathcal S}_{\pi_2^\star,\rho_2}(s)\\
&\le
\kappa_{\mu\mid\pi_{b1}}\kappa_{12}^{\mathcal S}\rho_1(s,a).
\end{aligned}
\]
Hence
\[
d^{P_2}_{\pi_2^\star,\rho_2}
+
(\Pi_\mu)^\top d^{P_2}_{\pi_2^\star,\rho_2}
\le
\left(
\kappa_{12}
+
\kappa_{\mu\mid\pi_{b1}}\kappa_{12}^{\mathcal S}
\right)\rho_1.
\]
It follows that
\[
\left|(I-\Pi_\mu)^\top m_2\right|
\le
\frac{B_{Q2}}{1-\gamma_2}
\left(
\kappa_{12}
+
\kappa_{\mu\mid\pi_{b1}}\kappa_{12}^{\mathcal S}
\right)
\rho_1.
\]
Applying the positive source resolvent and using source coverage gives
\[
\|l_{1,{\rm cross}}^\star\|_\infty
\le
\frac{
\kappa_1
\left(
\kappa_{12}
+
\kappa_{\mu\mid\pi_{b1}}\kappa_{12}^{\mathcal S}
\right)
}{
(1-\gamma_1)(1-\gamma_2)
}
B_{Q2}
=
B_{L12}.
\]
The self component contributes \(\beta B_{L1}^{\rm self}\), so
\[
\|l_{1,{\rm coup}}^\star\|_\infty
\le
\beta B_{L1}^{\rm self}+B_{L12}.
\]
\end{proof}

\subsection{Proof of Corollary~\ref{cor:first_order_modular_coupled_contrast}}

\begin{proof}
We expand the coupled Lagrangian around
\((q_1^\star,q_2^\star)\), keeping the dual variables fixed. Since
\[
b_1(q_1)
=
u_g^\star-(I-\gamma_1P_1^\mu)q_1,
\]
and
\[
b_2(q_1,q_2)
=
(I-\Pi_\mu)q_1+g+C+\gamma_2P_2\Omega(q_2)-q_2,
\]
we have
\[
b_1(q_1^\star+\Delta_1)-b_1(q_1^\star)
=
-(I-\gamma_1P_1^\mu)\Delta_1,
\]
and
\[
b_2(q_1^\star+\Delta_1,q_2^\star+\Delta_2)
-
b_2(q_1^\star,q_2^\star)
=
(I-\Pi_\mu)\Delta_1
+
\gamma_2P_2\bigl(\Omega(q_2^\star+\Delta_2)-\Omega(q_2^\star)\bigr)
-
\Delta_2.
\]
Using the polarization identity for the quadratic terms and Taylor expanding
\(\Omega\) around \(q_2^\star\), for arbitrary fixed source dual \(l_1\) and
target dual \(l_2^\star\),
\[
\begin{aligned}
&
\mathcal L_{\rm coup}^{\beta}
(q_1^\star+\Delta_1,q_2^\star+\Delta_2,l_1,l_2^\star)
-
\mathcal L_{\rm coup}^{\beta}
(q_1^\star,q_2^\star,l_1,l_2^\star)
\\
&=
\Big[
\beta\langle q_1^\star,\Delta_1\rangle_{\rho_1}
-
\left\langle l_1,(I-\gamma_1P_1^\mu)\Delta_1\right\rangle_{\rho_1}
+
\left\langle l_2^\star,(I-\Pi_\mu)\Delta_1\right\rangle_{\rho_2}
\Big]
\\
&\quad
+
\Big[
\langle q_2^\star,\Delta_2\rangle_{\rho_2}
-
\langle l_2^\star,\Delta_2\rangle_{\rho_2}
+
\gamma_2
\langle l_2^\star,P_2D\Omega(q_2^\star)[\Delta_2]\rangle_{\rho_2}
\Big]
\\
&\quad
+
\frac{\beta}{2}\|\Delta_1\|_{\rho_1}^2
+
\frac12\|\Delta_2\|_{\rho_2}^2
+
R_\Omega(\Delta_2).
\end{aligned}
\]
The target-direction first-order term vanishes by
\eqref{eq:dual_adjoint_l2}. For coupled transfer, the source-direction
first-order term vanishes by \eqref{eq:dual_adjoint_l1}. This proves the
coupled expansion.

For modular transfer, \(l_{1,\rm mod}^\star\) satisfies only the source-stage
adjoint equation
\[
\left\langle l_{1,\rm mod}^\star,
(I-\gamma_1P_1^\mu)h
\right\rangle_{\rho_1}
=
\beta\langle q_1^\star,h\rangle_{\rho_1},
\qquad
\forall h.
\]
Thus the first two source terms cancel, but the target coupling term remains:
\[
\left\langle l_2^\star,(I-\Pi_\mu)\Delta_1\right\rangle_{\rho_2}.
\]
This gives the modular expansion.
\end{proof}

\section{Proofs for Section~\ref{sec:structural_orthogonality}}
\label{app:profile_map}

This appendix gives the formal justification for the profiling argument used in
Section~\ref{sec:structural_orthogonality}. In the main text we introduced the
Schur-complement equation
\[
\Psi_{\rm coup}(\theta,\eta)
=
S_T(\theta,\eta)
-
D_\eta S_T(\theta^\star,\eta^\star)
\bigl[D_\eta S_N(\theta^\star,\eta^\star)\bigr]^{-1}
S_N(\theta,\eta),
\]
and interpreted it as the target KKT equation after eliminating the source
block. We first establish local solvability of the source block, then show that
\(\Psi_{\rm coup}\) is an off-profile extension of the exact reduced target
equation with zero first-order nuisance derivative.

Recall the block notation
\[
\eta=(q_1,l_1),
\qquad
\theta=(q_2,l_2),
\]
and the population KKT system
\[
S_N(\theta,\eta)=0,
\qquad
S_T(\theta,\eta)=0.
\]

\subsection{Invertibility of the source KKT block}

\begin{lemma}[Invertibility of the nuisance Jacobian]
\label{lem:secSO_SN_invertible}
Assume \(\gamma_1<1\) and \(\|P_1^\mu\|\le1\). Then
\(I-\gamma_1P_1^\mu\) is invertible. Moreover,
\[
D_\eta S_N(\theta^\star,\eta^\star)
=
\begin{pmatrix}
\beta I & -(I-\gamma_1P_1^\mu)^\top\\
-(I-\gamma_1P_1^\mu) & 0
\end{pmatrix},
\]
and this block operator is invertible.
\end{lemma}

\begin{proof}
Since \(\|P_1^\mu\|\le1\) and \(\gamma_1<1\),
\[
\|\gamma_1P_1^\mu\|\le\gamma_1<1.
\]
Hence
\[
(I-\gamma_1P_1^\mu)^{-1}
=
\sum_{t=0}^{\infty}(\gamma_1P_1^\mu)^t
\]
is well-defined.

Let
\[
A_1:=I-\gamma_1P_1^\mu.
\]
Since
\[
b_1(q_1)=u_g^\star-A_1q_1,
\]
and
\[
D_{q_1}\mathcal L_{\rm coup}^{\beta}
=
\beta q_1-A_1^\top l_1+(I-\Pi_\mu)^\top l_2,
\]
we have
\[
D_\eta S_N(\theta^\star,\eta^\star)
=
\begin{pmatrix}
\beta I & -A_1^\top\\
-A_1 & 0
\end{pmatrix}.
\]
To prove invertibility, consider
\[
\begin{pmatrix}
\beta I & -A_1^\top\\
-A_1 & 0
\end{pmatrix}
\begin{pmatrix}
u\\ v
\end{pmatrix}
=
\begin{pmatrix}
r_1\\ r_2
\end{pmatrix}.
\]
The second block equation gives
\[
u=-A_1^{-1}r_2.
\]
Substituting this into the first block equation gives
\[
v
=
-(A_1^\top)^{-1}r_1
-
(A_1^\top)^{-1}\beta A_1^{-1}r_2.
\]
Thus \((u,v)\) is uniquely determined for every \((r_1,r_2)\). Therefore
\(D_\eta S_N(\theta^\star,\eta^\star)\) is invertible.
\end{proof}

\subsection{Proof of Theorem~\ref{thm:profiled_target_orthogonality}}

\begin{proof}
The invertibility of \(I-\gamma_1P_1^\mu\) and
\(D_\eta S_N(\theta^\star,\eta^\star)\) follows from
Lemma~\ref{lem:secSO_SN_invertible}. It remains to identify the
Schur-complement equation and verify the orthogonality identity.

Let
\[
\Delta\theta:=\theta-\theta^\star,
\qquad
\Delta\eta:=\eta-\eta^\star.
\]
Since
\[
S_N(\theta^\star,\eta^\star)=0,
\qquad
S_T(\theta^\star,\eta^\star)=0,
\]
the first-order expansion of the source block is
\[
S_N(\theta,\eta)
=
D_\theta S_N\,\Delta\theta
+
D_\eta S_N\,\Delta\eta
+
o(\|\Delta\theta\|+\|\Delta\eta\|),
\]
where all derivatives are evaluated at \((\theta^\star,\eta^\star)\).
Solving this linearized equation for \(\Delta\eta\) gives
\[
\Delta\eta
=
(D_\eta S_N)^{-1}S_N(\theta,\eta)
-
(D_\eta S_N)^{-1}D_\theta S_N\,\Delta\theta
+
o(\|\Delta\theta\|+\|\Delta\eta\|).
\]
The target block expansion is
\[
S_T(\theta,\eta)
=
D_\theta S_T\,\Delta\theta
+
D_\eta S_T\,\Delta\eta
+
o(\|\Delta\theta\|+\|\Delta\eta\|).
\]
Substituting the previous display yields
\[
\begin{aligned}
&
S_T(\theta,\eta)
-
D_\eta S_T(D_\eta S_N)^{-1}S_N(\theta,\eta)
\\
&\quad=
\left[
D_\theta S_T
-
D_\eta S_T(D_\eta S_N)^{-1}D_\theta S_N
\right]\Delta\theta
+
o(\|\Delta\theta\|+\|\Delta\eta\|).
\end{aligned}
\]
Therefore
\[
\Psi_{\rm coup}(\theta,\eta)
=
S_T(\theta,\eta)
-
D_\eta S_T(\theta^\star,\eta^\star)
\bigl[D_\eta S_N(\theta^\star,\eta^\star)\bigr]^{-1}
S_N(\theta,\eta)
\]
is precisely the Schur-complement target equation obtained by linearized
elimination of the source block.

Finally, because the correction operators in \(\Psi_{\rm coup}\) are fixed at
the truth,
\[
\begin{aligned}
D_\eta\Psi_{\rm coup}(\theta^\star,\eta^\star)
&=
D_\eta S_T(\theta^\star,\eta^\star)
-
D_\eta S_T(\theta^\star,\eta^\star)
\bigl[D_\eta S_N(\theta^\star,\eta^\star)\bigr]^{-1}
D_\eta S_N(\theta^\star,\eta^\star)
\\
&=0.
\end{aligned}
\]
This proves the theorem.
\end{proof}

\subsection{Local profile map and exact reduced target equation}

The previous theorem uses a linearized source-block elimination. We next record
the exact local profile map that justifies viewing this calculation as a local
Schur-complement reduction of the KKT system.

\begin{proposition}[Local profile map for the source block]
\label{prop:secSO_profile_map}
Assume \(S_N\) is continuously Fr\'echet differentiable in a neighborhood of
\((\theta^\star,\eta^\star)\), and assume the conditions of
Lemma~\ref{lem:secSO_SN_invertible}. Then there exist a neighborhood
\(\mathcal U\) of \(\theta^\star\) and a unique \(C^1\) map
\[
\eta(\theta)
\]
such that
\[
S_N(\theta,\eta(\theta))=0,
\qquad
\eta(\theta^\star)=\eta^\star,
\qquad
\forall\theta\in\mathcal U.
\]
\end{proposition}

\begin{proof}
This is the Banach-space implicit function theorem applied to
\(S_N(\theta,\eta)=0\). The equation holds at
\((\theta^\star,\eta^\star)\), and
\(D_\eta S_N(\theta^\star,\eta^\star)\) is invertible by
Lemma~\ref{lem:secSO_SN_invertible}. Hence the required \(C^1\) local profile
map exists and is unique.
\end{proof}

Define the exact reduced target equation by
\[
\widetilde S_{\rm coup}(\theta)
:=
S_T(\theta,\eta(\theta)).
\]
This is the exact local target equation obtained by solving the source KKT block
and substituting the solution into the target KKT block.

\begin{lemma}[Agreement on the profile manifold]
\label{lem:secSO_agreement}
For every \(\theta\in\mathcal U\),
\[
\Psi_{\rm coup}(\theta,\eta(\theta))
=
\widetilde S_{\rm coup}(\theta).
\]
\end{lemma}

\begin{proof}
By the definition of the profile map,
\[
S_N(\theta,\eta(\theta))=0.
\]
Therefore, when \(\Psi_{\rm coup}\) is evaluated at
\((\theta,\eta(\theta))\), the correction term vanishes:
\[
\Psi_{\rm coup}(\theta,\eta(\theta))
=
S_T(\theta,\eta(\theta)).
\]
By definition,
\[
S_T(\theta,\eta(\theta))
=
\widetilde S_{\rm coup}(\theta).
\]
\end{proof}

Although the correction term vanishes on the profile manifold, it is not
irrelevant. Its role is to define an off-manifold extension of the target
equation whose first-order sensitivity to nuisance perturbations vanishes at
the truth.

\begin{proposition}[Derivative of the reduced target equation]
\label{prop:secSO_reduced_derivative}
The derivative of the reduced target equation at the truth satisfies
\[
D\widetilde S_{\rm coup}(\theta^\star)
=
D_\theta\Psi_{\rm coup}(\theta^\star,\eta^\star).
\]
Equivalently,
\[
D\widetilde S_{\rm coup}(\theta^\star)
=
D_\theta S_T(\theta^\star,\eta^\star)
-
D_\eta S_T(\theta^\star,\eta^\star)
\bigl[
D_\eta S_N(\theta^\star,\eta^\star)
\bigr]^{-1}
D_\theta S_N(\theta^\star,\eta^\star).
\]
\end{proposition}

\begin{proof}
Differentiate the profile identity
\[
S_N(\theta,\eta(\theta))=0
\]
at \(\theta^\star\). This gives
\[
D_\theta S_N
+
D_\eta S_N\,D\eta(\theta^\star)
=
0,
\]
and hence
\[
D\eta(\theta^\star)
=
-
(D_\eta S_N)^{-1}D_\theta S_N,
\]
where all derivatives are evaluated at \((\theta^\star,\eta^\star)\).
Differentiating
\[
\widetilde S_{\rm coup}(\theta)=S_T(\theta,\eta(\theta))
\]
at \(\theta^\star\) yields
\[
D\widetilde S_{\rm coup}(\theta^\star)
=
D_\theta S_T
+
D_\eta S_T\,D\eta(\theta^\star).
\]
Substituting the expression for \(D\eta(\theta^\star)\) gives
\[
D\widetilde S_{\rm coup}(\theta^\star)
=
D_\theta S_T
-
D_\eta S_T(D_\eta S_N)^{-1}D_\theta S_N.
\]
On the other hand, differentiating the definition of \(\Psi_{\rm coup}\) with
respect to \(\theta\) gives the same expression:
\[
D_\theta\Psi_{\rm coup}(\theta^\star,\eta^\star)
=
D_\theta S_T
-
D_\eta S_T(D_\eta S_N)^{-1}D_\theta S_N.
\]
Therefore
\[
D\widetilde S_{\rm coup}(\theta^\star)
=
D_\theta\Psi_{\rm coup}(\theta^\star,\eta^\star).
\]
\end{proof}

\subsection{Nuisance orthogonality}

For completeness, we restate the nuisance orthogonality in the form used in the
main text. By definition,
\[
\Psi_{\rm coup}(\theta,\eta)
=
S_T(\theta,\eta)
-
D_\eta S_T(\theta^\star,\eta^\star)
\bigl[
D_\eta S_N(\theta^\star,\eta^\star)
\bigr]^{-1}
S_N(\theta,\eta).
\]
Because the correction operators are fixed at the truth,
\[
\begin{aligned}
D_\eta\Psi_{\rm coup}(\theta^\star,\eta^\star)
&=
D_\eta S_T
-
D_\eta S_T(D_\eta S_N)^{-1}D_\eta S_N
\\
&=0,
\end{aligned}
\]
where all derivatives are evaluated at \((\theta^\star,\eta^\star)\). This
identity explains the role of the Schur-complement correction: the exact
profiled equation \(S_T(\theta,\eta(\theta))\) lives on the profile manifold,
whereas \(\Psi_{\rm coup}\) is an off-manifold extension of the same equation
whose first-order sensitivity to nuisance perturbations vanishes at the truth.

\subsection{Proof of Corollary~\ref{cor:orthogonalized_target_residual}}

\begin{proof}
The residual component of the target KKT block is \(b_2(q_1,q_2)\), and the
source residual component is \(b_1(q_1)\). Since
\[
D_{q_1}b_2(q_1,q_2)=I-\Pi_\mu,
\qquad
D_{q_1}b_1(q_1)=-(I-\gamma_1P_1^\mu),
\]
the Schur-complement correction to the target residual is
\[
-
D_{q_1}b_2
\bigl[D_{q_1}b_1\bigr]^{-1}b_1
=
-
(I-\Pi_\mu)
\bigl[-(I-\gamma_1P_1^\mu)\bigr]^{-1}
b_1.
\]
Therefore
\[
\widetilde b_{2,\rm coup}(q_1,q_2)
=
b_2(q_1,q_2)
+
(I-\Pi_\mu)(I-\gamma_1P_1^\mu)^{-1}b_1(q_1),
\]
which proves \eqref{eq:orthogonalized_target_residual}.

Differentiating with respect to \(q_1\),
\[
\begin{aligned}
D_{q_1}\widetilde b_{2,\rm coup}
&=
(I-\Pi_\mu)
+
(I-\Pi_\mu)(I-\gamma_1P_1^\mu)^{-1}
\bigl[-(I-\gamma_1P_1^\mu)\bigr]
\\
&=0.
\end{aligned}
\]
In particular,
\[
D_{q_1}\widetilde b_{2,\rm coup}(q_1^\star,q_2^\star)=0.
\]
\end{proof}

\section{Proofs for Section \ref{sec:error_propagation}}
\label{app:first_order_error_propagation}

This appendix proves the first-order decompositions in
Section~\ref{sec:error_propagation}. The main point is to separate two source
uncertainty channels:
\[
\text{operator estimation error through } \widehat P_1^\mu-P_1^\mu,
\qquad
\text{residual error through } \widehat b_1(\widehat q_1).
\]
The modular target equation contains both channels at first order. The coupled
profiled equation cancels the residual channel, but still contains the source
operator channel.

\subsection{Regularity conditions}

We use the following local conditions.

\begin{assumption}[Local regularity for first-order expansions]
\label{ass:app_first_order_regularity}
The following hold in a neighborhood of
\((q_1^\star,q_2^\star)\).

\begin{enumerate}
    \item The operators
    \[
    I-\gamma_1\widehat P_1^\mu,
    \qquad
    I-\gamma_1P_1^\mu,
    \qquad
    I-\gamma_2P_2^{\pi_2^\star}
    \]
    are invertible with probability tending to one.

    \item The soft value operator
    \(\Omega_{\tau_2,\pi_{2,\rm ref}}\) is twice Fr\'echet differentiable around
    \(q_2^\star\). In particular,
    \[
    \Omega_{\tau_2,\pi_{2,\rm ref}}(q_2^\star+\Delta_2)
    -
    \Omega_{\tau_2,\pi_{2,\rm ref}}(q_2^\star)
    =
    D\Omega_{\tau_2,\pi_{2,\rm ref}}(q_2^\star)[\Delta_2]
    +
    R_\Omega(\Delta_2),
    \]
    with
    \[
    R_\Omega(\Delta_2)=O(\|\Delta_2\|^2).
    \]

    \item The local empirical-process increments generated by evaluating
    empirical transition operators at estimated functions rather than at the
    truth are negligible at the first-order scale. Concretely, the remainders
    denoted below by \(R_{\rm emp}^{\rm mod}\) and \(R_{\rm emp}^{\rm coup}\)
    satisfy the required \(o_p\)-bounds relative to the displayed leading
    empirical perturbations.

    \item The higher-order terms from replacing population profiling operators
    by empirical profiling operators are negligible at the first-order scale.
\end{enumerate}
\end{assumption}

\subsection{Source-stage decomposition}

\begin{proof}[Proof of Lemma~\ref{lem:source_error_decomposition}]
Since \(q_1^\star\) solves the population source equation,
\[
u_g^\star+\gamma_1P_1^\mu q_1^\star-q_1^\star=0,
\]
we have
\[
u_g^\star=(I-\gamma_1P_1^\mu)q_1^\star.
\]
For any source estimate \(\widehat q_1\),
\[
\widehat b_1(\widehat q_1)
=
u_g^\star+\gamma_1\widehat P_1^\mu \widehat q_1-\widehat q_1
=
u_g^\star-(I-\gamma_1\widehat P_1^\mu)\widehat q_1.
\]
Thus
\[
(I-\gamma_1\widehat P_1^\mu)\widehat q_1
=
u_g^\star-\widehat b_1(\widehat q_1).
\]
Subtract
\[
(I-\gamma_1\widehat P_1^\mu)q_1^\star
\]
from both sides. This gives
\[
(I-\gamma_1\widehat P_1^\mu)(\widehat q_1-q_1^\star)
=
u_g^\star
-
(I-\gamma_1\widehat P_1^\mu)q_1^\star
-
\widehat b_1(\widehat q_1).
\]
Using
\[
u_g^\star=(I-\gamma_1P_1^\mu)q_1^\star,
\]
the first two terms on the right-hand side become
\[
(I-\gamma_1P_1^\mu)q_1^\star
-
(I-\gamma_1\widehat P_1^\mu)q_1^\star
=
\gamma_1(\widehat P_1^\mu-P_1^\mu)q_1^\star.
\]
Therefore
\[
(I-\gamma_1\widehat P_1^\mu)(\widehat q_1-q_1^\star)
=
\gamma_1(\widehat P_1^\mu-P_1^\mu)q_1^\star
-
\widehat b_1(\widehat q_1).
\]
Multiplying by
\[
(I-\gamma_1\widehat P_1^\mu)^{-1}
\]
proves
\[
\widehat q_1-q_1^\star
=
\gamma_1
(I-\gamma_1\widehat P_1^\mu)^{-1}
(\widehat P_1^\mu-P_1^\mu)q_1^\star
-
(I-\gamma_1\widehat P_1^\mu)^{-1}
\widehat b_1(\widehat q_1).
\]
\end{proof}

\subsection{Modular target decomposition}

We first prove the modular expansion. The modular target estimator treats
\(\widehat q_1^{\rm mod}\) as fixed and solves the target equation using the
empirical target transition operator.

Define
\[
\Delta_1^{\rm mod}
:=
\widehat q_1^{\rm mod}-q_1^\star,
\qquad
\Delta_2^{\rm mod}
:=
\widehat q_2^{\rm mod}-q_2^\star.
\]
The empirical target residual in the modular stage is
\[
\widehat b_2(\widehat q_1^{\rm mod},q_2)
=
(I-\Pi_\mu)\widehat q_1^{\rm mod}+g+C
+
\gamma_2\widehat P_2\Omega_{\tau_2,\pi_{2,\rm ref}}(q_2)
-
q_2.
\]

\begin{proposition}[Modular first-order decomposition]
\label{prop:app_mod_first_order_decomposition}
Under Assumption~\ref{ass:app_first_order_regularity},
\[
\begin{aligned}
&
(I-\gamma_2P_2^{\pi_2^\star})\Delta_2^{\rm mod}
\\
&=
\gamma_1
(I-\Pi_\mu)
(I-\gamma_1\widehat P_1^\mu)^{-1}
(\widehat P_1^\mu-P_1^\mu)q_1^\star
-
(I-\Pi_\mu)
(I-\gamma_1\widehat P_1^\mu)^{-1}
\widehat b_1(\widehat q_1^{\rm mod})
\\
&\quad
+
\gamma_2(\widehat P_2-P_2)
\Omega_{\tau_2,\pi_{2,\rm ref}}(q_2^\star)
+
\mathrm{Rem}_{\rm mod}.
\end{aligned}
\]
\end{proposition}

\begin{proof}
The population target equation is
\[
b_2(q_1^\star,q_2^\star)
=
(I-\Pi_\mu)q_1^\star+g+C
+
\gamma_2P_2\Omega_{\tau_2,\pi_{2,\rm ref}}(q_2^\star)
-
q_2^\star
=
0.
\]
The modular estimator solves the empirical target equation up to a local
empirical residual. Thus,
\[
0
=
\widehat b_2(\widehat q_1^{\rm mod},\widehat q_2^{\rm mod})
+
\text{local target optimality error}.
\]
Subtract the population target equation at
\((q_1^\star,q_2^\star)\). The leading expansion is
\[
\begin{aligned}
0
&=
(I-\Pi_\mu)(\widehat q_1^{\rm mod}-q_1^\star)
+
\gamma_2(\widehat P_2-P_2)
\Omega_{\tau_2,\pi_{2,\rm ref}}(q_2^\star)
\\
&\quad
+
\gamma_2P_2
\left[
\Omega_{\tau_2,\pi_{2,\rm ref}}(\widehat q_2^{\rm mod})
-
\Omega_{\tau_2,\pi_{2,\rm ref}}(q_2^\star)
\right]
-
(\widehat q_2^{\rm mod}-q_2^\star)
+
R_{\rm emp}^{\rm mod}.
\end{aligned}
\]
Here \(R_{\rm emp}^{\rm mod}\) collects local empirical-process increments and
any local target optimality error.

By differentiability of the soft value operator,
\[
\Omega_{\tau_2,\pi_{2,\rm ref}}(\widehat q_2^{\rm mod})
-
\Omega_{\tau_2,\pi_{2,\rm ref}}(q_2^\star)
=
D\Omega_{\tau_2,\pi_{2,\rm ref}}(q_2^\star)[\Delta_2^{\rm mod}]
+
R_\Omega(\Delta_2^{\rm mod}).
\]
Since
\[
P_2D\Omega_{\tau_2,\pi_{2,\rm ref}}(q_2^\star)[\Delta_2^{\rm mod}]
=
P_2^{\pi_2^\star}\Delta_2^{\rm mod},
\]
we obtain
\[
\begin{aligned}
0
&=
(I-\Pi_\mu)\Delta_1^{\rm mod}
+
\gamma_2(\widehat P_2-P_2)
\Omega_{\tau_2,\pi_{2,\rm ref}}(q_2^\star)
\\
&\quad
-
(I-\gamma_2P_2^{\pi_2^\star})\Delta_2^{\rm mod}
+
R_{\rm emp}^{\rm mod}
+
\gamma_2P_2R_\Omega(\Delta_2^{\rm mod}).
\end{aligned}
\]
Rearranging gives
\[
\begin{aligned}
(I-\gamma_2P_2^{\pi_2^\star})\Delta_2^{\rm mod}
&=
(I-\Pi_\mu)\Delta_1^{\rm mod}
+
\gamma_2(\widehat P_2-P_2)
\Omega_{\tau_2,\pi_{2,\rm ref}}(q_2^\star)
\\
&\quad
+
R_{\rm emp}^{\rm mod}
+
\gamma_2P_2R_\Omega(\Delta_2^{\rm mod}).
\end{aligned}
\]
Now apply Lemma~\ref{lem:source_error_decomposition} with
\[
\widehat q_1=\widehat q_1^{\rm mod}.
\]
This yields
\[
\begin{aligned}
(I-\Pi_\mu)\Delta_1^{\rm mod}
&=
\gamma_1
(I-\Pi_\mu)
(I-\gamma_1\widehat P_1^\mu)^{-1}
(\widehat P_1^\mu-P_1^\mu)q_1^\star
\\
&\quad
-
(I-\Pi_\mu)
(I-\gamma_1\widehat P_1^\mu)^{-1}
\widehat b_1(\widehat q_1^{\rm mod}).
\end{aligned}
\]
Substituting this display into the target expansion proves the proposition, with
\[
\mathrm{Rem}_{\rm mod}
:=
R_{\rm emp}^{\rm mod}
+
\gamma_2P_2R_\Omega(\Delta_2^{\rm mod}).
\]
\end{proof}

\subsection{Coupled profiled target decomposition}

We now prove the coupled expansion. By
Corollary~\ref{cor:orthogonalized_target_residual}, the residual component of
the profiled target KKT equation is
\[
\widetilde b_{2,\rm coup}(q_1,q_2)
=
b_2(q_1,q_2)
+
(I-\Pi_\mu)
(I-\gamma_1P_1^\mu)^{-1}
b_1(q_1).
\]
Equivalently,
\[
\widetilde b_{2,\rm coup}(q_1,q_2)
=
(I-\Pi_\mu)q_1+g+C+\gamma_2P_2\Omega_{\tau_2,\pi_{2,\rm ref}}(q_2)-q_2
+
(I-\Pi_\mu)
(I-\gamma_1P_1^\mu)^{-1}
\{u_g^\star+\gamma_1P_1^\mu q_1-q_1\}.
\]
Because
\[
b_1(q_1)=u_g^\star-(I-\gamma_1P_1^\mu)q_1,
\]
we have
\[
(I-\gamma_1P_1^\mu)^{-1}b_1(q_1)
=
(I-\gamma_1P_1^\mu)^{-1}u_g^\star-q_1.
\]
Thus
\[
\widetilde b_{2,\rm coup}(q_1,q_2)
=
g+C+\gamma_2P_2\Omega_{\tau_2,\pi_{2,\rm ref}}(q_2)-q_2
+
(I-\Pi_\mu)
(I-\gamma_1P_1^\mu)^{-1}u_g^\star.
\]
In particular, at the population level, the profiled target residual is
independent of \(q_1\).

The empirical analogue replaces \(P_1^\mu\) and \(P_2\) by
\(\widehat P_1^\mu\) and \(\widehat P_2\). At the first-order scale, this gives
the empirical profiled target residual
\[
\widehat{\widetilde b}_{2,\rm coup}(q_2)
=
g+C+\gamma_2\widehat P_2\Omega_{\tau_2,\pi_{2,\rm ref}}(q_2)-q_2
+
(I-\Pi_\mu)
(I-\gamma_1\widehat P_1^\mu)^{-1}u_g^\star,
\]
up to higher-order profiling terms.

\begin{proposition}[Coupled first-order decomposition]
\label{prop:app_coup_first_order_decomposition}
Under Assumption~\ref{ass:app_first_order_regularity},
\[
\begin{aligned}
&
(I-\gamma_2P_2^{\pi_2^\star})\Delta_2^{\rm coup}
\\
&=
\gamma_1
(I-\Pi_\mu)
(I-\gamma_1\widehat P_1^\mu)^{-1}
(\widehat P_1^\mu-P_1^\mu)q_1^\star
+
\gamma_2(\widehat P_2-P_2)
\Omega_{\tau_2,\pi_{2,\rm ref}}(q_2^\star)
+
\mathrm{Rem}_{\rm coup}.
\end{aligned}
\]
\end{proposition}

\begin{proof}
The empirical profiled target equation at the coupled estimator gives
\[
0
=
\widehat{\widetilde b}_{2,\rm coup}(\widehat q_2^{\rm coup})
+
\text{local profiled optimality error}.
\]
Subtract the population profiled equation at \(q_2^\star\):
\[
\widetilde b_{2,\rm coup}(q_2^\star)=0.
\]
The truth-level empirical discrepancy is
\[
\begin{aligned}
&
\widehat{\widetilde b}_{2,\rm coup}(q_2^\star)
-
\widetilde b_{2,\rm coup}(q_2^\star)
\\
&=
\gamma_2(\widehat P_2-P_2)
\Omega_{\tau_2,\pi_{2,\rm ref}}(q_2^\star)
+
(I-\Pi_\mu)
\left[
(I-\gamma_1\widehat P_1^\mu)^{-1}
-
(I-\gamma_1P_1^\mu)^{-1}
\right]
u_g^\star.
\end{aligned}
\]
Use the resolvent identity
\[
\begin{aligned}
&
(I-\gamma_1\widehat P_1^\mu)^{-1}
-
(I-\gamma_1P_1^\mu)^{-1}
\\
&\qquad
=
(I-\gamma_1\widehat P_1^\mu)^{-1}
\left[
(I-\gamma_1P_1^\mu)
-
(I-\gamma_1\widehat P_1^\mu)
\right]
(I-\gamma_1P_1^\mu)^{-1}.
\end{aligned}
\]
Since
\[
(I-\gamma_1P_1^\mu)
-
(I-\gamma_1\widehat P_1^\mu)
=
\gamma_1(\widehat P_1^\mu-P_1^\mu),
\]
and
\[
(I-\gamma_1P_1^\mu)^{-1}u_g^\star=q_1^\star,
\]
we obtain
\[
\begin{aligned}
&
(I-\Pi_\mu)
\left[
(I-\gamma_1\widehat P_1^\mu)^{-1}
-
(I-\gamma_1P_1^\mu)^{-1}
\right]
u_g^\star
\\
&\qquad =
\gamma_1
(I-\Pi_\mu)
(I-\gamma_1\widehat P_1^\mu)^{-1}
(\widehat P_1^\mu-P_1^\mu)q_1^\star.
\end{aligned}
\]
Therefore,
\[
\begin{aligned}
&
\widehat{\widetilde b}_{2,\rm coup}(q_2^\star)
-
\widetilde b_{2,\rm coup}(q_2^\star)
\\
&=
\gamma_1
(I-\Pi_\mu)
(I-\gamma_1\widehat P_1^\mu)^{-1}
(\widehat P_1^\mu-P_1^\mu)q_1^\star
+
\gamma_2(\widehat P_2-P_2)
\Omega_{\tau_2,\pi_{2,\rm ref}}(q_2^\star).
\end{aligned}
\]

Next, expand the population profiled target residual in the \(q_2\)-direction:
\[
\widetilde b_{2,\rm coup}(\widehat q_2^{\rm coup})
-
\widetilde b_{2,\rm coup}(q_2^\star)
=
\gamma_2P_2
\left[
\Omega_{\tau_2,\pi_{2,\rm ref}}(\widehat q_2^{\rm coup})
-
\Omega_{\tau_2,\pi_{2,\rm ref}}(q_2^\star)
\right]
-
(\widehat q_2^{\rm coup}-q_2^\star).
\]
By Taylor expansion,
\[
\Omega_{\tau_2,\pi_{2,\rm ref}}(\widehat q_2^{\rm coup})
-
\Omega_{\tau_2,\pi_{2,\rm ref}}(q_2^\star)
=
D\Omega_{\tau_2,\pi_{2,\rm ref}}(q_2^\star)[\Delta_2^{\rm coup}]
+
R_\Omega(\Delta_2^{\rm coup}).
\]
Thus
\[
\widetilde b_{2,\rm coup}(\widehat q_2^{\rm coup})
-
\widetilde b_{2,\rm coup}(q_2^\star)
=
-(I-\gamma_2P_2^{\pi_2^\star})\Delta_2^{\rm coup}
+
\gamma_2P_2R_\Omega(\Delta_2^{\rm coup}).
\]

Combining the truth-level empirical discrepancy and the population drift gives
\[
\begin{aligned}
0
&=
\gamma_1
(I-\Pi_\mu)
(I-\gamma_1\widehat P_1^\mu)^{-1}
(\widehat P_1^\mu-P_1^\mu)q_1^\star
\\
&\quad
+
\gamma_2(\widehat P_2-P_2)
\Omega_{\tau_2,\pi_{2,\rm ref}}(q_2^\star)
-
(I-\gamma_2P_2^{\pi_2^\star})\Delta_2^{\rm coup}
+
R_{\rm emp}^{\rm coup}
+
\gamma_2P_2R_\Omega(\Delta_2^{\rm coup}),
\end{aligned}
\]
where \(R_{\rm emp}^{\rm coup}\) collects the local empirical-process increment,
local profiled optimality error, and higher-order profiling terms. Rearranging
proves the proposition, with
\[
\mathrm{Rem}_{\rm coup}
:=
R_{\rm emp}^{\rm coup}
+
\gamma_2P_2R_\Omega(\Delta_2^{\rm coup}).
\]
\end{proof}

\subsection{Proof of the main first-order propagation theorem}

\begin{proof}[Proof of Theorem~\ref{thm:first_order_transfer_error_propagation}]
The modular expansion
\eqref{eq:mod_first_order_decomposition} is exactly
Proposition~\ref{prop:app_mod_first_order_decomposition}. The coupled expansion
\eqref{eq:coup_first_order_decomposition} is exactly
Proposition~\ref{prop:app_coup_first_order_decomposition}. This proves the
theorem.
\end{proof}

\subsection{Remainder negligibility}

The decompositions above are first-order statements. We now record one standard
set of sufficient conditions under which the remainder terms are negligible at
the first-order scale.

\begin{lemma}[Negligibility of local empirical and nonlinear remainders]
\label{lem:app_error_propagation_remainder_negligible}
Assume:
\begin{enumerate}
    \item \(\widehat q_1\to q_1^\star\) and \(\widehat q_2\to q_2^\star\) in
    probability;

    \item the target empirical process is locally stochastically equicontinuous
    near \((q_1^\star,q_2^\star)\), so that its local increment between the
    estimator and the truth is
    \[
    o_p(n_1^{-1/2}+n_2^{-1/2});
    \]

    \item the profiled empirical process for the coupled estimator is locally
    stochastically equicontinuous near the truth, with local increment
    \[
    o_p(n_1^{-1/2}+n_2^{-1/2});
    \]

    \item \(\Omega_{\tau_2,\pi_{2,\rm ref}}\) is twice continuously Fr\'echet
    differentiable near \(q_2^\star\), and
    \[
    \|\widehat q_2-q_2^\star\|
    =
    O_p(n_1^{-1/2}+n_2^{-1/2}).
    \]
\end{enumerate}
Then
\[
\mathrm{Rem}_{\rm mod}
=
o_p(n_1^{-1/2}+n_2^{-1/2}),
\qquad
\mathrm{Rem}_{\rm coup}
=
o_p(n_1^{-1/2}+n_2^{-1/2}).
\]
\end{lemma}

\begin{proof}
For the modular estimator,
\[
\mathrm{Rem}_{\rm mod}
=
R_{\rm emp}^{\rm mod}
+
\gamma_2P_2R_\Omega(\Delta_2^{\rm mod}).
\]
The local empirical-process assumption gives
\[
R_{\rm emp}^{\rm mod}
=
o_p(n_1^{-1/2}+n_2^{-1/2}).
\]
By second-order differentiability of \(\Omega_{\tau_2,\pi_{2,\rm ref}}\),
\[
\|R_\Omega(\Delta_2^{\rm mod})\|
=
O(\|\Delta_2^{\rm mod}\|^2).
\]
Since
\[
\|\Delta_2^{\rm mod}\|
=
O_p(n_1^{-1/2}+n_2^{-1/2}),
\]
we have
\[
R_\Omega(\Delta_2^{\rm mod})
=
o_p(n_1^{-1/2}+n_2^{-1/2}).
\]
Because \(P_2\) is bounded and \(\gamma_2<1\),
\[
\gamma_2P_2R_\Omega(\Delta_2^{\rm mod})
=
o_p(n_1^{-1/2}+n_2^{-1/2}).
\]
Thus
\[
\mathrm{Rem}_{\rm mod}
=
o_p(n_1^{-1/2}+n_2^{-1/2}).
\]

The coupled case is identical. We have
\[
\mathrm{Rem}_{\rm coup}
=
R_{\rm emp}^{\rm coup}
+
\gamma_2P_2R_\Omega(\Delta_2^{\rm coup}),
\]
where the profiled empirical-process assumption gives
\[
R_{\rm emp}^{\rm coup}
=
o_p(n_1^{-1/2}+n_2^{-1/2}),
\]
and second-order differentiability gives
\[
\gamma_2P_2R_\Omega(\Delta_2^{\rm coup})
=
o_p(n_1^{-1/2}+n_2^{-1/2}).
\]
Therefore
\[
\mathrm{Rem}_{\rm coup}
=
o_p(n_1^{-1/2}+n_2^{-1/2}).
\]
\end{proof}

\subsection{Asymptotic implication}

The first-order decompositions also give a transparent asymptotic comparison.
Suppose the leading empirical perturbations admit asymptotic linear
representations
\[
\gamma_1
(I-\Pi_\mu)
(I-\gamma_1\widehat P_1^\mu)^{-1}
(\widehat P_1^\mu-P_1^\mu)q_1^\star
=
\frac1{n_1}\sum_{i=1}^{n_1}\psi_{1,\rm op}(W_{1,i})
+
o_p(n_1^{-1/2}),
\]
\[
(I-\Pi_\mu)
(I-\gamma_1\widehat P_1^\mu)^{-1}
\widehat b_1(\widehat q_1^{\rm mod})
=
\frac1{n_1}\sum_{i=1}^{n_1}\psi_{1,\rm res}(W_{1,i})
+
o_p(n_1^{-1/2}),
\]
and
\[
\gamma_2(\widehat P_2-P_2)\Omega_{\tau_2,\pi_{2,\rm ref}}(q_2^\star)
=
\frac1{n_2}\sum_{j=1}^{n_2}\psi_2(W_{2,j})
+
o_p(n_2^{-1/2}).
\]
Let
\[
\Sigma_{1,\rm coup}
:=
\operatorname{Var}(\psi_{1,\rm op}(W_1)),
\]
\[
\Sigma_{1,\rm mod}
:=
\operatorname{Var}(\psi_{1,\rm op}(W_1)-\psi_{1,\rm res}(W_1)),
\]
and
\[
\Sigma_2
:=
\operatorname{Var}(\psi_2(W_2)).
\]
The minus sign in \(\Sigma_{1,\rm mod}\) matches the modular decomposition
\eqref{eq:mod_first_order_decomposition}. No independence between
\(\psi_{1,\rm op}\) and \(\psi_{1,\rm res}\) is assumed; the covariance is
included in \(\Sigma_{1,\rm mod}\).

\begin{theorem}[Unconditional first-order asymptotic comparison]
\label{thm:app_error_propagation_asymptotic}
Assume the asymptotic linear representations above,
Lemma~\ref{lem:app_error_propagation_remainder_negligible}, and
\[
\frac{n_2}{n_1}\to\lambda\in(0,\infty).
\]
Then
\[
\sqrt{n_2}\,(\widehat q_2^{\rm coup}-q_2^\star)
\rightsquigarrow
\mathcal N\!\left(
0,\,
(I-\gamma_2P_2^{\pi_2^\star})^{-1}
\bigl(
\lambda\Sigma_{1,\rm coup}+\Sigma_2
\bigr)
\bigl[(I-\gamma_2P_2^{\pi_2^\star})^{-1}\bigr]^\top
\right),
\]
whereas
\[
\sqrt{n_2}\,(\widehat q_2^{\rm mod}-q_2^\star)
\rightsquigarrow
\mathcal N\!\left(
0,\,
(I-\gamma_2P_2^{\pi_2^\star})^{-1}
\bigl(
\lambda\Sigma_{1,\rm mod}+\Sigma_2
\bigr)
\bigl[(I-\gamma_2P_2^{\pi_2^\star})^{-1}\bigr]^\top
\right).
\]
\end{theorem}

\begin{proof}
For the coupled estimator, Theorem~\ref{thm:first_order_transfer_error_propagation}
and Lemma~\ref{lem:app_error_propagation_remainder_negligible} imply
\[
\begin{aligned}
&
(I-\gamma_2P_2^{\pi_2^\star})(\widehat q_2^{\rm coup}-q_2^\star)
\\
&=
\frac1{n_1}\sum_{i=1}^{n_1}\psi_{1,\rm op}(W_{1,i})
+
\frac1{n_2}\sum_{j=1}^{n_2}\psi_2(W_{2,j})
+
o_p(n_1^{-1/2}+n_2^{-1/2}).
\end{aligned}
\]
Multiplying by \(\sqrt{n_2}\) gives
\[
\begin{aligned}
&
\sqrt{n_2}
(I-\gamma_2P_2^{\pi_2^\star})(\widehat q_2^{\rm coup}-q_2^\star)
\\
&=
\sqrt{\frac{n_2}{n_1}}
\frac1{\sqrt{n_1}}\sum_{i=1}^{n_1}\psi_{1,\rm op}(W_{1,i})
+
\frac1{\sqrt{n_2}}\sum_{j=1}^{n_2}\psi_2(W_{2,j})
+
o_p(1).
\end{aligned}
\]
Since \(n_2/n_1\to\lambda\), the central limit theorem yields
\[
\sqrt{n_2}
(I-\gamma_2P_2^{\pi_2^\star})(\widehat q_2^{\rm coup}-q_2^\star)
\rightsquigarrow
\mathcal N(0,\lambda\Sigma_{1,\rm coup}+\Sigma_2).
\]
Applying the inverse of \(I-\gamma_2P_2^{\pi_2^\star}\) and Slutsky's theorem
gives the coupled limit.

For the modular estimator, the first-order decomposition gives
\[
\begin{aligned}
&
(I-\gamma_2P_2^{\pi_2^\star})(\widehat q_2^{\rm mod}-q_2^\star)
\\
&=
\frac1{n_1}\sum_{i=1}^{n_1}
\bigl[
\psi_{1,\rm op}(W_{1,i})
-
\psi_{1,\rm res}(W_{1,i})
\bigr]
+
\frac1{n_2}\sum_{j=1}^{n_2}\psi_2(W_{2,j})
+
o_p(n_1^{-1/2}+n_2^{-1/2}).
\end{aligned}
\]
The same central limit and Slutsky arguments yield the modular limit with
\(\Sigma_{1,\rm mod}\).
\end{proof}

The theorem shows that coupled transfer retains the source operator variance
component but removes the first-order source residual component. Modular
transfer carries both. This is the asymptotic counterpart of the deterministic
decompositions in Theorem~\ref{thm:first_order_transfer_error_propagation}.

\section{Proofs for Section \ref{sec:global_policy_guarantees}}
\label{app:global_finite_sample}

This appendix proves the global finite-sample value bounds stated in
Theorem~\ref{thm:global_q2_bounds}. The proof uses only empirical minimax
optimality and uniform concentration over finite function classes. We treat the
modular and coupled estimators in a common framework, and then specialize the
argument to each estimator.

The modular proof has two steps. First, the source minimax problem controls
\(\widehat q_1^{\rm mod}-q_1^\star\) under \(\rho_1\). Second, this source
error is propagated through the target soft Bellman equation and combined with
the target empirical error. The coupled proof is different: it controls the
source and target empirical residuals in one saddle-point comparison, and the
\(\beta=0\) formulation directly yields a bound for
\(\widehat q_2^{\rm coup}-q_2^\star\).

Throughout this appendix, write
\[
\Omega:=\Omega_{\tau_2,\pi_{2,\rm ref}}.
\]
Let
\[
\mathcal D_1=\{(s_i^1,a_i^1,s_i^{1\prime})\}_{i=1}^{n_1},
\qquad
\mathcal D_2=\{(s_j^2,a_j^2,s_j^{2\prime})\}_{j=1}^{n_2},
\]
where \((s_i^1,a_i^1)\sim\rho_1\),
\(s_i^{1\prime}\sim P_1(\cdot\mid s_i^1,a_i^1)\), and
\((s_j^2,a_j^2)\sim\rho_2\),
\(s_j^{2\prime}\sim P_2(\cdot\mid s_j^2,a_j^2)\). The two datasets are
independent.

For empirical averages, write
\[
\widehat{\mathbb E}_{n_1} f
:=
\frac1{n_1}\sum_{i=1}^{n_1}
f(s_i^1,a_i^1,s_i^{1\prime}),
\qquad
\widehat{\mathbb E}_{n_2} f
:=
\frac1{n_2}\sum_{j=1}^{n_2}
f(s_j^2,a_j^2,s_j^{2\prime}).
\]

\subsection{Shared assumptions and envelopes}
\label{app:shared_assumptions_envelopes}

We collect the assumptions used by both modular and coupled analyses.

\begin{assumption}[Finite-class realizability and boundedness]
\label{ass:app_global_finite_realizable_bounded}
The function classes are finite and satisfy
\[
q_1^\star\in\mathcal Q_1,
\qquad
q_2^\star\in\mathcal Q_2,
\qquad
l_1^\star\in\mathcal L_1,
\qquad
l_2^\star\in\mathcal L_2.
\]
Moreover, for all functions in the relevant classes,
\[
\|q_1\|_\infty\le B_{Q1},
\qquad
\|q_2\|_\infty\le B_{Q2},
\qquad
\|l_2\|_\infty\le B_{L2}.
\]
The source soft-control signal and anchor satisfy
\[
\|u_b^\star\|_\infty\le B_U,
\qquad
0\le g(s)\le B_G.
\]
The reward shift is chosen so that the shifted transferred reward is
nonnegative at the population target and, for the envelope bounds below, the
shift contribution is absorbed into the \(B_{Q1}\)-scale. Consequently, the
source and target residual ranges are bounded by
\[
B_{\rm Res,1}:=B_U+B_G+(\gamma_1+1)B_{Q1},
\]
and
\[
B_{\rm Res,2}:=4B_{Q1}+B_G+(\gamma_2+1)B_{Q2}.
\]
\end{assumption}

\begin{assumption}[Coverage and dual radii]
\label{ass:app_global_coverage_dual}
The source and target coverage assumptions used in the main text hold. The
target dual radius is common to both the modular and coupled analyses. In
particular, the target dual class satisfies
\[
\|l_2\|_\infty\le B_{L2},
\qquad
B_{L2}:=\frac{\kappa_2}{1-\gamma_2}B_{Q2}.
\]

The source dual radius has two conceptually distinct components. The first is a
source self component, generated by the source quadratic regularization and the
source Bellman equation:
\[
B_{L1}^{\rm self}
:=
\frac{\kappa_1}{1-\gamma_1}B_{Q1}.
\]
This is the only source dual contribution needed for the modular source
minimax problem. Thus, for modular transfer, we take
\[
B_{L1}^{\rm mod}:=B_{L1}^{\rm self}.
\]

The second component is a cross-environment coupling contribution. This term is
present only in the coupled minimax KKT system, because the target residual
depends on the source variable through the transferred reward map
\[
q_1\mapsto (I-\Pi_\mu)q_1+g+C.
\]
Let \(B_{L12}\) denote a uniform bound on this coupling-induced source dual
component. Equivalently, \(B_{L12}\) controls the source resolvent applied to the
target-to-source adjoint term induced by \(l_2^\star\) and
\((I-\Pi_\mu)^\top\), under the cross-environment coverage assumptions.

Therefore, in the coupled formulation with source weight \(\beta\), the source
dual class is chosen with radius
\[
B_{L1}^{\rm coup}(\beta)
:=
\beta B_{L1}^{\rm self}+B_{L12}.
\]
In particular, for the \(\beta=0\) coupled formulation used in the main global
finite-sample theorem,
\[
B_{L1}^{\rm coup}(0)=B_{L12}.
\]
\end{assumption}

\begin{assumption}[Cross-measure and action coverage for modular propagation]
\label{ass:app_modular_cross_action}
For modular transfer, assume
\[
\chi_{21}:=
\left\|
\frac{\rho_2}{\rho_1}
\right\|_\infty
<\infty,
\]
and
\[
\kappa_{\mu\mid\pi_{b1}}
:=
\left\|
\frac{\mu}{\pi_{b1}}
\right\|_\infty
\le
\frac1{\epsilon_{b1}}.
\]
\end{assumption}

Define the finite-class cardinalities
\[
N_1:=|\mathcal Q_1||\mathcal L_1|,
\qquad
N_2^{\rm mod}:=|\mathcal Q_2||\mathcal L_2|,
\qquad
N_2^{\rm coup}:=|\mathcal Q_1||\mathcal Q_2||\mathcal L_2|.
\]
The constants appearing in the global theorem are
\[
M_1^{\rm mod}
:=
\frac{4\kappa_2}{\epsilon_{b1}(1-\gamma_2)^2}
\chi_{21}
\left(
\frac12 B_{Q1}^2+B_{L1}^{\rm mod}B_{\rm Res,1}
\right),
\]
\[
M_1^{\rm coup}
:=
B_{L12}B_{\rm Res,1},
\]
and
\[
M_2
:=
\frac12 B_{Q2}^2+B_{L2}B_{\rm Res,2}.
\]

\subsection{A shared resolvent stability lemma}
\label{app:shared_resolvent_lemma}

We will repeatedly use the following \(L^2(\rho)\) stability bound for
discounted resolvents. It converts a pathwise concentrability condition into an
operator norm bound.

\begin{lemma}[Discounted resolvent stability under concentrability]
\label{lem:discounted_resolvent_l2_stability}
Let \(K\) be a Markov operator on a measurable space \(\mathcal X\), let
\(\gamma\in[0,1)\), and let \(\rho\) be a reference probability measure on
\(\mathcal X\). Define
\[
d_{\rho}^{K}
:=
(1-\gamma)\sum_{t=0}^{\infty}\gamma^t (K^\top)^t\rho .
\]
If
\[
\left\|
\frac{d_{\rho}^{K}}{\rho}
\right\|_\infty
\le \kappa ,
\]
then, for every \(f\in L^2(\rho)\),
\[
\left\|
(I-\gamma K)^{-1}f
\right\|_{\rho}
\le
\frac{\sqrt{\kappa}}{1-\gamma}\|f\|_{\rho}.
\]
Equivalently,
\[
\left\|
(I-\gamma K)^{-1}f
\right\|_{\rho}^2
\le
\frac{\kappa}{(1-\gamma)^2}
\|f\|_{\rho}^2 .
\]
\end{lemma}

\begin{proof}
By the resolvent expansion,
\[
(I-\gamma K)^{-1}f
=
\sum_{t=0}^{\infty}\gamma^t K^t f
=
\frac{1}{1-\gamma}
\sum_{t=0}^{\infty} w_t K^t f,
\qquad
w_t:=(1-\gamma)\gamma^t,
\]
where \(w_t\ge0\) and \(\sum_{t\ge0}w_t=1\). Jensen's inequality gives
\[
\left((I-\gamma K)^{-1}f\right)^2
\le
\frac{1}{(1-\gamma)^2}
\sum_{t=0}^{\infty}w_t (K^t f)^2 .
\]
Since \(K^t\) is Markov, another application of Jensen's inequality gives
\[
(K^t f)^2\le K^t(f^2).
\]
Therefore,
\[
\left((I-\gamma K)^{-1}f\right)^2
\le
\frac{1}{(1-\gamma)^2}
\sum_{t=0}^{\infty}w_t K^t(f^2).
\]
Taking expectation under \(\rho\) and using
\[
\mathbb E_\rho[K^t h]
=
\mathbb E_{(K^\top)^t\rho}[h],
\]
we obtain
\[
\begin{aligned}
\left\|(I-\gamma K)^{-1}f\right\|_\rho^2
&\le
\frac{1}{(1-\gamma)^2}
\sum_{t=0}^{\infty}w_t
\mathbb E_{(K^\top)^t\rho}[f^2]
\\
&=
\frac{1}{(1-\gamma)^2}
\mathbb E_{d_\rho^K}[f^2].
\end{aligned}
\]
By the concentrability assumption,
\[
\mathbb E_{d_\rho^K}[f^2]
=
\mathbb E_\rho
\left[
\frac{d_\rho^K}{\rho}f^2
\right]
\le
\kappa\|f\|_\rho^2.
\]
Combining the preceding displays proves the squared bound, and taking square
roots proves the norm bound.
\end{proof}

When \(K=P_2^\pi\), this lemma applies with
\[
d_{\rho_2}^{K}
=
(1-\gamma_2)
\sum_{t=0}^\infty
\gamma_2^t
\big((P_2^\pi)^\top\big)^t\rho_2
=
d^{P_2}_{\pi,\rho_2}.
\]

\subsection{Modular minimax transfer}
\label{app:modular_global_proof}

The modular estimator first solves the source minimax problem
\[
(\widehat q_1^{\rm mod},\widehat l_1^{\rm mod})
\in
\arg\min_{q_1\in\mathcal Q_1}
\max_{l_1\in\mathcal L_1}
\widehat{\mathcal L}_{1,n_1}^{\rm mod}(q_1,l_1),
\]
where
\[
\widehat{\mathcal L}_{1,n_1}^{\rm mod}(q_1,l_1)
:=
\widehat{\mathbb E}_{n_1}Z_1^{\rm mod}(q_1,l_1),
\]
and
\[
Z_1^{\rm mod}(q_1,l_1)(s,a,s')
:=
\frac12q_1(s,a)^2
+
l_1(s,a)
\bigl(
u_g^\star(s,a)+\gamma_1(\mu q_1)(s')-q_1(s,a)
\bigr).
\]
Then it solves the target minimax problem with \(\widehat q_1^{\rm mod}\)
treated as fixed:
\[
(\widehat q_2^{\rm mod},\widehat l_2^{\rm mod})
\in
\arg\min_{q_2\in\mathcal Q_2}
\max_{l_2\in\mathcal L_2}
\widehat{\mathcal L}_{2,n_2}^{(\widehat q_1^{\rm mod})}(q_2,l_2),
\]
where
\[
\widehat{\mathcal L}_{2,n_2}^{(\widehat q_1^{\rm mod})}(q_2,l_2)
:=
\widehat{\mathbb E}_{n_2}
Z_2^{(\widehat q_1^{\rm mod})}(q_2,l_2),
\]
and
\[
\begin{aligned}
Z_2^{(\widehat q_1^{\rm mod})}(q_2,l_2)(s,a,s')
:=
\frac12 q_2(s,a)^2
+
l_2(s,a)
\Bigl(
(I-\Pi_\mu)\widehat q_1^{\rm mod}(s,a)
+
g(s)+C
+
\gamma_2\Omega(q_2)(s')
-
q_2(s,a)
\Bigr).
\end{aligned}
\]

\begin{lemma}[Source-stage concentration for modular transfer]
\label{lem:app_mod_source_concentration}
Under Assumptions~\ref{ass:app_global_finite_realizable_bounded} and
\ref{ass:app_global_coverage_dual}, with probability at least \(1-\delta\),
\[
\sup_{q_1\in\mathcal Q_1,\;l_1\in\mathcal L_1}
\left|
\widehat{\mathcal L}_{1,n_1}^{\rm mod}(q_1,l_1)
-
\mathcal L_1^{\rm mod}(q_1,l_1)
\right|
\le
\varepsilon_1^{\rm mod},
\]
where
\[
\varepsilon_1^{\rm mod}
:=
\left(
\frac12 B_{Q1}^2+B_{L1}^{\rm mod}B_{\rm Res,1}
\right)
\sqrt{\frac{2\log(2N_1/\delta)}{n_1}}.
\]
\end{lemma}

\begin{proof}
By the boundedness assumptions,
\[
\left|
u_g^\star(s,a)+\gamma_1(\mu q_1)(s')-q_1(s,a)
\right|
\le
B_U+B_G+(\gamma_1+1)B_{Q1}
=
B_{\rm Res,1}.
\]
Thus
\[
|Z_1^{\rm mod}(q_1,l_1)|
\le
\frac12 B_{Q1}^2+B_{L1}^{\rm mod}B_{\rm Res,1}.
\]
Hoeffding's inequality for each pair \((q_1,l_1)\), followed by a union bound
over \(\mathcal Q_1\times\mathcal L_1\), gives the claim.
\end{proof}

\begin{lemma}[Source-stage value bound for modular transfer]
\label{lem:app_mod_source_value}
On the event of Lemma~\ref{lem:app_mod_source_concentration},
\[
\|\widehat q_1^{\rm mod}-q_1^\star\|_{\rho_1}^2
\le
4\varepsilon_1^{\rm mod}.
\]
\end{lemma}

\begin{proof}
By empirical minimax optimality,
\[
\widehat{\mathcal L}_{1,n_1}^{\rm mod}
(\widehat q_1^{\rm mod},l_1^\star)
\le
\widehat{\mathcal L}_{1,n_1}^{\rm mod}
(\widehat q_1^{\rm mod},\widehat l_1^{\rm mod})
\le
\widehat{\mathcal L}_{1,n_1}^{\rm mod}
(q_1^\star,\widehat l_1^{\rm mod}).
\]
Using the uniform deviation event on the two outer terms gives
\[
\begin{aligned}
\mathcal L_1^{\rm mod}
(\widehat q_1^{\rm mod},l_1^\star)
&\le
\widehat{\mathcal L}_{1,n_1}^{\rm mod}
(\widehat q_1^{\rm mod},l_1^\star)
+
\varepsilon_1^{\rm mod}
\\
&\le
\widehat{\mathcal L}_{1,n_1}^{\rm mod}
(q_1^\star,\widehat l_1^{\rm mod})
+
\varepsilon_1^{\rm mod}
\\
&\le
\mathcal L_1^{\rm mod}
(q_1^\star,\widehat l_1^{\rm mod})
+
2\varepsilon_1^{\rm mod}.
\end{aligned}
\]
Since \(b_1(q_1^\star)=0\), the population source Lagrangian at \(q_1^\star\)
is independent of the dual variable. Hence
\[
\mathcal L_1^{\rm mod}
(q_1^\star,\widehat l_1^{\rm mod})
=
\mathcal L_1^{\rm mod}
(q_1^\star,l_1^\star).
\]
Therefore,
\[
\mathcal L_1^{\rm mod}
(\widehat q_1^{\rm mod},l_1^\star)
-
\mathcal L_1^{\rm mod}
(q_1^\star,l_1^\star)
\le
2\varepsilon_1^{\rm mod}.
\]
The source population saddle inequality gives
\[
\mathcal L_1^{\rm mod}(q_1,l_1^\star)
-
\mathcal L_1^{\rm mod}(q_1^\star,l_1^\star)
\ge
\frac12\|q_1-q_1^\star\|_{\rho_1}^2.
\]
Applying this with \(q_1=\widehat q_1^{\rm mod}\) proves the claim.
\end{proof}

\begin{lemma}[Target-stage concentration for modular transfer]
\label{lem:app_mod_target_concentration}
Conditional on the source dataset, with probability at least \(1-\delta\),
\[
\sup_{q_2\in\mathcal Q_2,\;l_2\in\mathcal L_2}
\left|
\widehat{\mathcal L}_{2,n_2}^{(\widehat q_1^{\rm mod})}(q_2,l_2)
-
\mathcal L_2^{(\widehat q_1^{\rm mod})}(q_2,l_2)
\right|
\le
\varepsilon_2^{\rm mod},
\]
where
\[
\varepsilon_2^{\rm mod}
:=
M_2
\sqrt{\frac{2\log(2N_2^{\rm mod}/\delta)}{n_2}}.
\]
\end{lemma}

\begin{proof}
For all \(q_2\in\mathcal Q_2\),
\[
|\Omega(q_2)(s)|\le B_{Q2}.
\]
Moreover,
\[
\|(I-\Pi_\mu)\widehat q_1^{\rm mod}\|_\infty\le 2B_{Q1}.
\]
By the definition of \(B_{\rm Res,2}\), the target residual term is uniformly
bounded by \(B_{\rm Res,2}\). Hence
\[
|Z_2^{(\widehat q_1^{\rm mod})}(q_2,l_2)|
\le
\frac12B_{Q2}^2+B_{L2}B_{\rm Res,2}
=
M_2.
\]
Hoeffding's inequality and a union bound over
\(\mathcal Q_2\times\mathcal L_2\), conditional on the source dataset, yield the
claim.
\end{proof}

\begin{lemma}[Target empirical error around the plug-in target]
\label{lem:app_mod_target_empirical_global}
Let \(\widetilde q_2\) be the plug-in population target defined by
\[
\widetilde q_2
=
(I-\Pi_\mu)\widehat q_1^{\rm mod}+g+C
+
\gamma_2P_2\Omega(\widetilde q_2).
\]
On the event of Lemma~\ref{lem:app_mod_target_concentration},
\[
\|\widehat q_2^{\rm mod}-\widetilde q_2\|_{\rho_2}^2
\le
4\varepsilon_2^{\rm mod}.
\]
\end{lemma}

\begin{proof}
Conditional on \(\widehat q_1^{\rm mod}\), the proof is the same minimax
comparison used in Lemma~\ref{lem:app_mod_source_value}. Let \(\widetilde l_2\)
denote a population target dual variable associated with \(\widetilde q_2\).
Empirical minimax optimality gives
\[
\widehat{\mathcal L}_{2,n_2}^{(\widehat q_1^{\rm mod})}
(\widehat q_2^{\rm mod},\widetilde l_2)
\le
\widehat{\mathcal L}_{2,n_2}^{(\widehat q_1^{\rm mod})}
(\widehat q_2^{\rm mod},\widehat l_2^{\rm mod})
\le
\widehat{\mathcal L}_{2,n_2}^{(\widehat q_1^{\rm mod})}
(\widetilde q_2,\widehat l_2^{\rm mod}).
\]
Using the uniform deviation event on the two outer terms yields
\[
\mathcal L_2^{(\widehat q_1^{\rm mod})}
(\widehat q_2^{\rm mod},\widetilde l_2)
-
\mathcal L_2^{(\widehat q_1^{\rm mod})}
(\widetilde q_2,\widehat l_2^{\rm mod})
\le
2\varepsilon_2^{\rm mod}.
\]
Since \(b_2(\widehat q_1^{\rm mod},\widetilde q_2)=0\), the population
Lagrangian at \(\widetilde q_2\) is independent of the target dual variable.
Thus
\[
\mathcal L_2^{(\widehat q_1^{\rm mod})}
(\widehat q_2^{\rm mod},\widetilde l_2)
-
\mathcal L_2^{(\widehat q_1^{\rm mod})}
(\widetilde q_2,\widetilde l_2)
\le
2\varepsilon_2^{\rm mod}.
\]
The target population saddle inequality gives
\[
\mathcal L_2^{(\widehat q_1^{\rm mod})}(q_2,\widetilde l_2)
-
\mathcal L_2^{(\widehat q_1^{\rm mod})}(\widetilde q_2,\widetilde l_2)
\ge
\frac12\|q_2-\widetilde q_2\|_{\rho_2}^2.
\]
Applying this inequality at \(q_2=\widehat q_2^{\rm mod}\) proves the claim.
\end{proof}

\begin{lemma}[Reward perturbation induced by the modular source estimate]
\label{lem:app_mod_reward_perturb_global}
Under Assumption~\ref{ass:app_modular_cross_action},
\[
\|
(I-\Pi_\mu)(\widehat q_1^{\rm mod}-q_1^\star)
\|_{\rho_2}
\le
\frac{2\sqrt{\chi_{21}}}{\sqrt{\epsilon_{b1}}}
\|\widehat q_1^{\rm mod}-q_1^\star\|_{\rho_1}.
\]
\end{lemma}

\begin{proof}
Let \(h:=\widehat q_1^{\rm mod}-q_1^\star\). By the change-of-measure
inequality,
\[
\|(I-\Pi_\mu)h\|_{\rho_2}
\le
\sqrt{\chi_{21}}\|(I-\Pi_\mu)h\|_{\rho_1}.
\]
Next,
\[
\|(I-\Pi_\mu)h\|_{\rho_1}
\le
\|h\|_{\rho_1}+\|\Pi_\mu h\|_{\rho_1}.
\]
Writing
\[
\rho_1(s,a)=\rho_1^{\mathcal S}(s)\pi_{b1}(a\mid s),
\]
Jensen's inequality gives
\[
\begin{aligned}
\|\Pi_\mu h\|_{\rho_1}^2
&=
\mathbb E_{s\sim\rho_1^{\mathcal S}}
\left[
\left(\sum_a\mu(a\mid s)h(s,a)\right)^2
\right]
\\
&\le
\mathbb E_{s\sim\rho_1^{\mathcal S}}
\sum_a\mu(a\mid s)h(s,a)^2
\\
&\le
\kappa_{\mu\mid\pi_{b1}}
\mathbb E_{s\sim\rho_1^{\mathcal S}}
\sum_a\pi_{b1}(a\mid s)h(s,a)^2
\\
&=
\kappa_{\mu\mid\pi_{b1}}\|h\|_{\rho_1}^2.
\end{aligned}
\]
Since \(\kappa_{\mu\mid\pi_{b1}}\le1/\epsilon_{b1}\),
\[
\|(I-\Pi_\mu)h\|_{\rho_1}
\le
\left(1+\frac1{\sqrt{\epsilon_{b1}}}\right)\|h\|_{\rho_1}
\le
\frac2{\sqrt{\epsilon_{b1}}}\|h\|_{\rho_1}.
\]
Combining the two displays proves the claim.
\end{proof}

\begin{lemma}[Target propagation of modular plug-in error]
\label{lem:app_mod_target_propag_global}
Assume the pathwise target concentrability condition
\[
\sup_{\pi\in\Pi_{\rm path}}
\left\|
\frac{d^{P_2}_{\pi,\rho_2}}{\rho_2}
\right\|_\infty
\le
\kappa_2^{\rm path},
\]
where \(\Pi_{\rm path}\) contains the policies induced along the line segment
between \(q_2^\star\) and \(\widetilde q_2\). Then
\[
\|\widetilde q_2-q_2^\star\|_{\rho_2}^2
\le
\frac{
4\kappa_2^{\rm path}\chi_{21}
}{
\epsilon_{b1}(1-\gamma_2)^2
}
\|\widehat q_1^{\rm mod}-q_1^\star\|_{\rho_1}^2.
\]
\end{lemma}

\begin{proof}
Subtract the target fixed-point equations for \(\widetilde q_2\) and
\(q_2^\star\):
\[
\widetilde q_2-q_2^\star
=
(I-\Pi_\mu)(\widehat q_1^{\rm mod}-q_1^\star)
+
\gamma_2P_2
\bigl[
\Omega(\widetilde q_2)-\Omega(q_2^\star)
\bigr].
\]
Let \(\Delta_2:=\widetilde q_2-q_2^\star\). By the fundamental theorem of
calculus,
\[
\Omega(\widetilde q_2)-\Omega(q_2^\star)
=
\int_0^1
D\Omega(q_2^\star+t\Delta_2)[\Delta_2]\,dt.
\]
The derivative \(D\Omega(q)\) is the statewise expectation under the policy
induced by \(q\). Hence the averaged derivative along the path can be written as
a Markov state-action averaging operator \(P_2^{\bar\pi}\), where
\(\bar\pi\in\Pi_{\rm path}\). Thus
\[
\Delta_2
=
(I-\Pi_\mu)(\widehat q_1^{\rm mod}-q_1^\star)
+
\gamma_2P_2^{\bar\pi}\Delta_2.
\]
Equivalently,
\[
\Delta_2
=
(I-\gamma_2P_2^{\bar\pi})^{-1}
(I-\Pi_\mu)(\widehat q_1^{\rm mod}-q_1^\star).
\]
Applying Lemma~\ref{lem:discounted_resolvent_l2_stability} with
\[
K=P_2^{\bar\pi},
\qquad
\gamma=\gamma_2,
\qquad
\rho=\rho_2,
\qquad
f=(I-\Pi_\mu)(\widehat q_1^{\rm mod}-q_1^\star),
\]
and using
\[
\left\|
\frac{d^{P_2}_{\bar\pi,\rho_2}}{\rho_2}
\right\|_\infty
\le
\kappa_2^{\rm path},
\]
we obtain
\[
\|\Delta_2\|_{\rho_2}
\le
\frac{\sqrt{\kappa_2^{\rm path}}}{1-\gamma_2}
\|(I-\Pi_\mu)(\widehat q_1^{\rm mod}-q_1^\star)\|_{\rho_2}.
\]
Applying Lemma~\ref{lem:app_mod_reward_perturb_global} and squaring both sides
proves the claim.
\end{proof}

\begin{theorem}[Global value bound for modular minimax transfer]
\label{thm:app_modular_global_bound}
Under Assumptions~\ref{ass:app_global_finite_realizable_bounded}--
\ref{ass:app_modular_cross_action} and the pathwise target concentrability
condition of Lemma~\ref{lem:app_mod_target_propag_global}, with probability at
least \(1-\delta\),
\[
\|\widehat q_2^{\rm mod}-q_2^\star\|_{\rho_2}^2
\le
8
\left(
M_1^{\rm mod}
\sqrt{\frac{2\log(4N_1/\delta)}{n_1}}
+
M_2
\sqrt{\frac{2\log(4N_2^{\rm mod}/\delta)}{n_2}}
\right).
\]
\end{theorem}

\begin{proof}
Apply Lemmas~\ref{lem:app_mod_source_concentration} and
\ref{lem:app_mod_target_concentration} with confidence level \(\delta/2\) in
each stage. With probability at least \(1-\delta\), both events hold.

By the decomposition
\[
\widehat q_2^{\rm mod}-q_2^\star
=
(\widehat q_2^{\rm mod}-\widetilde q_2)
+
(\widetilde q_2-q_2^\star),
\]
we have
\[
\|\widehat q_2^{\rm mod}-q_2^\star\|_{\rho_2}^2
\le
2\|\widehat q_2^{\rm mod}-\widetilde q_2\|_{\rho_2}^2
+
2\|\widetilde q_2-q_2^\star\|_{\rho_2}^2.
\]
Lemma~\ref{lem:app_mod_target_empirical_global} gives
\[
2\|\widehat q_2^{\rm mod}-\widetilde q_2\|_{\rho_2}^2
\le
8\varepsilon_2^{\rm mod}.
\]
Lemmas~\ref{lem:app_mod_target_propag_global} and
\ref{lem:app_mod_source_value} give
\[
2\|\widetilde q_2-q_2^\star\|_{\rho_2}^2
\le
\frac{
8\kappa_2^{\rm path}\chi_{21}
}{
\epsilon_{b1}(1-\gamma_2)^2
}
\|\widehat q_1^{\rm mod}-q_1^\star\|_{\rho_1}^2
\le
\frac{
32\kappa_2^{\rm path}\chi_{21}
}{
\epsilon_{b1}(1-\gamma_2)^2
}
\varepsilon_1^{\rm mod}.
\]
Using the definition of \(M_1^{\rm mod}\), the source contribution becomes
\[
8M_1^{\rm mod}
\sqrt{\frac{2\log(4N_1/\delta)}{n_1}},
\]
and the target contribution becomes
\[
8M_2
\sqrt{\frac{2\log(4N_2^{\rm mod}/\delta)}{n_2}}.
\]
Combining the two terms proves the theorem.
\end{proof}

\subsection{Coupled minimax transfer}
\label{app:coupled_global_proof}

We now prove the global bound for the coupled minimax estimator. We first
present the argument for \(\beta>0\), where the coupled population objective has
quadratic growth in both \(q_1\) and \(q_2\). We then specialize the same
argument to the \(\beta=0\) formulation used in
Theorem~\ref{thm:global_q2_bounds}. In the latter case, the source quadratic
term is absent, so the proof no longer yields a norm bound for
\(\widehat q_1^{\rm coup}-q_1^\star\). However, the target quadratic term
remains, and the same saddle comparison still gives the desired bound for
\(\widehat q_2^{\rm coup}-q_2^\star\).

For \(\beta\ge0\), define the empirical coupled Lagrangian
\[
\widehat{\mathcal L}_{\rm coup}^{\beta}
(q_1,q_2,l_1,l_2)
:=
\widehat{\mathbb E}_{n_1}Z_{1,\beta}^{\rm coup}(q_1,l_1)
+
\widehat{\mathbb E}_{n_2}Z_2^{\rm coup}(q_1,q_2,l_2),
\]
where
\[
Z_{1,\beta}^{\rm coup}(q_1,l_1)(s,a,s')
:=
\frac{\beta}{2}q_1(s,a)^2
+
l_1(s,a)
\bigl(
u_g^\star(s,a)+\gamma_1(\mu q_1)(s')-q_1(s,a)
\bigr),
\]
and
\[
\begin{aligned}
Z_2^{\rm coup}(q_1,q_2,l_2)(s,a,s')
:=
\frac12q_2(s,a)^2
+
l_2(s,a)
\Bigl(
(I-\Pi_\mu)q_1(s,a)
+
g(s)+C
+
\gamma_2\Omega(q_2)(s')
-
q_2(s,a)
\Bigr).
\end{aligned}
\]
The empirical coupled estimator is
\[
(\widehat q_1^{\rm coup},\widehat q_2^{\rm coup},
\widehat l_1^{\rm coup},\widehat l_2^{\rm coup})
\in
\arg\min_{q_1\in\mathcal Q_1,\;q_2\in\mathcal Q_2}
\max_{l_1\in\mathcal L_1,\;l_2\in\mathcal L_2}
\widehat{\mathcal L}_{\rm coup}^{\beta}
(q_1,q_2,l_1,l_2).
\]

For the source dual class in the coupled formulation, we use the radius
\[
B_{L1}^{\rm coup}(\beta)
:=
\beta B_{L1}^{\rm self}+B_{L12}.
\]
Accordingly, define
\[
M_{1,\beta}^{\rm coup}
:=
\frac{\beta}{2}B_{Q1}^2
+
B_{L1}^{\rm coup}(\beta)B_{\rm Res,1}
=
\frac{\beta}{2}B_{Q1}^2
+
\bigl(\beta B_{L1}^{\rm self}+B_{L12}\bigr)B_{\rm Res,1}.
\]
When \(\beta=0\), this reduces to
\[
M_{1,0}^{\rm coup}
=
B_{L12}B_{\rm Res,1}
=
M_1^{\rm coup}.
\]

\begin{lemma}[Uniform concentration for the coupled objective]
\label{lem:app_coupled_concentration}
Under Assumptions~\ref{ass:app_global_finite_realizable_bounded} and
\ref{ass:app_global_coverage_dual}, with probability at least \(1-\delta\),
\[
\sup_{q_1\in\mathcal Q_1,\;l_1\in\mathcal L_1}
\left|
\widehat{\mathbb E}_{n_1}Z_{1,\beta}^{\rm coup}
-
\mathbb E Z_{1,\beta}^{\rm coup}
\right|
\le
\varepsilon_{1,\beta}^{\rm coup},
\]
and
\[
\sup_{q_1\in\mathcal Q_1,\;q_2\in\mathcal Q_2,\;l_2\in\mathcal L_2}
\left|
\widehat{\mathbb E}_{n_2}Z_2^{\rm coup}
-
\mathbb E Z_2^{\rm coup}
\right|
\le
\varepsilon_2^{\rm coup},
\]
where
\[
\varepsilon_{1,\beta}^{\rm coup}
:=
M_{1,\beta}^{\rm coup}
\sqrt{\frac{2\log(2N_1/\delta)}{n_1}},
\qquad
\varepsilon_2^{\rm coup}
:=
M_2
\sqrt{\frac{2\log(2N_2^{\rm coup}/\delta)}{n_2}}.
\]
\end{lemma}

\begin{proof}
For the source block, the source residual is bounded by \(B_{\rm Res,1}\), and
the source dual radius is \(B_{L1}^{\rm coup}(\beta)\). Hence
\[
|Z_{1,\beta}^{\rm coup}(q_1,l_1)|
\le
\frac{\beta}{2}B_{Q1}^2
+
B_{L1}^{\rm coup}(\beta)B_{\rm Res,1}
=
M_{1,\beta}^{\rm coup}.
\]
For the target block, the residual is bounded by \(B_{\rm Res,2}\), and therefore
\[
|Z_2^{\rm coup}(q_1,q_2,l_2)|
\le
\frac12B_{Q2}^2+B_{L2}B_{\rm Res,2}
=
M_2.
\]
Hoeffding's inequality and union bounds over
\(\mathcal Q_1\times\mathcal L_1\) and
\(\mathcal Q_1\times\mathcal Q_2\times\mathcal L_2\) give the result.
\end{proof}

\begin{lemma}[Coupled empirical saddle comparison]
\label{lem:app_coupled_saddle_comparison}
On the event of Lemma~\ref{lem:app_coupled_concentration},
\[
\mathcal L_{\rm coup}^{\beta}
(\widehat q_1^{\rm coup},\widehat q_2^{\rm coup},l_1^\star,l_2^\star)
-
\mathcal L_{\rm coup}^{\beta}
(q_1^\star,q_2^\star,l_1^\star,l_2^\star)
\le
2(\varepsilon_{1,\beta}^{\rm coup}+\varepsilon_2^{\rm coup}).
\]
\end{lemma}

\begin{proof}
Let
\[
\widehat q^{\rm coup}
:=
(\widehat q_1^{\rm coup},\widehat q_2^{\rm coup}),
\qquad
q^\star:=(q_1^\star,q_2^\star),
\qquad
l^\star:=(l_1^\star,l_2^\star).
\]
Let \(\widehat l(\widehat q^{\rm coup})\) be an empirical maximizer at
\(\widehat q^{\rm coup}\), and let \(\widehat l(q^\star)\) be an empirical
maximizer at \(q^\star\). Adding and subtracting empirical terms gives
\[
\begin{aligned}
&
\mathcal L_{\rm coup}^{\beta}(\widehat q^{\rm coup},l^\star)
-
\mathcal L_{\rm coup}^{\beta}(q^\star,l^\star)
\\
&=
\Bigl[
\mathcal L_{\rm coup}^{\beta}(q^\star,\widehat l(q^\star))
-
\mathcal L_{\rm coup}^{\beta}(q^\star,l^\star)
\Bigr]
\\
&\quad+
\Bigl[
\widehat{\mathcal L}_{\rm coup}^{\beta}(q^\star,\widehat l(q^\star))
-
\mathcal L_{\rm coup}^{\beta}(q^\star,\widehat l(q^\star))
\Bigr]
\\
&\quad+
\Bigl[
\widehat{\mathcal L}_{\rm coup}^{\beta}
(\widehat q^{\rm coup},\widehat l(\widehat q^{\rm coup}))
-
\widehat{\mathcal L}_{\rm coup}^{\beta}
(q^\star,\widehat l(q^\star))
\Bigr]
\\
&\quad+
\Bigl[
\widehat{\mathcal L}_{\rm coup}^{\beta}
(\widehat q^{\rm coup},l^\star)
-
\widehat{\mathcal L}_{\rm coup}^{\beta}
(\widehat q^{\rm coup},\widehat l(\widehat q^{\rm coup}))
\Bigr]
\\
&\quad+
\Bigl[
\mathcal L_{\rm coup}^{\beta}(\widehat q^{\rm coup},l^\star)
-
\widehat{\mathcal L}_{\rm coup}^{\beta}(\widehat q^{\rm coup},l^\star)
\Bigr].
\end{aligned}
\]
The first bracket is nonpositive by the population saddle property. The third
bracket is nonpositive by empirical minimization over \(q\). The fourth bracket
is nonpositive by empirical maximization over \(l\). The second and fifth
brackets are each bounded by
\(\varepsilon_{1,\beta}^{\rm coup}+\varepsilon_2^{\rm coup}\) on the
concentration event. Therefore,
\[
\mathcal L_{\rm coup}^{\beta}(\widehat q^{\rm coup},l^\star)
-
\mathcal L_{\rm coup}^{\beta}(q^\star,l^\star)
\le
2(\varepsilon_{1,\beta}^{\rm coup}+\varepsilon_2^{\rm coup}).
\]
\end{proof}

\begin{theorem}[Global value bound for \(\beta>0\) coupled minimax transfer]
\label{thm:app_coupled_global_bound_beta_positive}
Assume \(\beta>0\). Under Assumptions~\ref{ass:app_global_finite_realizable_bounded}
and \ref{ass:app_global_coverage_dual}, with probability at least \(1-\delta\),
\[
\beta
\|\widehat q_1^{\rm coup}-q_1^\star\|_{\rho_1}^2
+
\|\widehat q_2^{\rm coup}-q_2^\star\|_{\rho_2}^2
\le
4
\left(
M_{1,\beta}^{\rm coup}
\sqrt{\frac{2\log(2N_1/\delta)}{n_1}}
+
M_2
\sqrt{\frac{2\log(2N_2^{\rm coup}/\delta)}{n_2}}
\right).
\]
\end{theorem}

\begin{proof}
By the population quadratic growth inequality for the coupled saddle point,
\[
\mathcal L_{\rm coup}^{\beta}
(\widehat q_1^{\rm coup},\widehat q_2^{\rm coup},l_1^\star,l_2^\star)
-
\mathcal L_{\rm coup}^{\beta}
(q_1^\star,q_2^\star,l_1^\star,l_2^\star)
\ge
\frac{\beta}{2}
\|\widehat q_1^{\rm coup}-q_1^\star\|_{\rho_1}^2
+
\frac12
\|\widehat q_2^{\rm coup}-q_2^\star\|_{\rho_2}^2.
\]
Combining this with Lemma~\ref{lem:app_coupled_saddle_comparison} yields
\[
\frac{\beta}{2}
\|\widehat q_1^{\rm coup}-q_1^\star\|_{\rho_1}^2
+
\frac12
\|\widehat q_2^{\rm coup}-q_2^\star\|_{\rho_2}^2
\le
2(\varepsilon_{1,\beta}^{\rm coup}+\varepsilon_2^{\rm coup}).
\]
Multiplying by \(2\) and substituting the definitions of
\(\varepsilon_{1,\beta}^{\rm coup}\) and \(\varepsilon_2^{\rm coup}\) proves the
claim.
\end{proof}

\begin{corollary}[Global value bound for \(\beta=0\) coupled minimax transfer]
\label{cor:app_coupled_global_bound_beta_zero}
Under the \(\beta=0\) coupled formulation, with probability at least
\(1-\delta\),
\[
\|\widehat q_2^{\rm coup}-q_2^\star\|_{\rho_2}^2
\le
4
\left(
M_1^{\rm coup}
\sqrt{\frac{2\log(2N_1/\delta)}{n_1}}
+
M_2
\sqrt{\frac{2\log(2N_2^{\rm coup}/\delta)}{n_2}}
\right),
\]
where
\[
M_1^{\rm coup}=B_{L12}B_{\rm Res,1}.
\]
\end{corollary}

\begin{proof}
When \(\beta=0\), the coupled population Lagrangian no longer contains the
source quadratic term. Thus the quadratic growth inequality does not provide a
norm bound for
\(\widehat q_1^{\rm coup}-q_1^\star\). However, the target quadratic term is
still present, and the same population saddle argument gives
\[
\mathcal L_{\rm coup}^{0}
(\widehat q_1^{\rm coup},\widehat q_2^{\rm coup},l_1^\star,l_2^\star)
-
\mathcal L_{\rm coup}^{0}
(q_1^\star,q_2^\star,l_1^\star,l_2^\star)
\ge
\frac12
\|\widehat q_2^{\rm coup}-q_2^\star\|_{\rho_2}^2.
\]
Lemma~\ref{lem:app_coupled_saddle_comparison}, applied with \(\beta=0\), yields
\[
\frac12
\|\widehat q_2^{\rm coup}-q_2^\star\|_{\rho_2}^2
\le
2(\varepsilon_{1,0}^{\rm coup}+\varepsilon_2^{\rm coup}).
\]
For \(\beta=0\),
\[
M_{1,0}^{\rm coup}
=
B_{L12}B_{\rm Res,1}
=
M_1^{\rm coup}.
\]
Substituting the definitions of
\(\varepsilon_{1,0}^{\rm coup}\) and \(\varepsilon_2^{\rm coup}\) proves the
claim.
\end{proof}

\begin{remark}[Why the \(\beta=0\) result still matches the global theorem]
The role of \(\beta>0\) is to give explicit quadratic growth in the source
primal variable \(q_1\). When \(\beta=0\), this source-side strong convexity is
lost, so the coupled saddle comparison no longer controls
\(\|\widehat q_1^{\rm coup}-q_1^\star\|_{\rho_1}\). This does not affect
Theorem~\ref{thm:global_q2_bounds}, because the stated coupled guarantee only
requires control of the transferred target value \(q_2^\star\). The target
quadratic term remains present, and therefore the same comparison argument still
yields the claimed \(q_2\)-error bound.
\end{remark}

\subsection{Combined global finite-sample theorem}
\label{app:combined_global_finite_sample}

We finally restate the two value bounds in the form used in
Section~\ref{sec:global_policy_guarantees}.

\begin{theorem}[Global finite-sample bounds for transferred values]
\label{thm:app_global_q2_bounds}
Under the assumptions stated above, the following statements hold.

\medskip
\noindent
\textbf{(i) Modular minimax transfer.}
With probability at least \(1-\delta\),
\[
\|\widehat q_2^{\rm mod}-q_2^\star\|_{\rho_2}^2
\le
8
\left(
M_1^{\rm mod}
\sqrt{\frac{2\log(4N_1/\delta)}{n_1}}
+
M_2
\sqrt{\frac{2\log(4N_2^{\rm mod}/\delta)}{n_2}}
\right).
\]

\medskip
\noindent
\textbf{(ii) Coupled minimax transfer.}
Under the \(\beta=0\) coupled formulation, with probability at least
\(1-\delta\),
\[
\|\widehat q_2^{\rm coup}-q_2^\star\|_{\rho_2}^2
\le
4
\left(
M_1^{\rm coup}
\sqrt{\frac{2\log(2N_1/\delta)}{n_1}}
+
M_2
\sqrt{\frac{2\log(2N_2^{\rm coup}/\delta)}{n_2}}
\right).
\]
\end{theorem}

\begin{proof}
The modular statement is exactly
Theorem~\ref{thm:app_modular_global_bound}. The coupled statement is exactly
Corollary~\ref{cor:app_coupled_global_bound_beta_zero}. This proves the combined
theorem.
\end{proof}

\subsection{Policy performance guarantees}
\label{app:policy_performance_global}

This subsection proves Theorem~\ref{thm:policy_regret}. The proof has three
steps. First, we control the sensitivity of the target policy induced by the
weighted softmax map. Second, we use a regularized performance-difference
identity to reduce policy regret to a policy-distance term under the optimal
target occupancy. Third, we use target concentrability to convert this term into
the \(L_2(\rho_2)\) transferred-value error controlled by
Theorem~\ref{thm:global_q2_bounds}.

Recall that the target regularized performance is
\[
J_2(\pi)
=
\frac{1}{1-\gamma_2}
\mathbb E_{(s,a)\sim d^{P_2}_{\pi,\rho_2}}
\left[
r_{2s}^\star(s,a)
-
\tau_2
\log\frac{\pi(a\mid s)}{\pi_{2,\rm ref}(a\mid s)}
\right],
\]
where
\[
d^{P_2}_{\pi,\rho_2}
=
(1-\gamma_2)
\sum_{t=0}^{\infty}
\gamma_2^t
\big((P_2^\pi)^\top\big)^t\rho_2.
\]
The optimal target policy and the estimated target policy are
\[
\pi_2^\star(a\mid s)
=
\frac{\pi_{2,\rm ref}(a\mid s)\exp(q_2^\star(s,a)/\tau_2)}
{\sum_{a'}\pi_{2,\rm ref}(a'\mid s)\exp(q_2^\star(s,a')/\tau_2)},
\]
and
\[
\widehat\pi_2(a\mid s)
=
\frac{\pi_{2,\rm ref}(a\mid s)\exp(\widehat q_2(s,a)/\tau_2)}
{\sum_{a'}\pi_{2,\rm ref}(a'\mid s)\exp(\widehat q_2(s,a')/\tau_2)}.
\]

\begin{lemma}[Softmax sensitivity via the Jacobian]
\label{lem:app_softmax_sensitivity}
Fix a state \(s\). For any \(q(s,\cdot),q'(s,\cdot)\in\mathbb R^{|\mathcal A|}\),
define
\[
\pi_q(a\mid s)
:=
\frac{\pi_{2,\rm ref}(a\mid s)\exp(q(s,a)/\tau_2)}
{\sum_{a'}\pi_{2,\rm ref}(a'\mid s)\exp(q(s,a')/\tau_2)}.
\]
Then
\begin{equation}
\label{eq:app_softmax_L1L2}
\|\pi_q(\cdot\mid s)-\pi_{q'}(\cdot\mid s)\|_1
\le
\frac{\sqrt{|\mathcal A|}}{\tau_2}
\|q(s,\cdot)-q'(s,\cdot)\|_2 .
\end{equation}
In particular,
\[
\|\widehat\pi_2(\cdot\mid s)-\pi_2^\star(\cdot\mid s)\|_1
\le
\frac{\sqrt{|\mathcal A|}}{\tau_2}
\|\widehat q_2(s,\cdot)-q_2^\star(s,\cdot)\|_2 .
\]
\end{lemma}

\begin{proof}
Fix \(s\) and write \(\pi_{2,\rm ref}(a)=\pi_{2,\rm ref}(a\mid s)\). Let
\(\sigma_\lambda\) denote the standard softmax with inverse temperature
\(\lambda\):
\[
\sigma_\lambda(z)(a)
=
\frac{\exp(\lambda z(a))}
{\sum_{a'}\exp(\lambda z(a'))}.
\]
Define the shifted logits
\[
z_q(a):=q(s,a)+\tau_2\log\pi_{2,\rm ref}(a),
\qquad
z_{q'}(a):=q'(s,a)+\tau_2\log\pi_{2,\rm ref}(a).
\]
Then
\[
\pi_q(\cdot\mid s)=\sigma_{1/\tau_2}(z_q),
\qquad
\pi_{q'}(\cdot\mid s)=\sigma_{1/\tau_2}(z_{q'}).
\]

The Jacobian of the standard softmax \(\sigma_\lambda\) is
\[
D\sigma_\lambda(z)
=
\lambda
\left(
\operatorname{diag}(\sigma_\lambda(z))
-
\sigma_\lambda(z)\sigma_\lambda(z)^\top
\right).
\]
The matrix
\[
\operatorname{diag}(p)-pp^\top
\]
is positive semidefinite and has \(\ell_2\)-operator norm at most \(1\) for
every probability vector \(p\). Hence
\[
\|D\sigma_\lambda(z)\|_{2\to2}\le \lambda.
\]
Equivalently, by the mean-value theorem,
\[
\|\sigma_\lambda(z)-\sigma_\lambda(z')\|_2
\le
\lambda\|z-z'\|_2.
\]
This is the standard softmax Lipschitz bound by \citet{gao2017properties} (Proposition~4).

Applying this with \(\lambda=1/\tau_2\) gives
\[
\|\pi_q(\cdot\mid s)-\pi_{q'}(\cdot\mid s)\|_2
\le
\frac1{\tau_2}\|z_q-z_{q'}\|_2.
\]
The shifted-logit terms involving \(\tau_2\log\pi_{2,\rm ref}\) cancel, so
\[
z_q-z_{q'}=q(s,\cdot)-q'(s,\cdot).
\]
Therefore,
\[
\|\pi_q(\cdot\mid s)-\pi_{q'}(\cdot\mid s)\|_2
\le
\frac1{\tau_2}
\|q(s,\cdot)-q'(s,\cdot)\|_2.
\]
Finally, using
\[
\|v\|_1\le \sqrt{|\mathcal A|}\|v\|_2
\]
proves \eqref{eq:app_softmax_L1L2}.
\end{proof}

\begin{lemma}[Regularized performance difference]
\label{lem:app_regularized_perf_diff}
For any policy \(\pi\), let \(q_2^\pi\) be its regularized action-value
function, defined by
\[
q_2^\pi(s,a)
:=
r_{2s}^\star(s,a)
-
\tau_2\log\frac{\pi(a\mid s)}{\pi_{2,\rm ref}(a\mid s)}
+
\gamma_2
\mathbb E_{s'\sim P_2(\cdot\mid s,a)}
V_2^\pi(s'),
\]
where
\[
V_2^\pi(s)
:=
\left\langle
\pi(\cdot\mid s),
q_2^\pi(s,\cdot)
\right\rangle .
\]
Assume \(\|q_2^\pi\|_\infty\le B_{Q2}\). Then
\[
J_2(\pi_2^\star)-J_2(\pi)
\le
\frac{B_{Q2}}{1-\gamma_2}
\mathbb E_{s\sim d^{P_2,\mathcal S}_{\pi_2^\star,\rho_2}}
\left[
\|\pi_2^\star(\cdot\mid s)-\pi(\cdot\mid s)\|_1
\right],
\]
where
\[
d^{P_2,\mathcal S}_{\pi_2^\star,\rho_2}(s)
:=
\sum_a d^{P_2}_{\pi_2^\star,\rho_2}(s,a).
\]
In particular, this applies to \(\pi=\widehat\pi_2\) whenever
\(\|q_2^{\widehat\pi_2}\|_\infty\le B_{Q2}\).
\end{lemma}

\begin{proof}
Let \(T^\pi\) denote the regularized Bellman evaluation operator
\[
(T^\pi V)(s)
:=
\left\langle
\pi(\cdot\mid s),
r_{2s}^\star(s,\cdot)
-
\tau_2\log\frac{\pi(\cdot\mid s)}{\pi_{2,\rm ref}(\cdot\mid s)}
+
\gamma_2 P_2V(s,\cdot)
\right\rangle .
\]
Then \(V_2^\pi=T^\pi V_2^\pi\). The optimal value satisfies
\[
V_2^\star=T^{\pi_2^\star}V_2^\star,
\]
because \(\pi_2^\star\) is the soft-optimal policy induced by \(q_2^\star\).

We compare \(V_2^\star\) and \(V_2^\pi\) using the Bellman operator of
\(\pi_2^\star\):
\[
\begin{aligned}
V_2^\star-V_2^\pi
&=
T^{\pi_2^\star}V_2^\star
-
V_2^\pi
\\
&=
T^{\pi_2^\star}V_2^\star
-
T^{\pi_2^\star}V_2^\pi
+
T^{\pi_2^\star}V_2^\pi
-
V_2^\pi
\\
&=
\gamma_2 P_2^{\pi_2^\star}(V_2^\star-V_2^\pi)
+
\Bigl(T^{\pi_2^\star}V_2^\pi-V_2^\pi\Bigr).
\end{aligned}
\]
Thus
\[
V_2^\star-V_2^\pi
=
(I-\gamma_2P_2^{\pi_2^\star})^{-1}
\Bigl(T^{\pi_2^\star}V_2^\pi-V_2^\pi\Bigr).
\]
Taking expectation with respect to the initial distribution and using the
discounted occupancy representation gives
\[
J_2(\pi_2^\star)-J_2(\pi)
=
\frac{1}{1-\gamma_2}
\mathbb E_{s\sim d^{P_2,\mathcal S}_{\pi_2^\star,\rho_2}}
\left[
T^{\pi_2^\star}V_2^\pi(s)-V_2^\pi(s)
\right].
\]

It remains to bound the one-step difference. By the definition of
\(q_2^\pi\),
\[
q_2^\pi(s,a)
=
r_{2s}^\star(s,a)
-
\tau_2\log\frac{\pi(a\mid s)}{\pi_{2,\rm ref}(a\mid s)}
+
\gamma_2
\mathbb E_{s'\sim P_2(\cdot\mid s,a)}
V_2^\pi(s').
\]
Therefore,
\[
\begin{aligned}
T^{\pi_2^\star}V_2^\pi(s)
&=
\left\langle
\pi_2^\star(\cdot\mid s),
q_2^\pi(s,\cdot)
+
\tau_2\log\frac{\pi(\cdot\mid s)}{\pi_{2,\rm ref}(\cdot\mid s)}
-
\tau_2\log\frac{\pi_2^\star(\cdot\mid s)}{\pi_{2,\rm ref}(\cdot\mid s)}
\right\rangle
\\
&=
\left\langle
\pi_2^\star(\cdot\mid s),
q_2^\pi(s,\cdot)
\right\rangle
-
\tau_2
\mathrm{KL}\!\left(
\pi_2^\star(\cdot\mid s)
\,\middle\|\,
\pi(\cdot\mid s)
\right).
\end{aligned}
\]
On the other hand,
\[
V_2^\pi(s)
=
\left\langle
\pi(\cdot\mid s),
q_2^\pi(s,\cdot)
\right\rangle.
\]
Hence
\[
\begin{aligned}
T^{\pi_2^\star}V_2^\pi(s)-V_2^\pi(s)
&=
\left\langle
\pi_2^\star(\cdot\mid s)-\pi(\cdot\mid s),
q_2^\pi(s,\cdot)
\right\rangle
\\
&\quad
-
\tau_2
\mathrm{KL}\!\left(
\pi_2^\star(\cdot\mid s)
\,\middle\|\,
\pi(\cdot\mid s)
\right)
\\
&\le
\left\langle
\pi_2^\star(\cdot\mid s)-\pi(\cdot\mid s),
q_2^\pi(s,\cdot)
\right\rangle .
\end{aligned}
\]
The KL term is nonnegative, so it appears with a minus sign and can be dropped
for an upper bound.

Finally, by Hölder's inequality and \(\|q_2^\pi\|_\infty\le B_{Q2}\),
\[
\left\langle
\pi_2^\star(\cdot\mid s)-\pi(\cdot\mid s),
q_2^\pi(s,\cdot)
\right\rangle
\le
B_{Q2}
\|\pi_2^\star(\cdot\mid s)-\pi(\cdot\mid s)\|_1 .
\]
Substituting this pointwise bound into the occupancy expression proves the
claim.
\end{proof}

\begin{lemma}[From policy sensitivity to \(L_2(\rho_2)\) value error]
\label{lem:app_policy_sensitivity_to_value_error}
Assume
\[
\left\|
\frac{d^{P_2}_{\pi_2^\star,\rho_2}}{\rho_2}
\right\|_\infty
\le
\kappa_2,
\qquad
\pi_2^\star(a\mid s)\ge \epsilon_{\pi2}^\star
\quad
\forall(s,a).
\]
Then
\begin{equation}
\label{eq:app_policy_sensitivity_to_q_error}
\mathbb E_{s\sim d^{P_2,\mathcal S}_{\pi_2^\star,\rho_2}}
\left[
\|\widehat\pi_2(\cdot\mid s)-\pi_2^\star(\cdot\mid s)\|_1
\right]
\le
\frac{\sqrt{|\mathcal A|}}{\tau_2}
\sqrt{
\frac{\kappa_2}{\epsilon_{\pi2}^\star}
}
\|\widehat q_2-q_2^\star\|_{\rho_2}.
\end{equation}
\end{lemma}

\begin{proof}[Proof of Theorem~\ref{thm:policy_regret}]
Combining Lemma~\ref{lem:app_regularized_perf_diff} with
Lemma~\ref{lem:app_policy_sensitivity_to_value_error} gives, for any estimate
\(\widehat q_2\) and induced policy \(\widehat\pi_2\),
\[
\begin{aligned}
J_2(\pi_2^\star)-J_2(\widehat\pi_2)
&\le
\frac{B_{Q2}}{1-\gamma_2}
\mathbb E_{s\sim d^{P_2,\mathcal S}_{\pi_2^\star,\rho_2}}
\left[
\|\widehat\pi_2(\cdot\mid s)-\pi_2^\star(\cdot\mid s)\|_1
\right]
\\
&\le
\frac{B_{Q2}}{1-\gamma_2}
\cdot
\frac{\sqrt{|\mathcal A|}}{\tau_2}
\cdot
\sqrt{
\frac{\kappa_2}{\epsilon_{\pi2}^\star}
}
\|\widehat q_2-q_2^\star\|_{\rho_2}.
\end{aligned}
\]
This proves \eqref{eq:policy_regret_from_q}.

Applying this inequality with
\(\widehat q_2=\widehat q_2^{\rm mod}\) and using
\[
\|\widehat q_2^{\rm mod}-q_2^\star\|_{\rho_2}^2
\le
\mathcal B_{\rm mod}(\delta)
\]
gives
\[
J_2(\pi_2^\star)-J_2(\widehat\pi_2^{\rm mod})
\le
C_\pi\sqrt{\mathcal B_{\rm mod}(\delta)}.
\]
Similarly, applying the same inequality with
\(\widehat q_2=\widehat q_2^{\rm coup}\) and using
\[
\|\widehat q_2^{\rm coup}-q_2^\star\|_{\rho_2}^2
\le
\mathcal B_{\rm coup}(\delta)
\]
gives
\[
J_2(\pi_2^\star)-J_2(\widehat\pi_2^{\rm coup})
\le
C_\pi\sqrt{\mathcal B_{\rm coup}(\delta)}.
\]
Here
\[
C_\pi
=
\frac{B_{Q2}\sqrt{|\mathcal A|\kappa_2}}
{(1-\gamma_2)\tau_2\sqrt{\epsilon_{\pi2}^\star}}.
\]
This completes the proof.
\end{proof}

\section{Proofs for Section \ref{sec:learned_behavior_policy}}
\label{app:learned_behavior_policy}

This appendix proves Theorem~\ref{thm:end_to_end_learned_pi_b1}. The proof has
three steps. First, empirical cross-entropy minimization gives expected
conditional KL control for the learned behavior policy. Second, this KL control
is converted into an \(L_2(\rho_1)\) source-signal error bound. Third, this
signal perturbation is propagated through the source Bellman equation, the
transferred reward map, and the target soft Bellman equation.

Throughout this appendix, write
\[
u_g^\star(s,a)
:=
\log\frac{\pi_{b1}(a\mid s)}{\pi_{1,\rm ref}(a\mid s)}-g(s),
\qquad
\widehat u_g(s,a)
:=
\log\frac{\widehat\pi_{b1}(a\mid s)}{\pi_{1,\rm ref}(a\mid s)}-g(s),
\]
and
\[
\Delta_u(s,a)
:=
\widehat u_g(s,a)-u_g^\star(s,a)
=
\log\frac{\widehat\pi_{b1}(a\mid s)}{\pi_{b1}(a\mid s)}.
\]

\subsection{Cross-entropy control of the learned behavior policy}

Let
\[
\mathcal D_b=\{(s_i^b,a_i^b)\}_{i=1}^{n_b}
\]
be i.i.d. source demonstrations with
\[
s_i^b\sim\rho_1^{\mathcal S},
\qquad
a_i^b\sim\pi_{b1}(\cdot\mid s_i^b).
\]
For \(\pi\in\Pi_{b1}\), define the empirical and population cross-entropy risks
\[
\widehat{\mathcal R}_b(\pi)
:=
\frac1{n_b}\sum_{i=1}^{n_b}
-\log\pi(a_i^b\mid s_i^b),
\]
and
\[
\mathcal R_b(\pi)
:=
\mathbb E_{s\sim\rho_1^{\mathcal S},\,a\sim\pi_{b1}(\cdot\mid s)}
[-\log\pi(a\mid s)].
\]
The estimator is
\[
\widehat\pi_{b1}
\in
\arg\min_{\pi\in\Pi_{b1}}
\widehat{\mathcal R}_b(\pi).
\]

\begin{assumption}[Finite behavior class and clipping]
\label{ass:app_behavior_finite_clip}
The class \(\Pi_{b1}\) is finite and realizable:
\[
|\Pi_{b1}|<\infty,
\qquad
\pi_{b1}\in\Pi_{b1}.
\]
Moreover, every candidate policy is clipped below by \(\epsilon\in(0,1]\):
\[
\pi(a\mid s)\ge\epsilon,
\qquad
\forall \pi\in\Pi_{b1},\ (s,a).
\]
\end{assumption}

\begin{lemma}[Uniform deviation for bounded log-loss]
\label{lem:app_uniform_dev_ce}
Under Assumption~\ref{ass:app_behavior_finite_clip}, with probability at least
\(1-\delta\),
\[
\sup_{\pi\in\Pi_{b1}}
\left|
\widehat{\mathcal R}_b(\pi)-\mathcal R_b(\pi)
\right|
\le
\log(1/\epsilon)
\sqrt{\frac{2\log(|\Pi_{b1}|/\delta)}{n_b}}.
\]
\end{lemma}

\begin{proof}
For each fixed \(\pi\in\Pi_{b1}\), the loss
\[
-\log\pi(a\mid s)
\]
is bounded in \([0,\log(1/\epsilon)]\). Hoeffding's inequality gives
\[
\mathbb P\left(
\left|
\widehat{\mathcal R}_b(\pi)-\mathcal R_b(\pi)
\right|
>t
\right)
\le
2\exp\left(
-\frac{2n_bt^2}{\log^2(1/\epsilon)}
\right).
\]
A union bound over \(\Pi_{b1}\) and the displayed choice of \(t\) prove the
claim.
\end{proof}

\begin{lemma}[Cross-entropy excess risk equals expected conditional KL]
\label{lem:app_ce_to_expected_kl}
Under Assumption~\ref{ass:app_behavior_finite_clip}, with probability at least
\(1-\delta\),
\[
\mathcal R_b(\widehat\pi_{b1})-\mathcal R_b(\pi_{b1})
\le
\varepsilon_{\rm CE}(n_b,\delta),
\]
where
\[
\varepsilon_{\rm CE}(n_b,\delta)
:=
2\log(1/\epsilon)
\sqrt{\frac{2\log(|\Pi_{b1}|/\delta)}{n_b}}.
\]
Moreover,
\[
\mathcal R_b(\widehat\pi_{b1})-\mathcal R_b(\pi_{b1})
=
\mathbb E_{s\sim\rho_1^{\mathcal S}}
\mathrm{KL}\!\left(
\pi_{b1}(\cdot\mid s)
\,\middle\|\,
\widehat\pi_{b1}(\cdot\mid s)
\right).
\]
\end{lemma}

\begin{proof}
Let
\[
\varepsilon_b
:=
\sup_{\pi\in\Pi_{b1}}
\left|
\widehat{\mathcal R}_b(\pi)-\mathcal R_b(\pi)
\right|.
\]
By empirical optimality of \(\widehat\pi_{b1}\),
\[
\widehat{\mathcal R}_b(\widehat\pi_{b1})
\le
\widehat{\mathcal R}_b(\pi_{b1}).
\]
Therefore,
\[
\begin{aligned}
\mathcal R_b(\widehat\pi_{b1})
&\le
\widehat{\mathcal R}_b(\widehat\pi_{b1})+\varepsilon_b \\
&\le
\widehat{\mathcal R}_b(\pi_{b1})+\varepsilon_b \\
&\le
\mathcal R_b(\pi_{b1})+2\varepsilon_b .
\end{aligned}
\]
Applying Lemma~\ref{lem:app_uniform_dev_ce} gives the stated excess-risk bound.

For the identity, condition on a state \(s\). Then
\[
\begin{aligned}
&
\mathbb E_{a\sim\pi_{b1}(\cdot\mid s)}
[-\log\widehat\pi_{b1}(a\mid s)]
-
\mathbb E_{a\sim\pi_{b1}(\cdot\mid s)}
[-\log\pi_{b1}(a\mid s)]
\\
&\qquad =
\sum_a
\pi_{b1}(a\mid s)
\log
\frac{\pi_{b1}(a\mid s)}{\widehat\pi_{b1}(a\mid s)}
=
\mathrm{KL}\!\left(
\pi_{b1}(\cdot\mid s)
\,\middle\|\,
\widehat\pi_{b1}(\cdot\mid s)
\right).
\end{aligned}
\]
Taking expectation over \(s\sim\rho_1^{\mathcal S}\) proves the identity.
\end{proof}

\subsection{From expected KL to log-policy \(L_2(\rho_1)\) error}

\begin{lemma}[Expected KL to \(L_2\) log-policy error]
\label{lem:app_kl_to_l2_log_policy}
Under Assumption~\ref{ass:app_behavior_finite_clip},
\[
\|\Delta_u\|_{\rho_1}^2
\le
\frac{2}{\epsilon}
\mathbb E_{s\sim\rho_1^{\mathcal S}}
\mathrm{KL}\!\left(
\pi_{b1}(\cdot\mid s)
\,\middle\|\,
\widehat\pi_{b1}(\cdot\mid s)
\right).
\]
Consequently, on the event of Lemma~\ref{lem:app_ce_to_expected_kl},
\[
\|\Delta_u\|_{\rho_1}
\le
\left(
\frac{2}{\epsilon}
\varepsilon_{\rm CE}(n_b,\delta)
\right)^{1/2}.
\]
\end{lemma}

\begin{proof}
Fix \(s\), and write
\[
p(a):=\pi_{b1}(a\mid s),
\qquad
\widehat p(a):=\widehat\pi_{b1}(a\mid s),
\qquad
r(a):=\frac{p(a)}{\widehat p(a)}.
\]
Then
\[
\Delta_u(s,a)=\log\widehat p(a)-\log p(a)=-\log r(a),
\]
and
\[
\mathrm{KL}(p\|\widehat p)
=
\sum_a
\widehat p(a)
\bigl(r(a)\log r(a)-r(a)+1\bigr).
\]
By clipping, \(p(a),\widehat p(a)\in[\epsilon,1]\), hence
\[
r(a)\in[\epsilon,1/\epsilon].
\]

We use the pointwise inequality
\[
r(\log r)^2
\le
\frac{2}{\epsilon}
\bigl(r\log r-r+1\bigr),
\qquad
r\in[\epsilon,1/\epsilon].
\]
Indeed, define
\[
h(r)
:=
\frac{2}{\epsilon}
\bigl(r\log r-r+1\bigr)
-
r(\log r)^2.
\]
Then \(h(1)=0\), and
\[
h'(r)
=
\log r
\left(
\frac{2}{\epsilon}-2-\log r
\right).
\]
For \(r\in[\epsilon,1/\epsilon]\),
\[
\log r\le \log(1/\epsilon),
\]
and
\[
\frac{2}{\epsilon}-2-\log(1/\epsilon)
=
\frac{2}{\epsilon}-2+\log\epsilon
\ge0,
\]
because the function \(x\mapsto 2/x-2+\log x\) is nonnegative on
\((0,1]\). Hence \(h'(r)<0\) for \(r<1\) and \(h'(r)>0\) for \(r>1\). Thus
\(r=1\) is the global minimizer of \(h\) on \([\epsilon,1/\epsilon]\), and
\(h(r)\ge0\).

Multiplying the pointwise inequality by \(\widehat p(a)\) and summing over
actions gives
\[
\sum_a p(a)(\log r(a))^2
=
\sum_a \widehat p(a)r(a)(\log r(a))^2
\le
\frac{2}{\epsilon}
\sum_a \widehat p(a)
\bigl(r(a)\log r(a)-r(a)+1\bigr).
\]
Therefore,
\[
\mathbb E_{a\sim\pi_{b1}(\cdot\mid s)}
[\Delta_u(s,a)^2]
\le
\frac{2}{\epsilon}
\mathrm{KL}\!\left(
\pi_{b1}(\cdot\mid s)
\,\middle\|\,
\widehat\pi_{b1}(\cdot\mid s)
\right).
\]
Averaging over \(s\sim\rho_1^{\mathcal S}\) proves the first claim. The second
claim follows from Lemma~\ref{lem:app_ce_to_expected_kl}.
\end{proof}

\subsection{Propagation through the transfer system}

Let \(q_1^\star\) be the source fixed point under the true source signal:
\[
q_1^\star
=
u_g^\star+\gamma_1P_1^\mu q_1^\star.
\]
Let \(\widetilde q_1^\star\) be the plug-in source fixed point:
\[
\widetilde q_1^\star
=
\widehat u_g+\gamma_1P_1^\mu\widetilde q_1^\star.
\]
Define
\[
h:=\widetilde q_1^\star-q_1^\star.
\]
The corresponding transferred rewards are
\[
r_C^\star
:=
(I-\Pi_\mu)q_1^\star+g+C,
\qquad
\widetilde r_C
:=
(I-\Pi_\mu)\widetilde q_1^\star+g+C.
\]

\begin{lemma}[Source fixed-point perturbation]
\label{lem:app_source_perturb_learned_pi}
Assume
\[
\left\|
\frac{d^{P_1}_{\mu,\rho_1}}{\rho_1}
\right\|_\infty
\le
\kappa_1.
\]
Then
\[
\|\widetilde q_1^\star-q_1^\star\|_{\rho_1}
\le
\frac{\sqrt{\kappa_1}}{1-\gamma_1}
\|\Delta_u\|_{\rho_1}.
\]
\end{lemma}

\begin{proof}
Subtracting the two source fixed-point equations gives
\[
\widetilde q_1^\star-q_1^\star
=
\Delta_u+\gamma_1P_1^\mu(\widetilde q_1^\star-q_1^\star).
\]
Hence
\[
\widetilde q_1^\star-q_1^\star
=
(I-\gamma_1P_1^\mu)^{-1}\Delta_u.
\]
Applying Lemma~\ref{lem:discounted_resolvent_l2_stability} with
\[
P=P_1^\mu,
\qquad
\gamma=\gamma_1,
\qquad
\rho=\rho_1,
\qquad
f=\Delta_u,
\]
and using
\[
\left\|
\frac{d^{P_1}_{\mu,\rho_1}}{\rho_1}
\right\|_\infty
\le
\kappa_1
\]
proves the claim.
\end{proof}

\begin{lemma}[Transferred reward perturbation]
\label{lem:app_reward_perturb_learned_pi}
Assume
\[
\chi_{21}:=
\left\|
\frac{\rho_2}{\rho_1}
\right\|_\infty
<\infty,
\]
and
\[
\kappa_{\mu\mid\pi_{b1}}
:=
\left\|
\frac{\mu}{\pi_{b1}}
\right\|_\infty
<\infty.
\]
Then
\[
\|\widetilde r_C-r_C^\star\|_{\rho_2}
\le
(1+\sqrt{\kappa_{\mu\mid\pi_{b1}}})
\sqrt{\chi_{21}}\,
\|\widetilde q_1^\star-q_1^\star\|_{\rho_1}.
\]
In particular, if \(\kappa_{\mu\mid\pi_{b1}}\le1/\epsilon_{b1}\), then
\[
\|\widetilde r_C-r_C^\star\|_{\rho_2}
\le
\frac{2\sqrt{\chi_{21}}}{\sqrt{\epsilon_{b1}}}
\|\widetilde q_1^\star-q_1^\star\|_{\rho_1}.
\]
\end{lemma}

\begin{proof}
Since
\[
\widetilde r_C-r_C^\star
=
(I-\Pi_\mu)(\widetilde q_1^\star-q_1^\star)
=
h-\Pi_\mu h,
\]
the change-of-measure inequality gives
\[
\|\widetilde r_C-r_C^\star\|_{\rho_2}
\le
\sqrt{\chi_{21}}\,
\|h-\Pi_\mu h\|_{\rho_1}.
\]
By the triangle inequality,
\[
\|h-\Pi_\mu h\|_{\rho_1}
\le
\|h\|_{\rho_1}+\|\Pi_\mu h\|_{\rho_1}.
\]
Write
\[
\rho_1(s,a)=\rho_1^{\mathcal S}(s)\pi_{b1}(a\mid s).
\]
Then Jensen's inequality gives
\[
\begin{aligned}
\|\Pi_\mu h\|_{\rho_1}^2
&=
\mathbb E_{s\sim\rho_1^{\mathcal S}}
\left[
\left(
\sum_a\mu(a\mid s)h(s,a)
\right)^2
\right]
\\
&\le
\mathbb E_{s\sim\rho_1^{\mathcal S}}
\sum_a\mu(a\mid s)h(s,a)^2
\\
&\le
\kappa_{\mu\mid\pi_{b1}}
\mathbb E_{s\sim\rho_1^{\mathcal S}}
\sum_a\pi_{b1}(a\mid s)h(s,a)^2
\\
&=
\kappa_{\mu\mid\pi_{b1}}\|h\|_{\rho_1}^2.
\end{aligned}
\]
Thus
\[
\|\Pi_\mu h\|_{\rho_1}
\le
\sqrt{\kappa_{\mu\mid\pi_{b1}}}\|h\|_{\rho_1}.
\]
Combining the preceding displays proves the first claim. The simplified bound
follows from
\[
1+\frac1{\sqrt{\epsilon_{b1}}}
\le
\frac2{\sqrt{\epsilon_{b1}}},
\qquad
\epsilon_{b1}\le1.
\]
\end{proof}

\begin{lemma}[Target fixed-point perturbation]
\label{lem:app_target_perturb_learned_pi}
Let \(q_2^\star\) and \(\widetilde q_2^\star\) be the target fixed points under
rewards \(r_C^\star\) and \(\widetilde r_C\), respectively:
\[
q_2^\star
=
r_C^\star
+
\gamma_2P_2\Omega_{\tau_2,\pi_{2,\rm ref}}(q_2^\star),
\]
and
\[
\widetilde q_2^\star
=
\widetilde r_C
+
\gamma_2P_2\Omega_{\tau_2,\pi_{2,\rm ref}}(\widetilde q_2^\star).
\]
Assume the pathwise target concentrability condition
\[
\sup_{\pi\in\Pi_{\rm path}}
\left\|
\frac{d^{P_2}_{\pi,\rho_2}}{\rho_2}
\right\|_\infty
\le
\kappa_2^{\rm path},
\]
where \(\Pi_{\rm path}\) contains the policies induced along the line segment
between \(q_2^\star\) and \(\widetilde q_2^\star\). Then
\[
\|\widetilde q_2^\star-q_2^\star\|_{\rho_2}
\le
\frac{\sqrt{\kappa_2^{\rm path}}}{1-\gamma_2}
\|\widetilde r_C-r_C^\star\|_{\rho_2}.
\]
\end{lemma}

\begin{proof}
Let
\[
\Delta_2:=\widetilde q_2^\star-q_2^\star,
\qquad
\Delta_r:=\widetilde r_C-r_C^\star.
\]
Subtracting the two target fixed-point equations gives
\[
\Delta_2
=
\Delta_r
+
\gamma_2P_2
\left[
\Omega_{\tau_2,\pi_{2,\rm ref}}(\widetilde q_2^\star)
-
\Omega_{\tau_2,\pi_{2,\rm ref}}(q_2^\star)
\right].
\]
By the fundamental theorem of calculus,
\[
\Omega_{\tau_2,\pi_{2,\rm ref}}(\widetilde q_2^\star)(s)
-
\Omega_{\tau_2,\pi_{2,\rm ref}}(q_2^\star)(s)
=
\int_0^1
\left\langle
\nabla\Omega_{\tau_2,\pi_{2,\rm ref}}(q_2^\star+t\Delta_2)(s,\cdot),
\Delta_2(s,\cdot)
\right\rangle
dt.
\]
The derivative \(\nabla\Omega_{\tau_2,\pi_{2,\rm ref}}(q)(s,\cdot)\) is the policy
induced by \(q\) at state \(s\). Therefore the averaged derivative along the
path defines a Markov state-action averaging operator, which we denote
\(P_2^{\bar\pi}\), with \(\bar\pi\in\Pi_{\rm path}\). Hence
\[
\Delta_2
=
\Delta_r+\gamma_2P_2^{\bar\pi}\Delta_2,
\]
or
\[
\Delta_2
=
(I-\gamma_2P_2^{\bar\pi})^{-1}\Delta_r.
\]
Applying Lemma~\ref{lem:discounted_resolvent_l2_stability} with
\[
P=P_2^{\bar\pi},
\qquad
\gamma=\gamma_2,
\qquad
\rho=\rho_2,
\qquad
f=\Delta_r,
\]
and using
\[
\left\|
\frac{d^{P_2}_{\bar\pi,\rho_2}}{\rho_2}
\right\|_\infty
\le
\kappa_2^{\rm path}
\]
proves the claim.
\end{proof}

\begin{corollary}[Behavior-policy estimation bias]
\label{cor:app_behavior_bias}
Under the assumptions of Lemmas~\ref{lem:app_source_perturb_learned_pi}--
\ref{lem:app_target_perturb_learned_pi},
\[
\|\widetilde q_2^\star-q_2^\star\|_{\rho_2}
\le
\frac{
(1+\sqrt{\kappa_{\mu\mid\pi_{b1}}})
\sqrt{\kappa_1\kappa_2^{\rm path}\chi_{21}}
}{
(1-\gamma_1)(1-\gamma_2)
}
\|\Delta_u\|_{\rho_1}.
\]
Consequently, on the event of Lemma~\ref{lem:app_kl_to_l2_log_policy},
\[
\|\widetilde q_2^\star-q_2^\star\|_{\rho_2}
\le
\frac{
(1+\sqrt{\kappa_{\mu\mid\pi_{b1}}})
\sqrt{\kappa_1\kappa_2^{\rm path}\chi_{21}}
}{
(1-\gamma_1)(1-\gamma_2)
}
\left(
\frac{2}{\epsilon}\varepsilon_{\rm CE}(n_b,\delta)
\right)^{1/2}.
\]
\end{corollary}

\begin{proof}
Combine Lemma~\ref{lem:app_target_perturb_learned_pi},
Lemma~\ref{lem:app_reward_perturb_learned_pi}, and
Lemma~\ref{lem:app_source_perturb_learned_pi}. This gives
\[
\|\widetilde q_2^\star-q_2^\star\|_{\rho_2}
\le
\frac{\sqrt{\kappa_2^{\rm path}}}{1-\gamma_2}
(1+\sqrt{\kappa_{\mu\mid\pi_{b1}}})
\sqrt{\chi_{21}}
\frac{\sqrt{\kappa_1}}{1-\gamma_1}
\|\Delta_u\|_{\rho_1}.
\]
The second claim follows by substituting
Lemma~\ref{lem:app_kl_to_l2_log_policy}.
\end{proof}

\subsection{Proof of the end-to-end theorem}

\begin{proof}[Proof of Theorem~\ref{thm:end_to_end_learned_pi_b1}]
Let \((\widetilde q_1^\star,\widetilde q_2^\star)\) be the plug-in population
target induced by \(\widehat u_g\). By the triangle inequality,
\[
\|\widehat q_2^{\rm alg}-q_2^\star\|_{\rho_2}
\le
\|\widehat q_2^{\rm alg}-\widetilde q_2^\star\|_{\rho_2}
+
\|\widetilde q_2^\star-q_2^\star\|_{\rho_2}.
\]
The first term is controlled by Theorem~\ref{thm:global_q2_bounds}, applied to
the plug-in population system:
\[
\|\widehat q_2^{\rm alg}-\widetilde q_2^\star\|_{\rho_2}
\le
\sqrt{\mathcal B_{\rm alg}(\delta)}.
\]
The second term is controlled by
Corollary~\ref{cor:app_behavior_bias}. Combining the two bounds gives
\[
\|\widehat q_2^{\rm alg}-q_2^\star\|_{\rho_2}
\le
\sqrt{\mathcal B_{\rm alg}(\delta)}
+
\frac{
(1+\sqrt{\kappa_{\mu\mid\pi_{b1}}})
\sqrt{\kappa_1\kappa_2^{\rm path}\chi_{21}}
}{
(1-\gamma_1)(1-\gamma_2)
}
\left(
\frac{2}{\epsilon}\varepsilon_{\rm CE}(n_b,\delta)
\right)^{1/2}.
\]
This proves \eqref{eq:end_to_end_q2_learned_pi_b1}. The simplified
\(\epsilon_{b1}\)-version follows from
\(\kappa_{\mu\mid\pi_{b1}}\le1/\epsilon_{b1}\).
\end{proof}
\section{Experimental Details and Additional Results}
\label{app:experimental_details}

This appendix provides additional details for the experiments in
Section~\ref{sec:experiments}. The goal is to make the experimental pipeline
fully reproducible and to document the additional diagnostics supporting the
main empirical claims.

Sections~\ref{app:simulator_mdp}--\ref{app:shift_construction} describe the
construction of the tabular sepsis environment and the source--target dynamics
shift. Section~\ref{app:outcome_reward_expert_policy} explains how the simulator
outcome reward is used only to construct the expert-like source behavior policy,
while Section~\ref{app:oracle_construction} defines the oracle source and target
quantities used for evaluation. Section~\ref{app:data_generation} gives the data
generation protocol, and Section~\ref{app:training_details} describes the three
estimators and their optimization settings. Section~\ref{app:metric_definitions}
defines all reported metrics. Section~\ref{app:additional_source_size_results} reports
additional source-size experiments, and Section~\ref{app:temperature_sensitivity}
summarizes the effect of target soft-control temperature. Sections~\ref{app:v2_decomposition} and
\ref{app:reward_recovery_diagnostic} provide diagnostics explaining the large
target state-value improvement and the reward-recovery behavior of
Coupled-Offset.  Finally,
Sections~\ref{app:beta_zero_ablation} and \ref{app:large_shift_results} provide
two robustness checks: an ablation with \(\beta=0\) and an experiment under a
larger source--target transition shift. Section~\ref{app:additional_findings_summary}
summarizes the resulting empirical conclusions.

\subsection{Simulator and tabular MDP construction}
\label{app:simulator_mdp}

We use the Gumbel-Max structural causal model (SCM) simulator of
\citet{oberst2019counterfactual} as a controlled healthcare decision-making
benchmark for sequential sepsis treatment. The simulator models patient-state
evolution through clinical variables and binary treatment decisions. We use a
finite tabular reduction of the simulator with \(|\mathcal S|=128\) states and
\(|\mathcal A|=8\) actions. Each action corresponds to a vector of three binary
treatment decisions,
\[
a=(a_{\rm abx},a_{\rm vent},a_{\rm vaso})\in\{0,1\}^3,
\]
where the components indicate antibiotics, ventilation, and vasopressors. We
encode actions by
\[
\mathrm{idx}(a)=4a_{\rm abx}+2a_{\rm vent}+a_{\rm vaso}.
\]
Thus action \(0\) corresponds to no treatment, and action \(7\) corresponds to
all three treatments.

The upstream simulator code is available in the
\href{https://github.com/clinicalml/gumbel-max-scm}{\texttt{clinicalml/gumbel-max-scm}}
repository under the MIT License (copyright 2019 Michael Oberst). Our submitted
supplement includes the code and configuration files used to generate the
tabular environments and reproduce the reported figures and tables.

Each episode has horizon \(20\). We use the simulator to generate source and
target trajectories and then work with the induced finite state-action
representation. This tabular construction allows us to compute oracle quantities
such as \(q_1^\star\), \(r^\star\), \(q_2^\star\), \(\pi_2^\star\), and
\(V_2^\star\), while retaining a clinically motivated sequential decision
structure.

\begin{table}[h]
\centering
\caption{Summary of the experimental environment.}
\label{tab:env_summary}
\begin{tabular}{ll}
\toprule
Component & Description \\
\midrule
Simulator & Gumbel-Max SCM sepsis simulator \\
State space & \(128\) tabular states \\
Action space & \(8\) treatment actions from three binary decisions \\
Episode horizon & \(20\) steps \\
Source environment & Original simulator-induced dynamics \(P_1\) \\
Target environment & Shifted dynamics \(P_2\) with preserved support \\
anchor policy \(\mu\) & Point mass on no-treatment action \(a_0\) \\
Target reference policy & Uniform policy over actions \\
\bottomrule
\end{tabular}
\end{table}

Table~\ref{tab:env_summary} summarizes the resulting experimental environment.

\subsection{Source--target shift construction}
\label{app:shift_construction}

The source and target stages share the same tabular state-action space but use
different transition kernels and discount factors; the target stage is further
evaluated under different soft-control temperatures.
We construct the source transition kernel \(P_1\) from the original
simulator dynamics. The target transition kernel \(P_2\) is obtained by
perturbing the same clinical transition mechanisms in the SCM while preserving
support. Thus, if a transition is infeasible under \(P_1\), it remains infeasible
under \(P_2\); otherwise, its probability may change under the target
perturbation. This creates a controlled source--target distribution shift while
avoiding support mismatch.

In the tabular representation, this gives two Markov kernels
\[
P_1(s'\mid s,a),
\qquad
P_2(s'\mid s,a),
\]
defined on the same finite state-action space. The source environment is used
for reward recovery, while the target environment is used for evaluating the
transferred reward and the induced target soft-control policy.

As a diagnostic for the magnitude of the source--target dynamics shift, we
compute the average per-state-action total variation distance
\[
d_{\rm TV}^{\rm avg}(P_1,P_2)
:=
\frac{1}{|\mathcal S||\mathcal A|}
\sum_{s,a}
\frac12
\sum_{s'}
\left|
P_1(s'\mid s,a)-P_2(s'\mid s,a)
\right|,
\]
together with the worst-case per-state-action distance
\[
d_{\rm TV}^{\rm max}(P_1,P_2)
:=
\max_{s,a}
\frac12
\sum_{s'}
\left|
P_1(s'\mid s,a)-P_2(s'\mid s,a)
\right|.
\]
In the main experiments, these quantities are
\[
d_{\rm TV}^{\rm avg}(P_1,P_2)=0.015,
\qquad
d_{\rm TV}^{\rm max}(P_1,P_2)=0.026.
\]
Thus the one-step kernel perturbation is mild at each state-action pair.
However, because policies are evaluated over multi-step trajectories, this
local transition shift can accumulate over time and induce a nontrivial shift in
the state-occupancy distribution.

We also consider a larger source--target dynamics shift in Appendix
Section~\ref{app:large_shift_results}. For this larger-shift target kernel, the
corresponding distances are
\[
d_{\rm TV}^{\rm avg}(P_1,P_2)=0.088,
\qquad
d_{\rm TV}^{\rm max}(P_1,P_2)=0.157.
\]
This setting provides a robustness check for the qualitative conclusions
obtained under the mild-shift target kernel.

\begin{remark}[Support preservation]
Support preservation isolates the effect of statistical source-to-target error
propagation. Since \(P_1\) and \(P_2\) share support, performance differences are
not driven by target transitions that are completely absent in the source
environment. Instead, the experiment focuses on how errors in the source
reward-recovery equation propagate into the shifted target Bellman equation.
\end{remark}

\subsection{Simulator outcome reward and expert-policy construction}
\label{app:outcome_reward_expert_policy}

The simulator provides an outcome reward \(R(s')\) depending on the next state.
Following the simulator convention, the agent receives reward \(-1\) when the
next state has at least two abnormal physiological variables, and reward \(+1\)
when the next state satisfies the simulator's discharge condition. All other
next states receive reward \(0\):
\[
R(s')
=
\begin{cases}
+1, & \text{if all physiological variables are normal and no treatment is active},\\
-1, & \text{if \(s'\) has at least two abnormal physiological variables},\\
0, & \text{otherwise}.
\end{cases}
\]
This simulator outcome reward is used only to construct the expert-like source
behavior policy \(\pi_{b1}\). It is not the reward used for target-stage
learning or evaluation. The transferred reward used in the target stage is the
anchor-normalized reward \(r^\star(s,a)\), recovered from the source IRL
construction.

Given \(P_1\), we first form the expected one-step simulator outcome reward
\[
\bar R_1(s,a)
=
\mathbb E_{s'\sim P_1(\cdot\mid s,a)}[R(s')]
=
\sum_{s'}P_1(s'\mid s,a)R(s').
\]
An expert-like source policy \(\pi_{b1}\) is then constructed by solving an
entropy-regularized soft Bellman equation under \(P_1\) and \(\bar R_1\):
\[
Q_b(s,a)
=
\bar R_1(s,a)
+
\gamma_b
\sum_{s'}P_1(s'\mid s,a)
\Xi_{\tau_b}(Q_b)(s'),
\]
where
\[
\Xi_\tau(Q)(s)
=
\tau\log\sum_{a'\in\mathcal A}\exp\{Q(s,a')/\tau\}.
\]
The resulting soft policy is
\[
\pi_{b1}^{\rm soft}(a\mid s)
=
\frac{\exp(Q_b(s,a)/\tau_b)}
{\sum_{a'}\exp(Q_b(s,a')/\tau_b)}.
\]

This policy is used to generate source demonstrations and to define
\(u_b^\star(s,a)=\log\pi_{b1}(a\mid s)\) in the source IRL module.

\subsection{Oracle construction}
\label{app:oracle_construction}

We compute oracle source and target quantities using the tabular kernels. In the
source stage, let \(u_b^\star(s,a):=\log\pi_{b1}(a\mid s)\). We use the
point-mass distribution on the no-treatment action \(a_0\) as the anchor policy:
\[
\mu(a\mid s)=\mathbf 1\{a=a_0\}.
\]
The anchor statewise reward is
\[
g(s)=\bar R_1(s,a_0),
\]
We set the anchor function \(g\) to the desired no-treatment reward level, so
that the recovered reward satisfies \(r^\star(s,a_0)=g(s)\).

so the normalization condition fixes the reward of the no-treatment action. In
the no-treatment anchor case, the recovered reward can be written as an action
contrast against \(a_0\).

The oracle source value \(q_1^\star\) solves
\[
q_1^\star
=
u_b^\star-g+\gamma_1P_1^\mu q_1^\star,
\]
or equivalently
\[
q_1^\star=(I-\gamma_1P_1^\mu)^{-1}(u_b^\star-g).
\]
The normalized oracle reward is
\[
r^\star=(I-\Pi_\mu)q_1^\star+g.
\]
Under the no-treatment anchor, this reduces to
\[
r^\star(s,a)=q_1^\star(s,a)-q_1^\star(s,a_0)+g(s).
\]

For target experiments, we use the shifted reward
\[
r_C^\star=r^\star+C,
\]
where \(C\) is chosen so that \(r_C^\star(s,a)\ge0\) for all state-action pairs.
The constant shift is useful for the minimax formulation and does not change the
induced soft policy.

In the target stage, we take the reference policy \(\pi_{2,\rm ref}\) to be
uniform over actions, corresponding to entropy-regularized soft control. This
avoids introducing an additional target behavior-policy estimation step. The
target oracle \(q_2^\star\) solves
\[
q_2^\star
=
r_C^\star+\gamma_2P_2\Omega_{\tau_2,\pi_{2,\rm ref}}(q_2^\star),
\]
where
\[
\Omega_{\tau_2,\pi_{2,\rm ref}}(q)(s)
=
\tau_2\log\sum_{a\in\mathcal A}
\pi_{2,\rm ref}(a\mid s)\exp\{q(s,a)/\tau_2\}.
\]
The corresponding oracle policy is
\[
\pi_2^\star(a\mid s)
=
\frac{
\pi_{2,\rm ref}(a\mid s)\exp(q_2^\star(s,a)/\tau_2)
}{
\sum_{a'}\pi_{2,\rm ref}(a'\mid s)\exp(q_2^\star(s,a')/\tau_2)
}.
\]
Finally, the target oracle state value is
\[
V_2^\star(s)
=
\sum_{a\in\mathcal A}\pi_2^\star(a\mid s)q_2^\star(s,a).
\]

\subsection{Data generation and evaluation protocol}
\label{app:data_generation}

For each experimental configuration, we generate independent source and target
datasets. Each dataset consists of trajectories of length \(20\). The source
dataset is generated from
\[
s_0\sim \rho,\qquad
a_t\sim \pi_{b1}(\cdot\mid s_t),\qquad
s_{t+1}\sim P_1(\cdot\mid s_t,a_t),
\]
and the target dataset is generated from
\[
s_0\sim \rho,\qquad
a_t\sim \pi_t(\cdot\mid s_t),\qquad
s_{t+1}\sim P_2(\cdot\mid s_t,a_t).
\]
The target logging policy \(\pi_t\) is an exploratory policy used only for data
collection. In our implementation, it is constructed as a mixture of the
expert-like source policy and the uniform policy:
\[
\pi_t(a\mid s)
=
(1-\epsilon_2)\pi_{b1}(a\mid s)
+
\epsilon_2\frac1{|\mathcal A|},
\]
where $\epsilon_2=0.2$.
This gives better action coverage than the expert-like policy while preserving
a clinically meaningful behavior pattern.

The target dataset \(\mathcal D_2\) is fixed at \(25{,}000\) episodes. The source dataset
\(\mathcal D_1\) is varied as a fraction of \(12{,}500\) episodes:
\[
\mathcal D_1\text{ fraction}\in\{0.2,0.4,0.6,0.8,1.0\}.
\]
Thus \(\mathcal D_1\) fraction \(0.2\) corresponds to \(2{,}500\) source episodes, and
\(\mathcal D_1\) fraction \(1.0\) corresponds to \(12{,}500\) source episodes. This
choice ensures that the largest source dataset is only half the size of the
target dataset, reflecting the source-scarce regime motivated in the main text.

For each configuration, we report means and standard deviations over \(100\)
random experiments, corresponding to \(10\) independent dataset draws and
\(10\) optimization seeds per dataset draw. The same generated datasets are used
across all compared methods, so that performance differences are attributable to
the transfer procedure rather than to differences in sampled data.

The source behavior policy is estimated empirically from \(\mathcal D_1\) as
\[
\widehat\pi_{b1}(a\mid s)
=
\frac{
\#\{(s_t,a_t)\in D_1:s_t=s,\ a_t=a\}
}{
\#\{s_t\in D_1:s_t=s\}
}.
\]
Oracle quantities are
computed using the data-generating \(\pi_{b1}\).
The target
evaluation is performed under the fixed target environment
\[
(\rho_2,P_2,r_C^\star,\gamma_2,\tau_2).
\]

\subsection{Methods and optimization}
\label{app:training_details}

We compare three transfer procedures.

\paragraph{Modular.}
The modular estimator follows a two-stage reward-transfer pipeline. In the
first stage, it solves the source-side residual objective to obtain
\(\widehat q_1^{\rm mod}\). It then constructs the plug-in transferred reward
\[
\widehat r_{\rm mod}
=
(I-\Pi_\mu)\widehat q_1^{\rm mod}+g+C,
\]
where \(\Pi_\mu q(s,a)=q(s,a_0)\) in the anchor-action/no-treatment case. In the
second stage, this reward is held fixed and the target soft-control residual
objective is optimized to obtain \(\widehat q_2^{\rm mod}\) and the induced
target policy \(\widehat\pi_2^{\rm mod}\).

\paragraph{Coupled.}
The coupled estimator jointly optimizes the source and target residual
objectives. It treats \(q_1,l_1,q_2,l_2\) as trainable tabular variables and
solves the joint saddle-point objective
\[
\min_{q_1,q_2}\max_{l_1,l_2}
\frac{\beta}{2}\|q_1\|_{\rho_1}^2
+
\frac{1}{2}\|q_2\|_{\rho_2}^2
+
\langle l_1,b_1(q_1)\rangle_{\rho_1}
+
\langle l_2,b_2(q_1,q_2)\rangle_{\rho_2}.
\]
Here \(b_1\) is the source residual and \(b_2\) is the target residual using the
reward recovered from \(q_1\). This formulation allows target-side residual
information to feed back into the source representation.

\paragraph{Coupled-Offset.}
Coupled-Offset first runs the Modular procedure. It then initializes from the
modular solution and jointly learns tabular offset corrections:
\[
q_1 = \widehat q_1^{\rm mod} + \Delta q_1,
\qquad
q_2 = \widehat q_2^{\rm mod} + \Delta q_2 .
\]
The offset variables are initialized at zero, so the initial coupled solution is
exactly the Modular estimator. This baseline tests whether joint correction of a
modular initialization is sufficient to remove the main plug-in propagation
error.

\paragraph{Parameterization and initialization.}
All methods use a tabular parameterization over the state-action space
\(|\mathcal S|\times|\mathcal A|=128\times 8\). We set the discount factors to
\(\gamma_1=0.95\) and \(\gamma_2=0.975\), and vary the target soft-control
temperature \(\tau_2\) across experiments. Thus \(q_1,l_1,q_2,l_2\) are
represented as direct tabular tensors. To isolate the statistical transfer
effect from optimization failures, Modular and Coupled use the same controlled
initialization: the \(q\)- and \(l\)-tensors are initialized around a fixed
reference initialization with additive Gaussian noise of standard deviation
\(1.5\). Coupled-Offset is initialized from the trained Modular solution, and
its learned offset variables are initialized at zero.

\paragraph{Optimization.}
All optimization is performed with Adam on the tabular tensors. We use learning
rate \(10^{-3}\) for all \(q\)-variables and \(10^{-4}\) for all dual
\(l\)-variables. Each training round performs 10 ascent steps for the dual
variables and 1 descent step for the primal \(q\)-variables. We use the
empirical state-action weights \(\rho_d\) and optimize only over visited
state-action pairs. Gradients are unclipped in the reported runs. All runs are
performed on CPU with one PyTorch thread per worker.

\paragraph{Training horizon.}
For Modular, we train the source stage for \(40{,}000\) steps and the target
stage for \(70{,}000\) steps. For Coupled and Coupled-Offset, we train the joint
source--target objective for \(40{,}000\) steps. These horizons were chosen so
that the corresponding training objectives had converged; additional training
did not materially change the reported metrics. The joint objective uses
\(\beta=100\) as the quadratic penalty weight on the source value \(q_1\). This
choice is not tuned for performance, but is used to stabilize optimization by
balancing the different numerical scales of the source- and target-stage
quantities. We also do a ablation study when $\beta=0$ in Appendix~\ref{app:beta_zero_ablation}. We evaluate the training trajectory periodically and report the
final iterate.

The target soft-control operator uses
\[
\Omega_{\tau_2,\pi_{2,\rm ref}}(q)(s)
=
\tau_2\log\sum_a \pi_{2,\rm ref}(a\mid s)\exp(q(s,a)/\tau_2),
\]
and the learned policy is
\[
\widehat\pi_2(a\mid s)
=
\frac{\pi_{2,\rm ref}(a\mid s)\exp(\widehat q_2(s,a)/\tau_2)}
{\sum_{a'}\pi_{2,\rm ref}(a'\mid s)\exp(\widehat q_2(s,a')/\tau_2)},
\]
where \(\pi_{2,\rm ref}\) denotes the uniform policy over actions.

All methods are evaluated on the same generated datasets and random seeds. For
the source-size experiments, we use 10 random seeds and 10 independently drawn
source subsets per seed for each \(\mathcal D_1\) fraction.

\subsection{Metric definitions}
\label{app:metric_definitions}

We evaluate each method using target-side and source-side metrics.

\paragraph{Target policy regret.}
The target policy regret is
\[
J_2(\pi_2^\star)-J_2(\widehat\pi_2),
\]
where \(J_2\) is evaluated in the target environment using
\((P_2,r_C^\star,\gamma_2,\tau_2)\). Smaller regret indicates better downstream
target-control performance.

\paragraph{Target action-value error.}
The target action-value error is
\[
\|\widehat q_2-q_2^\star\|_{\rho_2}^2
=
\sum_{s,a}\rho_2(s,a)
\left(\widehat q_2(s,a)-q_2^\star(s,a)\right)^2.
\]

\paragraph{Target state-value error.}
For any policy-value pair \((\pi_2,q_2)\), the target state value is computed as
\[
V_2(s):=(\pi_2q_2)(s)
=
\sum_{a\in\mathcal A}\pi_2(a\mid s)q_2(s,a).
\]
The unweighted target state-value error is
\[
\|\widehat V_2-V_2^\star\|_{\rm unif}^2
=
\frac1{|\mathcal S|}
\sum_{s\in\mathcal S}
\left(\widehat V_2(s)-V_2^\star(s)\right)^2.
\]
This metric is uniform over states but remains policy-weighted over actions
through the construction of \(V_2\).

\paragraph{Source reward error.}
The source reward error is
\[
\|\widehat r-r^\star\|_{\rho_1}^2
=
\sum_{s,a}\rho_1(s,a)
\left(\widehat r(s,a)-r^\star(s,a)\right)^2.
\]

\paragraph{Source value error.}
The source value error is
\[
\|\widehat q_1-q_1^\star\|_{\rho_1}^2
=
\sum_{s,a}\rho_1(s,a)
\left(\widehat q_1(s,a)-q_1^\star(s,a)\right)^2.
\]

\paragraph{Relative improvement.}
In the relative-improvement plots, we report improvement over Modular as
\[
\frac{
\mathrm{Metric}_{\rm Modular}
-
\mathrm{Metric}_{\rm Method}
}{
\mathrm{Metric}_{\rm Modular}
}
\times 100\%.
\]
For all reported metrics, smaller values are better, so larger relative
improvement indicates better performance.

\begin{figure}[h]
    \centering
    \includegraphics[width=0.95\linewidth]{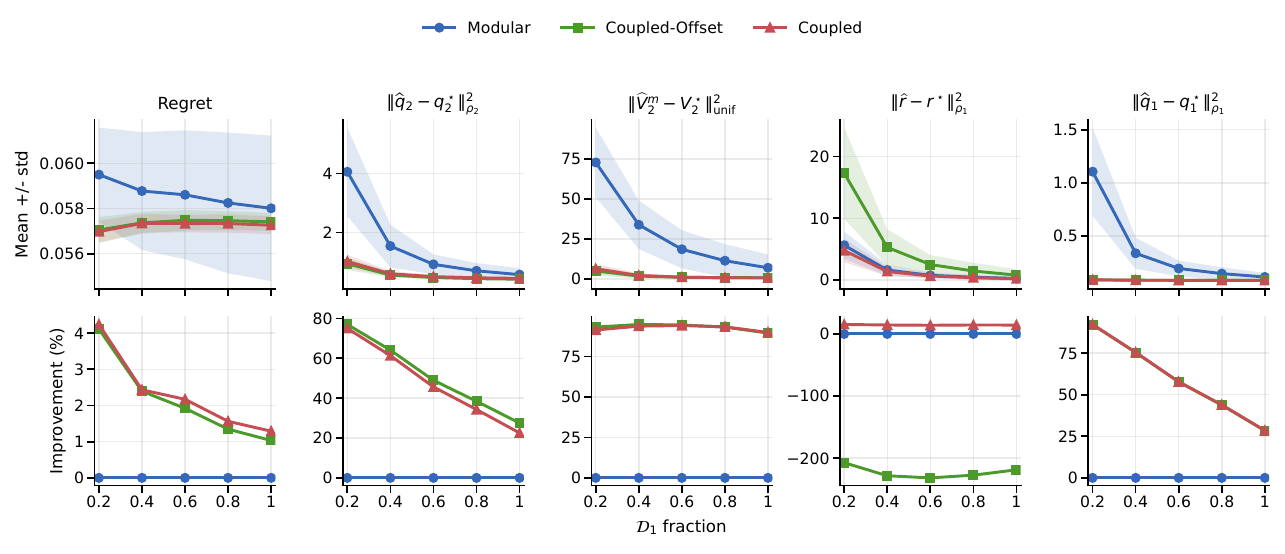}
    \caption{
    Effect of source sample size on modular and coupled transfer for
    \(\tau_2=0.2\): absolute metrics over \(\mathcal D_1\)
    fractions \(0.2\)--\(1.0\) (top) and percent improvement over Modular
    (bottom); larger is better.
    }
    \label{fig:d1_tau02}
\end{figure}

\begin{figure}[h]
    \centering
    \includegraphics[width=0.95\linewidth]{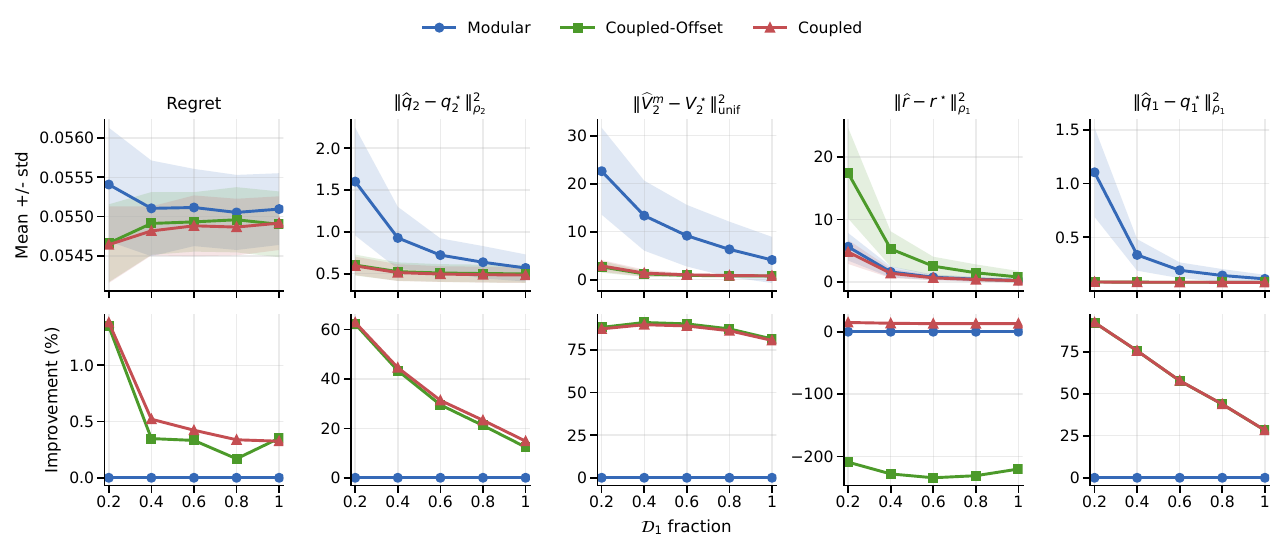}
    \caption{
    Effect of source sample size on modular and coupled transfer for
    \(\tau_2=0.4\): absolute metrics over \(\mathcal D_1\)
    fractions \(0.2\)--\(1.0\) (top) and percent improvement over Modular
    (bottom); larger is better.
    }
    \label{fig:d1_tau06}
\end{figure}

\subsection{Additional source-size results}
\label{app:additional_source_size_results}

Figure~\ref{fig:d1_subset} in the main text reports the source-size experiment
for \(\tau_2=0.05\). Figures~\ref{fig:d1_tau02}--\ref{fig:d1_tau06} show the
same experiment at higher target temperatures. We observe the same qualitative
behavior: the coupled methods are most beneficial when source data are scarce,
and the gap between coupled and modular transfer narrows as the source dataset
grows. This is consistent with the first-order error propagation theory: when
the source residual is large, modular plug-in transfer passes this
error directly into the target Bellman equation, while coupled transfer removes
the direct first-order source-residual channel.

\subsection{Temperature sensitivity}
\label{app:temperature_sensitivity}

Table~\ref{tab:target_transfer_tau} reports the main temperature sweep. The
target soft-control temperature \(\tau_2\) controls how concentrated the
oracle target policy is. Lower temperature produces a more deterministic policy,
while higher temperature produces a softer policy. In our experiments, the
average top-action probability of the oracle target policy,
\[
\frac1{|\mathcal S|}
\sum_{s\in\mathcal S}
\max_a \pi_2^\star(a\mid s),
\]
is \(0.844\), \(0.560\), and \(0.391\) for
\(\tau_2=0.05\), \(0.2\), and \(0.4\) respectively.

As shown in Table~\ref{tab:target_transfer_tau}, smaller \(\tau_2\) amplifies
the downstream effect of target \(q_2\)-estimation errors. The larger empirical
gains of the coupled methods in this regime suggest that Modular is more
affected by source-induced target errors, whereas the coupled methods are less
sensitive to the direct first-order source-residual propagation channel. This
behavior is consistent with Theorem~\ref{thm:policy_regret}: as \(\tau_2\)
decreases, the bound translating target \(q_2\)-estimation error into policy
regret becomes larger, scaling as \(1/\tau_2\).

\subsection{Value-error decomposition}
\label{app:v2_decomposition}

The main text observes that the reduction in \(V_2\) error is much larger than
the reduction in the overall \(\rho_2\)-weighted \(q_2\) MSE. This is because
the state-value metric averages the target action-value function under the
learned policy. Specifically,
\[
\widehat V_2(s)-V_2^\star(s)
=
\sum_a \widehat\pi_2(a\mid s)
\{\widehat q_2(s,a)-q_2^\star(s,a)\}
+
\sum_a
\{\widehat\pi_2(a\mid s)-\pi_2^\star(a\mid s)\}
q_2^\star(s,a).
\]
The first term is the policy-weighted \(q_2\) estimation error, while the second
term is the policy mismatch error. In the low-temperature setting
\(\tau_2=0.05\), the target policies are highly concentrated; the corresponding
oracle target policy has average top-action probability \(0.844\). Therefore,
the first term is close to the \(q_2\) error on the action selected by the
learned policy.

\begin{table}[h]
\centering
\caption{Decomposition of \(V_2\) error at \(\tau_2=0.05\) and \(\mathcal D_1\) fraction \(0.2\).}
\label{tab:v2_decomposition}
\begin{tabular}{lccc}
\toprule
Method
& \(\|\widehat V_2-V_2^\star\|_{\rm unif}^2\)
& Policy-weighted \(q_2\) term
& Policy mismatch term \\
\midrule
Modular & \(83.38\) & \(82.74\) & \(0.016\) \\
Coupled-Offset & \(6.415\) & \(6.38\) & \(0.016\) \\
Coupled & \(6.739\) & \(6.72\) & \(0.016\) \\
\bottomrule
\end{tabular}
\end{table}

Table~\ref{tab:v2_decomposition} reports this decomposition at
\(\tau_2=0.05\) and \(\mathcal D_1\) fraction \(0.2\). The policy mismatch term is small
for all methods. Thus, the large \(V_2\) error of Modular is primarily caused by
large \(q_2\) error on policy-selected actions, rather than by differences
between \(\widehat\pi_2\) and \(\pi_2^\star\).

\subsection{Reward-recovery diagnostic}
\label{app:reward_recovery_diagnostic}

The main text notes that Coupled-Offset can have much larger reward error than
Coupled, despite having nearly identical weighted \(q_1\) MSE. This is due to
the contrast structure of anchor-normalized reward recovery. Under the
no-treatment anchor,
\[
\widehat r(s,a)-r^\star(s,a)
=
(\widehat q_1-q_1^\star)(s,a)
-
(\widehat q_1-q_1^\star)(s,a_0),
\]
where \(a_0\) is the no-treatment action. Therefore, an error in the
no-treatment baseline \(q_1(s,a_0)\) affects the recovered reward for every
action \(a\). As a result, small weighted \(q_1\) error does not necessarily
imply accurate reward recovery.

\begin{table}[h]
\centering
\caption{Reward contrast diagnostic at \(\tau_2=0.05\) and \(\mathcal D_1\) fraction \(0.2\).}
\label{tab:reward_contrast_diagnostic}
\begin{tabular}{lccc}
\toprule
Method
& \(\|\widehat q_1-q_1^\star\|_{\rho_1}^2\)
& No-treatment \(q_1\) MSE
& \(\|\widehat r-r^\star\|_{\rho_1}^2\) \\
\midrule
Coupled-Offset & \(0.083\) & \(489.2\) & \(18.62\) \\
Coupled & \(0.083\) & \(121.8\) & \(4.76\) \\
\bottomrule
\end{tabular}
\end{table}

Table~\ref{tab:reward_contrast_diagnostic} illustrates this effect at
\(\tau_2=0.05\) and \(\mathcal D_1\) fraction \(0.2\). Coupled-Offset and Coupled have
nearly identical weighted \(q_1\) MSE, but Coupled-Offset has a much larger
no-treatment baseline error. This explains its substantially larger reward MSE.

\begin{figure}[t]
    \centering
    \includegraphics[width=0.9\linewidth]{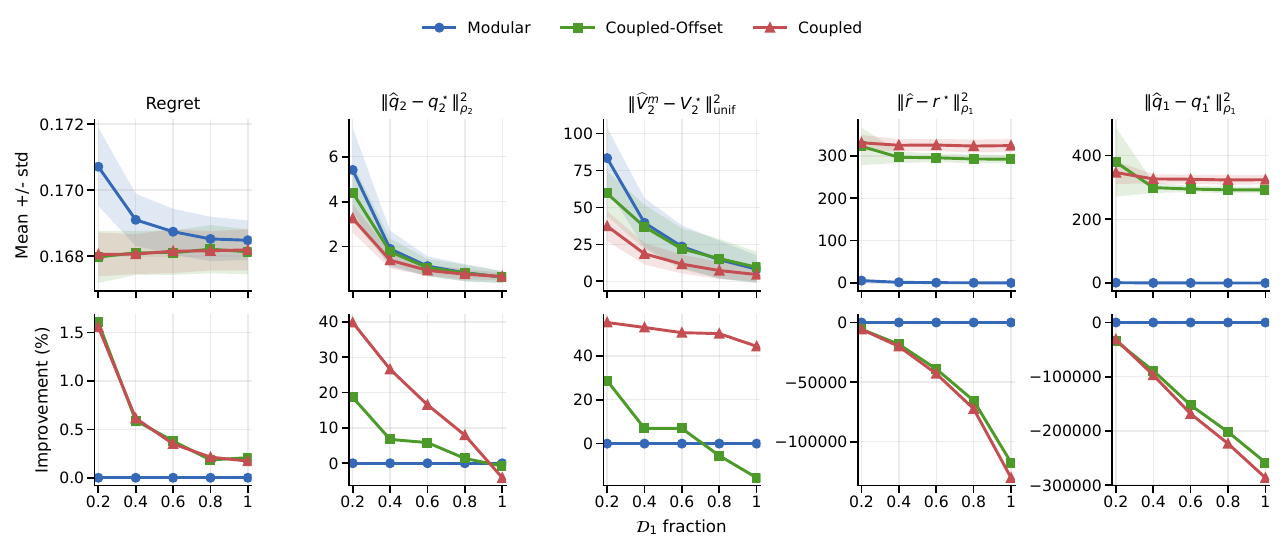}
    \caption{
    Comparison of Modular, Coupled-Offset, and Coupled with \(\beta=0\) under
    the mild-shift target kernel \(P_2\) at \(\tau_2=0.05\): absolute metrics over \(\mathcal D_1\)
    fractions \(0.2\)--\(1.0\) (top) and percent improvement over Modular
    (bottom); larger is better.
    }
    \label{fig:beta0_tau005}
\end{figure}

\begin{figure}[t]
    \centering
    \includegraphics[width=0.9\linewidth]{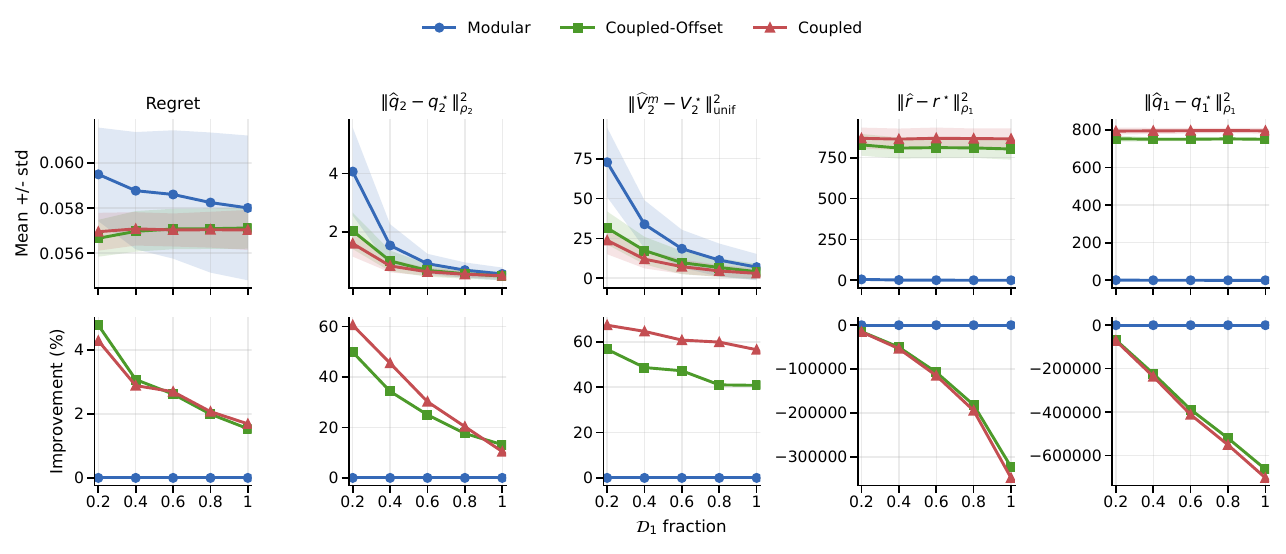}
    \caption{
    Comparison of Modular, Coupled-Offset, and Coupled with \(\beta=0\) under
    the mild-shift target kernel \(P_2\) at \(\tau_2=0.2\): absolute metrics over \(\mathcal D_1\)
    fractions \(0.2\)--\(1.0\) (top) and percent improvement over Modular
    (bottom); larger is better.
    }
    \label{fig:beta0_tau02}
\end{figure}

\begin{figure}[t]
    \centering
    \includegraphics[width=0.9\linewidth]{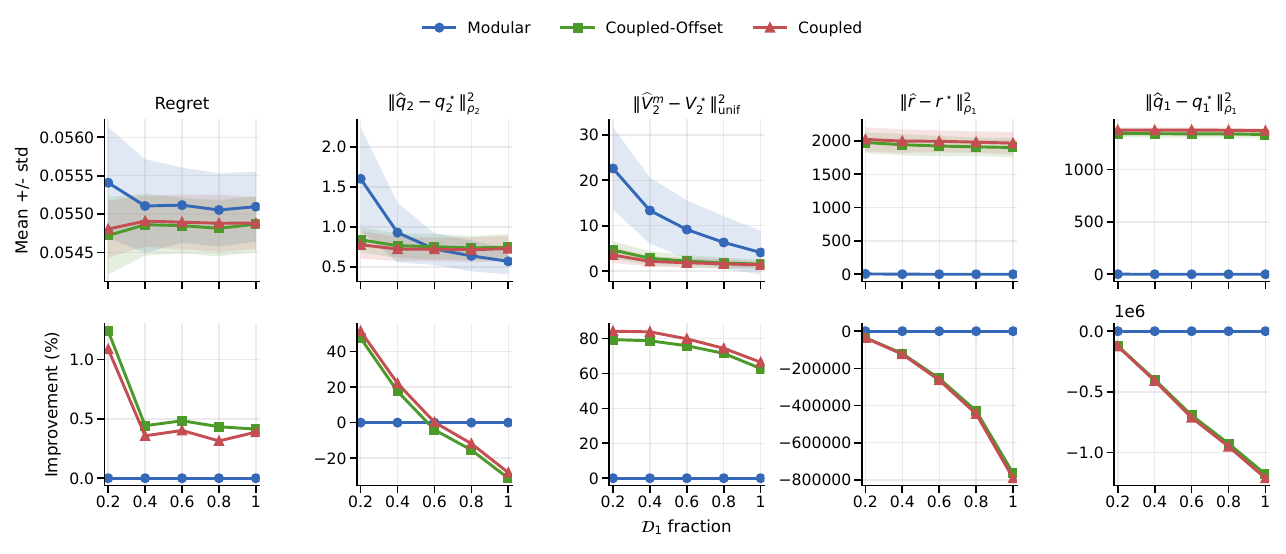}
    \caption{
    Comparison of Modular, Coupled-Offset, and Coupled with \(\beta=0\) under
    the mild-shift target kernel \(P_2\) at \(\tau_2=0.4\): absolute metrics over \(\mathcal D_1\)
    fractions \(0.2\)--\(1.0\) (top) and percent improvement over Modular
    (bottom); larger is better.
    }
    \label{fig:beta0_tau04}
\end{figure}

\begin{figure}[t]
    \centering
    \includegraphics[width=0.9\linewidth]{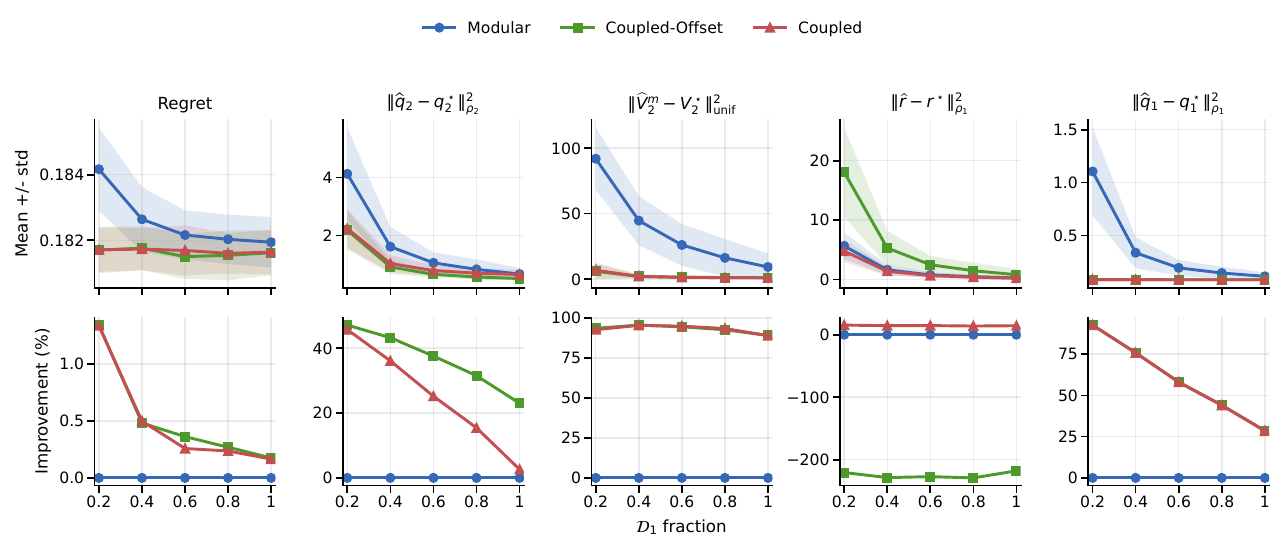}
    \caption{
    Comparison of Modular, Coupled-Offset, and Coupled under the larger-shift
    target kernel \(P_2\) at \(\tau_2=0.05\): absolute metrics over \(\mathcal D_1\)
    fractions \(0.2\)--\(1.0\) (top) and percent improvement over Modular
    (bottom); larger is better.
    }
    \label{fig:large_tv_tau005}
\end{figure}

\subsection{Ablation on the source quadratic penalty: \(\beta=0\)}
\label{app:beta_zero_ablation}

We conduct an ablation study to examine the role of the source-side quadratic
penalty in the coupled objective. Specifically, we set the source penalty weight
to \(\beta=0\), keep the mild-shift target kernel \(P_2\) used in the main
experiments, and compare the same three methods,\emph{Modular},
\emph{Coupled-Offset}, and \emph{Coupled}, at target temperatures
\(\tau_2\in\{0.05,0.2,0.4\}\). For this mild-shift setting,
\(d_{\rm TV}^{\rm avg}(P_1,P_2)=0.01461\) and
\(d_{\rm TV}^{\rm max}(P_1,P_2)=0.026\). All other aspects of the
data-generation and optimization setup are kept unchanged.

Figures~\ref{fig:beta0_tau005}--\ref{fig:beta0_tau04} summarize the results.
The main observation is that when \(\beta=0\), the coupled methods still control
the target-side quantities well: both the target value error
\(\|\widehat q_2-q_2^\star\|_{\rho_2}^2\) and the target policy regret remain
competitive across all three temperatures. In contrast, the source-side
quantities deteriorate substantially: the reward error
\(\|\widehat r-r^\star\|_{\rho_1}^2\) and the source value error
\(\|\widehat q_1-q_1^\star\|_{\rho_1}^2\) are no longer well controlled.

This behavior is consistent with the coupled minimax objective. When
\(\beta>0\), the primal objective contains an explicit quadratic penalty on
\(q_1\), which helps stabilize the source component and ties the coupled
solution more tightly to the oracle source representation. When \(\beta=0\),
this direct control disappears. Consequently, unlike the \(\beta>0\) case in
Theorem~\ref{thm:population_saddle_kkt}, the primal--dual gap
\(\mathcal L_{\rm coup}^{\beta}(q_1,q_2,l_1^\star,l_2^\star)
-
\mathcal L_{\rm coup}^{\beta}(q_1^\star,q_2^\star,l_1^\star,l_2^\star)\) no longer directly
controls the source error \(\|q_1-q_1^\star\|_{\rho_1}^2\). Consequently, the optimization can still find source
representations that induce a good target policy and accurate \(q_2\), while
allowing substantial drift in \(q_1\) and the recovered reward \(r\).

In short, this ablation shows that the source quadratic penalty is not crucial
for obtaining strong target-stage performance, but it is important for accurate
source recovery and for stabilizing the learned source representation.

\subsection{Robustness under a larger source--target dynamics shift}
\label{app:large_shift_results}

We next examine robustness under a larger source--target dynamics shift. In this
experiment, we replace the mild-shift target kernel used in the main text with a
larger-shift target kernel \(P_2\), described in
Section~\ref{app:shift_construction}. This larger shift has
\(d_{\rm TV}^{\rm avg}(P_1,P_2)=0.08766\) and
\(d_{\rm TV}^{\rm max}(P_1,P_2)=0.15675\). We again compare
\emph{Modular}, \emph{Coupled-Offset}, and \emph{Coupled} at
$\tau_2=0.05$, keeping all other aspects of the setup
unchanged.

Figures~\ref{fig:large_tv_tau005} shows that the same
qualitative conclusions continue to hold under the larger transition shift. In
particular, the coupled methods remain superior to the Modular estimator on the
target-side metrics, and Coupled-Offset continues to closely track Coupled on
most target-side quantities. Thus, the main empirical phenomenon in the paper is
not specific to the mild-shift regime.

The source-side diagnostic is also qualitatively unchanged. In particular,
Coupled-Offset can still match Coupled well on target-side performance while
showing noticeably worse reward recovery than Coupled. This is again consistent
with the contrast-based structure of reward recovery: accurate target transfer
does not by itself guarantee accurate source reward reconstruction.

Overall, these results show that the empirical advantages of coupled minimax
transfer are robust to a substantially larger one-step transition shift. The
same source-to-target error-propagation picture observed in the main text
persists beyond the mild-shift setting.

\subsection{Summary of additional findings}
\label{app:additional_findings_summary}

The additional diagnostics support five conclusions. First, coupled transfer is
most useful in the source-scarce regime, where source residuals have higher
variance and modular plug-in transfer is most vulnerable to first-order
source-to-target error propagation. Second, the \(V_2\) error reduction is
driven primarily by improved \(q_2\) estimation on policy-selected actions, not
by changes in policy mismatch. Third, reward recovery must be evaluated through
action contrasts: small weighted \(q_1\) error can hide large baseline-action
errors that induce large reward error after anchor normalization. Fourth, the
\(\beta=0\) ablation shows that the source quadratic penalty is important for
stabilizing \(q_1\) and the recovered reward, even though the coupled methods
can still control target-side quantities such as \(q_2\) and regret. Finally,
the larger-shift experiments show that the qualitative advantage of coupled
transfer persists under a substantially larger source--target transition shift.

\section{Limitations and Future Work}
\label{app:future_work}

Our analysis assumes that the source-stage temperature \(\tau_1\) in the IRL
model is known, following a common convention in the literature. In practice,
however, demonstration datasets may have different levels of stochasticity,
suboptimality, or measurement noise, which correspond to different effective
temperatures. Misspecifying \(\tau_1\) may therefore affect both reward recovery
and downstream target performance.

An important direction for future work is to relax this assumption by estimating
or adapting \(\tau_1\) from data. Possible approaches include jointly learning
\(\tau_1\) with the reward representation, calibrating \(\tau_1\) using a
validation criterion, or designing reward-transfer algorithms that remain stable
when \(\tau_1\) is misspecified. These extensions would make the framework
better suited to real-world demonstrations collected from heterogeneous sources,
such as different experts, institutions, or noise levels.

The theory also relies on realizability, coverage, boundedness, and known
source- and target-reference policies. These assumptions isolate the statistical
effect studied here: how source IRL error propagates into target control and how
coupling changes that propagation. The experiments use a finite tabular
reduction of one synthetic sepsis simulator, which makes oracle diagnostics
available but does not test function approximation, continuous states or actions,
or real clinical deployment. Extending the analysis and algorithms to those
settings is an important next step.

\end{document}